\documentclass{article}

 \usepackage[main, final]{neurips_2026}
\usepackage[utf8]{inputenc} % allow utf-8 input
\usepackage[T1]{fontenc}    % use 8-bit T1 fonts
\usepackage{hyperref}       % hyperlinks
\usepackage{url}            % simple URL typesetting
\usepackage{booktabs}       % professional-quality tables
\usepackage{amsfonts}       % blackboard math symbols
\usepackage{nicefrac}       % compact symbols for 1/2, etc.
\usepackage{microtype}      % microtypography

\usepackage{graphicx}
\usepackage{subcaption}
\usepackage{multirow}

\usepackage{algorithm}
\usepackage{algorithmic}

\usepackage{wrapfig}
\usepackage{caption}

\usepackage{amsmath}

\usepackage{amssymb}
\usepackage{mathtools}
\usepackage{amsthm}
\newcommand{\E}{\mathbb{E}}

\usepackage{tikz}
\usetikzlibrary{arrows.meta, positioning, calc, shapes.geometric, backgrounds, fit}
\usepackage{xcolor}

\usepackage[capitalize,noabbrev]{cleveref}

\theoremstyle{plain}
\newtheorem{theorem}{Theorem}[section]
\newtheorem{proposition}[theorem]{Proposition}

\newtheorem{corollary}[theorem]{Corollary}
\theoremstyle{definition}

\newtheorem{assumption}[theorem]{Assumption}
\theoremstyle{remark}

\title{Flow Matching for Offline Reinforcement Learning with Discrete Actions}

\author{%
  Fairoz Nower Khan \\
  Department of Computer Science\\
  University of Kentucky\\
  Lexington, KY 40506, USA \\
  \texttt{fairoz.khan@uky.edu}
  \And
  Nabuat Zaman Nahim \\
  Department of Computer Science\\
  University of Kentucky\\
  Lexington, KY 40506, USA \\
  \texttt{nabuat.nahim@uky.edu}
  \And
  Ruiquan Huang \\
  Department of Computer Science\\
  University of Kentucky\\
  Lexington, KY 40506, USA \\
  \texttt{ruiquanhuang@uky.edu}
  \And
  Haibo Yang \\
  Department of Computing and Information Sciences\\
  Rochester Institute of Technology\\
  Rochester, NY 14623, USA\\
  \texttt{hbycis@rit.edu}
  \And
  Peizhong Ju \\
  Department of Computer Science\\
  University of Kentucky\\
  Lexington, KY 40506, USA \\
  \texttt{peizhong.ju@uky.edu}
}

\begin{document}

\maketitle

\begin{abstract}
Generative policies based on diffusion models and flow matching have shown strong promise for offline reinforcement learning (RL), but their applicability remains largely confined to continuous action spaces. To address a broader range of offline RL settings, we extend flow matching to a general framework that supports discrete action spaces with multiple objectives.
Specifically, we replace continuous flows with continuous-time Markov chains, trained using a Q-weighted flow matching objective. We then extend our design to multi-agent settings, mitigating the exponential growth of joint action spaces via a factorized conditional path. For an oracle single policy-improvement step from the behavior distribution, we theoretically show that, under idealized conditions, the Q-weighted objective recovers the KL-regularized Boltzmann policy.
Extensive experiments demonstrate that our method performs robustly across diverse settings and benchmarks, including high-dimensional control, multi-agent games, and dynamically changing preferences over multiple objectives, while achieving competitive performance and demonstrating advantages in multimodal decision-making and preference-conditioned policy representation. Our discrete framework can also be applied to continuous-control problems through action quantization, providing a flexible trade-off between representational complexity and performance.
\end{abstract}

\section{Introduction}

Offline reinforcement learning (RL) has recently seen a shift toward \emph{generative policy modeling}
driven by the success of diffusion models and continuous-time generative frameworks
\citep{wang2022diffusion,chen2023offline,kang2023efficient}.
Specifically, diffusion-based methods introduced value guidance either heuristically or via learned classifiers \citep{dhariwal2021diffusion,ho2021classifier}.
Subsequent work formalized this via energy-guided generative modeling, interpreting diffusion scores as gradients of a KL-regularized optimal policy~\citep{lu2023contrastive,zheng2023guided}

As a successor of diffusion models, flow matching models \citep{lipman2022flow,liu2022rectified} have recently been applied to offline RL in \citet{zhang2025energy}. By replacing stochastic diffusion dynamics with deterministic flow matching, \citet{zhang2025energy} showed that energy guidance can be combined directly into the training objective via Q-weighted regression.
This eliminated the need for auxiliary classifiers and yielded theoretically grounded policy improvement in continuous action spaces.

However, applying flow matching models to RL is far from complete, since the existing generative offline RL methods are mostly restricted to \emph{continuous} action spaces. Both diffusion models and flow matching rely on stochastic or deterministic differential equations defined on Euclidean domains \citep{chen2018neural,lipman2022flow}, making them ill-suited for discrete or combinatorial decision problems in RL. In such settings, good decisions often lie in several distinct and incompatible modes, and forcing continuous generative models to interpolate between them can produce invalid actions and poor generalization when the dataset does not fully cover all behaviors \citep{yuan2025moduli}. 
% Meanwhile, applying generative offline RL to real-world domains such as recommendation, token-based generation, and structured control requires policy representations and training objectives that are \emph{native to discrete spaces}. 
This raises a meaningful question: 
\begin{center}
    \textit{How to adapt flow matching to RL with discrete actions?}
\end{center}

% \paragraph{Why discrete action spaces fundamentally change flow-based offline RL.}

% At a conceptual level, 
The answer is not simple because the current restriction to continuous action spaces is not only
a modeling choice but also a structural requirement of existing flow-based RL methods.
Flow matching relies on ordinary differential equations
(ODEs) that transport probability mass via a vector field defined on a Euclidean domain.
RL policy improvement is achieved by coupling these dynamics with gradients of value or energy functions with respect to actions.
When the action space is discrete, this formulation breaks down: vector fields are undefined, action-space derivatives do not exist, and the continuity equation underlying flow matching no longer applies.
The lack of action-gradient guidance and continuous change-of-variables arguments, prevents existing theoretical guarantees for continuous actions from carrying over directly.
As a result, extending energy-guided generative policy optimization to discrete offline RL requires rethinking the policy representation, training objective, and underlying theory.

\paragraph{Motivating Example.}
Before introducing our solution, we illustrate why \emph{treating inherently discrete decisions with continuous generative models} can lead to failures.
Figure~\ref{fig:mofork_cartoon} presents a simple grid-world example in which the offline dataset contains two disjoint optimal strategies, while the intermediate region is a catastrophic trap state that is never observed.
This setting highlights a key limitation of continuous relaxations and regression-style policies: interpolating between discrete modes can produce invalid behavior that is absent from the data and leads directly to failure. This example motivates the need for a generative policy that can handle discrete spaces well.

\begin{figure}[t]
  \centering
  \vspace{-4pt}
  \begin{minipage}[t]{0.43\linewidth}
    \vspace{0pt}
    \centering
    \includegraphics[
      width=\linewidth,
      trim=20 20 20 20,
      clip
    ]{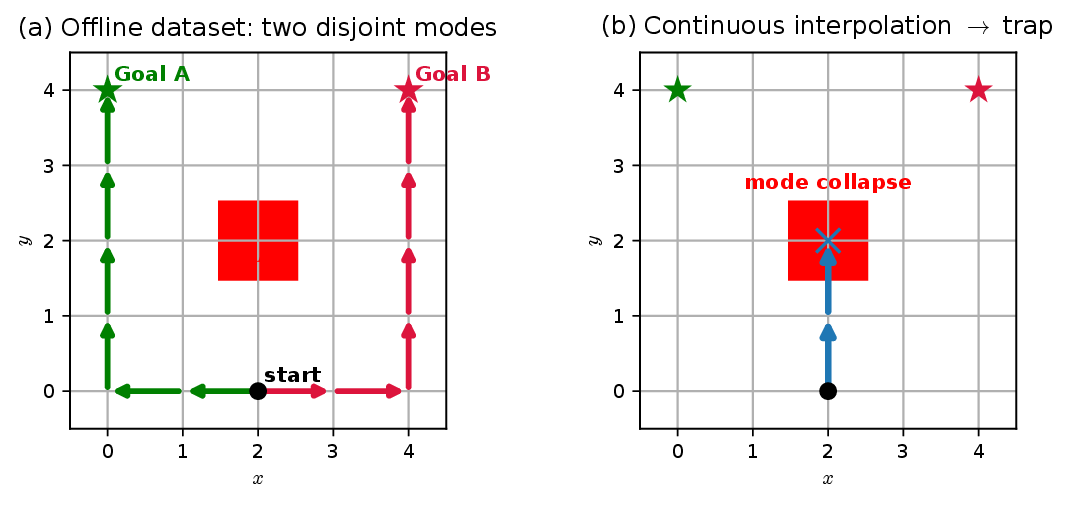}
  \end{minipage}
  \hfill
  \begin{minipage}[t]{0.53\linewidth}
    \vspace{0pt}
    \caption{\small
    Motivating example
    a) Offline dataset exhibits two disjoint modes corresponding to different objectives (Goal~A and Goal~B); the central region is a trap state never observed in the data.
    b) A continuous treatment of the decision space can induce interpolation between the modes and potential collapse toward the mean, causing trajectories to enter the trap.
    }
    \label{fig:mofork_cartoon}
  \end{minipage}
  \vspace{-8pt}
\end{figure}

\paragraph{Our Method.}
To answer the aforementioned question, we propose a new flow matching algorithm, Q-weighted discrete flow matching (QDFM), to solve the offline RL problem with a discrete action space.
Our key idea is to replace continuous ODE-based RL policy representations with a \emph{Continuous-Time Markov Chain (CTMC)} defined over a discrete action space, leveraging the recent Discrete Flow Matching (DFM) techniques \citep{campbell2024generative, gat2024discrete,lipman2024flow}.
Instead of learning a continuous vector field that transports probability mass via a continuity equation, our proposed algorithm targets a time-dependent rate function whose induced dynamics follow the Kolmogorov forward equation.

We then integrate DFM with value-based policy improvement by deriving a \emph{Q-weighted conditional objective} that leads to an energy-weighted loss for discrete spaces.
Theoretical analysis shows that the weighted conditional DFM objective is
equivalent in gradient to a guided marginal objective. For an oracle single
policy-improvement step that draws endpoints from the fixed behavior
distribution, we further show under idealized coverage and realizability
conditions that the resulting CTMC recovers the KL-regularized Boltzmann
policy without auxiliary classifiers. The practical algorithm uses
current policy endpoint sampling iteratively, which we analyze separately as
an approximate iterative policy improvement procedure. To address a wider range of practical scenarios, we develop our algorithm within a general multi-objective offline RL framework where the classical single-objective setting is a special case. By conditioning the flow matching objective on randomly sampled preference vectors during training, the resulting model can generate actions for changing preferences at inference time without retraining.
% We further extend the framework to \emph{multi-objective offline reinforcement learning} by conditioning both the critic and the CTMC policy on a preference vector.
% This yields a single generative policy capable of zero-shot control across preferences and smooth traversal of the Pareto frontier, avoiding the mode collapse commonly observed in discrete baselines.
Through extensive experiments, we show that our algorithm is robust across discretized, intrinsically discrete, multi-objective, and multi-agent benchmarks highlighting the adaptability of our design.
% To the best of our knowledge, this work presents the first principled framework for energy-guided generative policy optimization in discrete action spaces, extended to preference-conditioned multi-objective and multi-agent settings with a theoretical characterization of the underlying oracle policy-improvement objective.
To the best of our knowledge, QDFM is the first Q-guided discrete
flow matching framework for offline reinforcement learning, using a CTMC
generator to perform value-guided policy improvement directly over discrete
actions. We do not claim priority over value-guided discrete generative
policies more broadly, for which discrete diffusion approaches have also been
studied.
% Overall, this work establishes Discrete Flow Matching as a principled foundation for discrete and preference-conditioned generative control. 

% \paragraph{Contributions.}
% Our main contributions are:
% \begin{itemize}
%     \item We introduce QDFM, a Q-guided discrete flow-matching framework for
%     offline RL that combines CTMC-based discrete flow matching with value-guided
%     policy improvement, without requiring action gradients or continuous
%     relaxations.

%     \item We derive a Q-weighted conditional flow-matching objective and prove
%     its gradient equivalence to a guided marginal objective. Under coverage and
%     realizability assumptions, an oracle single policy-improvement step from the
%     behavior distribution recovers the KL-regularized Boltzmann policy.

%     \item We show that the generative discrete policy can preserve multiple
%     high-value action modes and represent preference-conditioned trade-offs with
%     a single model, while remaining competitive with standard offline RL methods
%     on average return.

%     \item We extend the formulation to multi-objective and multi-agent settings
%     through preference conditioning and factorized CTMC rates, respectively.
% \end{itemize}

\paragraph{Contributions.}
We (i) introduce QDFM, a Q-guided discrete flow matching framework for offline RL, (ii) derive a Q-weighted conditional objective with gradient equivalence and an
oracle single-step KL-regularized recovery result, (iii) demonstrate multimodal and
preference-conditioned policy representation, and (iv) extend the framework to
multi-objective and multi-agent settings.

\section{Related Work}

\paragraph{Energy-Guided Diffusion and Flow Matching for RL.}
Contrastive Energy Prediction (CEP) \citep{lu2023contrastive} formalized value guidance by learning a time-dependent energy model whose gradients steer diffusion sampling toward the KL-regularized optimal policy.
Guided flows further generalized this idea, interpreting guidance as a modification of probability flows in continuous time \citep{zheng2023guided}.
Energy-Weighted Flow Matching \citep{zhang2025energy} advanced this line of work by incorporating Q-based guidance directly into the flow matching objective, eliminating the need for auxiliary classifiers and providing theoretical guarantees for policy recovery in continuous action spaces.
Our work builds on this line by moving from continuous ODE flows to discrete
CTMC dynamics. Value-guided discrete generative RL has also been studied using
diffusion~\citep{coleman2023discrete,ma2026reinforcement} so we claim
novelty specifically for Q-guided discrete \emph{flow matching}, rather than
for discrete generative RL broadly. For a broader overview of flow matching methods in reinforcement learning, 
we refer readers to \citet{nahim2026a}.

\paragraph{Discrete Flow Matching and CTMC.}
Discrete Flow Matching \citep{campbell2024generative,gat2024discrete} extends the flow matching framework to discrete state spaces by replacing ODE dynamics with Continuous-Time Markov Chains (CTMCs) governed by the Kolmogorov forward equation.
This framework has been successfully applied to structured generation problems, including graphs and biological sequences \citep{qin2024defog,chen2025multi}.
Recent theoretical analyses have established learnability and consistency properties of discrete flow matching models \citep{su2025theoretical,wan2025discrete}.
However, prior work has not explored the integration of discrete flow matching with reinforcement learning objectives or value-based policy improvement.
As for Lexicographic MORL which imposes a strict priority ordering over objectives
\citep{skalse2022lexicographic}, QDFM uses preference weights to
represent preference-dependent Pareto trade-offs with a single model.

\section{Preliminaries}
\label{sec:prelims}

\subsection{Problem Setup and Notations}

We consider a multi-objective Markov decision process (MOMDP) $\mathcal{M}$ which consists of $(\mathcal{S}, \mathcal{A}, \mathbf{r}, P, \gamma)$,
where $\mathcal{S}$ is the state space, $\mathcal{A}$ is the action space,
$\mathbf{r}:\mathcal{S}\times\mathcal{A}\mapsto \mathbb{R}^K$ is a vector-valued reward for $K$ objectives, $P$ is the transition kernel,
and $\gamma \in (0,1)$ is the discount factor. In this work, we focus on the situation of discrete actions (so the number of actions $\left|A\right|$ is finite). The state space can be either discrete or continuous.

The offline dataset is $\mathcal{D} = \{(s_i, a_i, \mathbf{r}_i, s'_i)\}_{i=1}^N$, 
where $s_i'$ denotes the state after taking an action $a_i$ on the state $s_i$, $\mathbf{r}_i\in \mathbb{R}^K$ denotes the reward vector, and the dataset consists of $N$ samples. 
% These samples are collected by an unknown behavior policy $\mu(a \mid s)$.

In our multi-objective setting, the notion of a state--action value function depends
explicitly on a user preference over objectives. We learn a vector-valued critic
$
\mathbf{Q}_\psi : \mathcal{S}\times\mathcal{A} \to \mathbb{R}^K,
$
which predicts one state-action value per objective.
Here $\Delta^{K-1}:=\{\omega\in\mathbb{R}_{\geq 0}^{K}:
\sum_{k=1}^{K}\omega_k=1\}$ is the preference simplex.
For a given preference vector $\omega \in \Delta^{K-1}$, we define the scalarized reward
$r_\omega(s,a) := \langle \omega, \mathbf{r}(s,a) \rangle$.
The state--action value function (Q-function) is then defined as
\begin{equation}
\label{eq:q_omega_def}
Q_\omega(s,a)
\;:=\;
\mathbb{E}\!\left[
\sum_{t=0}^{\infty} \gamma^t r_\omega(s_t,a_t)
\;\middle|\;
s_0=s,\; a_0=a\;
\right]
\end{equation}
In our implementation, we define the learned scalarized critic as
$Q_\omega(s,a) := \langle\omega,\mathbf{Q}_\psi(s,a)\rangle$,
an approximation to the expected return Q-function above.

% {\color{red}In multi-objective setting, state-action function (Q-function) ...}

% We define a parameterized (vector-valued) state-action function $\mathbf{Q}_\psi(s,a)$ (i.e., our critic) for $K$ objectives:
% \[
% \mathbf{Q}_\psi(s,a) \coloneqq \big(Q_\psi^{(1)}(s,a), \dots, Q_\psi^{(K)}(s,a)\big) \in \mathbb{R}^K,
% \]
% where $\psi$ denotes the parameters of the critic.
% We further define the weighted-sum critic (which is a scalar) as
% \[
% Q_\omega(s,a) \coloneqq \sum\nolimits_{k=1}^K\omega_k Q_{\psi}(s,a),
% \]
% where $\omega=(w_1,w_2,\cdots,w_K) \in \mathbb{R}^K$ denotes a preference vector such that $\sum_{k=1}^K \omega_k = 1$ and $\omega_k \ge 0$ for all $k$.
% Given a preference vector $\omega \in \Delta^{K-1}$, where
% \[
% \Delta^{K-1} := \left\{ \omega \in \mathbb{R}^K : \omega_k \ge 0,\; \sum_{k=1}^K \omega_k = 1 \right\},
% \]
% we 

For each fixed state $s$ and preference vector $\omega$, the optimal policy under
KL-regularized policy optimization admits a closed-form solution. 
Specifically, when maximizing expected return while penalizing deviation from the
behavior policy $\mu(\cdot\mid s)$ via KL divergence,
the optimal policy is given by a Boltzmann distribution
\citep{todorov2006linearly,haarnoja2017reinforcement,lu2023contrastive}:
\begin{equation}
\label{eq:boltzmann_policy}
% \pi^\star_\omega(a \mid s)
% \;\propto\;
% \mu(a \mid s)\exp\!\big(\beta Q_\omega(s,a)\big),
\pi^\star_\omega(a \mid s)
=
\mu(a \mid s)\exp\!\big(\beta Q_\omega(s,a)\big)/C(s),
\end{equation}
where $C(s)$ is the normalization factor\footnote{When the context is clear, we omit the normalization factor by writing $\pi^\star_\omega(a | s)\propto \mu(a | s)\exp\!\big(\beta Q_\omega(s,a)\big)$.} $C(s)=\sum_a \mu(a|s)\exp\!\big(\beta Q_\omega(s,a)\big)$, and
$\beta > 0$ is an inverse-temperature parameter that controls the strength of value based
guidance: larger $\beta$ emphasizes more on high value actions, while smaller
$\beta$ keeps the policy closer to the behavior policy.

\subsection{Flow Matching}

Flow matching~\citep{lipman2022flow,liu2022rectified} is a generative modeling framework that learns a time-dependent velocity field that transports a simple source distribution $p_0$ to a target distribution $p_1$ through intermediate distributions $\{p_t\}_{t\in[0,1]}$. The induced dynamics satisfy the continuity equation
% \begin{equation}
% \label{eq:continuity_fm}
$\frac{\partial}{\partial t} p_t(x)
=
-\nabla_x \cdot \big(p_t(x)\,v_t(x)\big)$,
% \end{equation}
which describes the conservation of probability mass under deterministic transport. The sampling is performed by solving the ordinary differential equation (ODE)
% \begin{equation}
% \label{eq:fm_ode}
$\frac{d}{dt} X_t = v_t(X_t)$,
% \end{equation}
which deterministically maps samples from $p_0$ to $p_1$.
The velocity field $v_t$ is learned by matching it to a prescribed probability path,
rather than by maximizing likelihood or simulating stochastic diffusion dynamics. Recent work further combines flow matching with value-based guidance for offline RL in continuous action spaces~\citep{zhang2025energy}.
However, both the continuity equation and the ODE rely on gradients with respect to the underlying space,
which makes this formulation incompatible with discrete action spaces.

\subsection{Discrete Flow Matching}
\label{sec:dfm_prelims}

Discrete Flow Matching (DFM) extends flow matching from continuous Euclidean spaces to
\emph{discrete} spaces by modeling probability flows with a Continuous-Time Markov Chain
(CTMC) \citep{campbell2024generative}.
Where continuous flow matching learns a vector field governed by a continuity equation, DFM learns
a \emph{rate field} governed by the Kolmogorov forward equation. 

\textbf{Kolmogorov forward equation.}
Let $p_t$ denote the marginal probability mass function (PMF) of $X_t\in \mathcal{X}$,
and let $u_t(y,x)$ denote the $(y,x)$ entry of the time-dependent CTMC generator,
representing the jump rate from $x$ to $y$ for $y\neq x$.
The CTMC dynamics evolve according to the Kolmogorov forward equation 
$\frac{d}{dt}p_t(y)=\sum\nolimits_{x\in \mathcal{X}} u_t(y,x)\,p_t(x)$.

\textbf{CTMC rate field.}
Let $\mathcal{X}$ be a finite discrete space.
In this work, $\mathcal{X}$ corresponds to the discrete action space $\mathcal{A}$.  A CTMC is specified by time-dependent transition rates
$u_t(y,x)$ for $x,y\in \mathcal{X}$ (also called a generator or velocity field), satisfying the standard
rate (generator) constraints 
% \begin{equation}
% \label{eq:dfm_rate_constraints_prelim}
$u_t(y,x)\ge 0 \text{ for all } y\neq x,
\text{ and }
\sum\nolimits_{y} u_t(y,x)=0$. In discrete spaces, the object that replaces a velocity field is a rate matrix and learning it lets us know ``flow'' probability mass by jumps rather than derivatives.

% \end{equation}
% Equivalently, $\sum_{y\in \mathcal{X}} u_t(y,x)=0$ for each $x$, and off-diagonal entries are nonnegative.

% In later sections, we build on this
% framework to design an energy-guided, preference-conditioned policy optimization
% method for offline reinforcement learning.

\section{Method: Discrete Preference-Conditioned Flow Matching on Offline RL}
\label{sec:method}

Our goal is to learn a flow matching model that can generate discrete actions following an optimized policy in an offline RL setting.
To do that, we need to replace continuous ODE dynamics and action-space derivatives in existing flow-based offline RL methods.
Specifically, we consider a Markov jump process over discrete actions, and learn its transition rates from the corresponding flow. We first introduce the conditional
flow and Q-weighted training objective, then explain CTMC inference,
and finally summarize the three-phase training procedure in
Algorithm~\ref{alg:sketch}.

\subsection{Discrete Flow in Offline RL}
\label{sec:main_method}

% \paragraph{Conditional probability path.}

% \paragraph{Other Notations.}
% We write $s\in\mathcal S$ for a state, $a\in\mathcal A$ for a discrete action, and $\omega\in\Delta^{K-1}$ for a preference vector.
We use $A_t$ for the (discrete-action-valued) random variable of the flow indexed by continuous time $0\leq t\leq 1$.
A conditional path is conditioned on $Z \coloneqq (s,\omega,A_0,A_1)$, where $A_0$ is the initialization action
and $A_1$ is the endpoint action.
The behavior distribution that generates the offline dataset is $\mu(\cdot\mid s)$, while
$\hat\mu$ denotes the estimated (learned) version of $\mu$.
% Let $\mathcal{G}\subseteq \mathcal{A}$ be the action support in the offline dataset, whose cardinality is $|\mathcal{G}|=M$.
We use $\delta(\cdot,\cdot)$ for the Kronecker delta (a function that outputs $1$ when two inputs are the same, while outputs $0$ otherwise),
and $\gamma\in(0,1)$ for the discount factor in RL.

% We define the conditioning variable $Z := (s,\omega,A_0,A_1)$,
The conditional flow path we use is
\begin{equation}
\label{eq:dfm-mixture-path}
p_{t \mid Z}(a)
=
(1-t)\delta(a,A_0) + t\delta(a,A_1),
\ 0\leq t \leq 1.   
\end{equation}

The target of our training Algorithm~\ref{alg:sketch} is learning a time-dependent rate function
$u_{\theta,t}(a',a \mid s,\omega)$, where $\theta$ denotes the trainable
neural-network parameters.
In order to make sure that the CTMC implied by this rate function is valid, the rate function $u_{\theta,t}$ needs to satisfy the rate constraints
$u_{\theta,t}(a',a\mid s,\omega) \ge 0 \text{ for all } a' \neq a$ and $ 
\sum_{a'\in \mathcal{A}} u_{\theta,t}(a',a \mid s, \omega)=0$. 
% We enforce the CTMC generator constraints by construction which we detail in Appendix~\ref{app:rate_parameterization}.

\paragraph{Source and endpoint distributions.}
The conditional path in Eq.~\eqref{eq:dfm-mixture-path} is specified by a start action $A_0$ and an endpoint action $A_1$, whose roles depend on the training phase.
During the generative warm-up (Phase~1 of Algorithm~\ref{alg:sketch}), we sample $A_0\sim\mathrm{Uniform}$ and draw the endpoint
$A_1\sim\hat{\mu}(\cdot\mid s)$, encouraging the model to map a base distribution to the dataset-supported action distribution.
During policy improvement (Phase~3), we instead set $A_0\gets a$ for $(s,a)\sim\mathcal D$. For the target, we generate a support set $\{A_1^{(j)}\}_{j=1}^M$ of size $M$ by simulating the current rate model $u_{\theta,t}$ (using our inference Algorithm~\ref{alg:ctmc_inference}). These generated candidates are then reweighted according to $Q_\omega$, effectively biasing the probability flow toward the highest value regions of the model's current distribution.

% \paragraph{A valid conditional CTMC generator.}
% To realize the interpolation path in Eq.~\eqref{eq:dfm-mixture-path} using a {\color{red}valid CTMC,
% we construct a jump process whose transition rates vary with time
% (i.e., the jump intensity increases as $t \to 1$)
% and whose dynamics move probability mass directly toward the target action,
% after which no further transitions occur (endpoint absorbing).
% This design is important because it (i) induces the desired conditional path,
% and (ii) satisfies the CTMC generator constraints by construction.}

\paragraph{A valid conditional CTMC generator.}
To train the rate model $u_{\theta,t}$ via conditional discrete flow matching,
we require a \emph{ground-truth conditional rate field} $u_t(\cdot,\cdot \mid Z)$
against which the learned rate is compared in the flow matching objective given in Eq.~\eqref{eq:qcdfm}.
This target rate must (i) define a valid CTMC, and (ii) induce the desired conditional interpolation path
$p_{t\mid Z}$ in Eq.~\eqref{eq:dfm-mixture-path}.
We therefore carefully construct a time-dependent jump process whose
transition rates move probability mass toward the target action and satisfy
the CTMC generator constraints by construction ensuring that the flow matching objective compares
$u_{\theta,t}$ to a well-defined CTMC generator rather than an abstract
marginal rate.

Concretely, for $Z=(s,\omega,A_0,A_1)$, the ideal jump-to-endpoint
generator that exactly realizes the conditional path in
Eq.~\eqref{eq:dfm-mixture-path} is
\begin{equation}
\label{eq:jump_to_endpoint_main}
u_t(a',a \mid Z) =
\begin{cases}
\frac{1}{1-t}, & a \neq A_1,\;\; a' = A_1, \\
-\frac{1}{1-t}, & a \neq A_1,\;\; a' = a, \\
0, & \text{otherwise},
\end{cases}
\end{equation}

The ideal rate in Eq.~\eqref{eq:jump_to_endpoint_main} becomes singular
as $t\to1$. To ensure a finite integrated objective, all experiments
use the regularized target
\begin{equation}
\label{eq:jump_to_endpoint_eps}
u_t^{(\epsilon)}(a',a\mid Z)
=\frac{1-t}{1-t+\epsilon}\,u_t(a',a\mid Z),
\qquad \epsilon=10^{-4},\quad 0\le t<1.
\end{equation}
At $t=1$, the regularized rate is defined by continuous extension.
This preserves the CTMC generator constraints.
Further details are given in Appendix~\ref{app:ctmc_cdfm_details}.

\subsection{Training Objective and Loss}

The marginal flow matching objective depends on an intractable marginal rate field, so we instead need to optimize a tractable \emph{conditional} objective and prove that it has the correct gradient after adding value based weighting in the discrete setting. We defer the proof to Section~\ref{sec:theory_main}.

The conditional DFM loss we use is defined as:

\begin{equation}
\label{eq:qcdfm}
\begin{aligned}
\mathcal{L}_{t}^w(\theta)
&=
\E_{t,Z,A_t}\Big[
w(s,\omega,A_1)
\cdot
\big\|
u_t^{(\epsilon)}(\cdot,A_t \mid Z) -
u_{\theta,t}(\cdot,A_t \mid s,\omega)
\big\|_2^2
\Big].
\end{aligned}
\end{equation}

where $t\sim\mathrm{Unif}[0,1]$, $\epsilon=10^{-4}$ in all experiments,
$Z=(s,\omega,A_0,A_1)$ is sampled according to the procedure in
Algorithm~\ref{alg:sketch} (i.e., from $\mathcal D$ and $\hat\mu$), $A_t\sim p_{t\mid Z}$, and $\|\cdot\|_2$ denotes the $\ell_2$ norm\footnote{
While we use a squared $\ell_2$ norm in Eq.~\eqref{eq:qcdfm},
our design and analysis apply more generally to any Bregman divergence (the squared $\ell_2$-norm is one of them).} over the
vector of outgoing transition rates from state $A_t$.
The weight $w(s,\omega,a)$ is the normalized Q-based guidance term: 

\begin{equation}
\label{eq:q_weight}
w(s,\omega,a)
:=
\frac{\exp\!\big(\hat{\beta} Q_\omega(s,a)\big)}
{\sum_{j=1}^{M}
\!\left[\exp\!\big(\hat{\beta} Q_\omega(s,a'_j)\big)\right]}.
\end{equation}

Here $\hat{\beta}$ is the empirical guidance scale, and Phase~3 normalizes
weights over $M$ endpoints sampled from the current CTMC policy. Unlike the
fixed-source oracle update in Corollary~\ref{cor:optimal_policy}, the practical
procedure therefore performs iterative policy improvement rather than exact
one-step Boltzmann recovery. We analyze the population update and finite-sample
effects in Appendices~\ref{app:theory_algo_gap} and
\ref{app:approximation_errors}.

Since Eq.~\eqref{eq:qcdfm} trains a CTMC rather than a continuous vector field, existing continuous-action RL theory does not apply. In Section~\ref{sec:theory_main}, we show that the weighted loss has the same gradient as a guided marginal objective, giving a principled basis for discrete-action policy improvement.

\subsection{Inference Process}

\begin{wrapfigure}{r}{0.48\linewidth}
\vspace{-17pt}
\centering
\footnotesize
\refstepcounter{algorithm}
\noindent\textbf{Algorithm~\thealgorithm} Inference via CTMC simulation
\label{alg:ctmc_inference}
\vspace{2pt}
\hrule
\vspace{3pt}

\begin{algorithmic}[1]
\STATE {\bfseries Input:} State $s$, preference $\omega$, trained rate model $u_{\theta,t}$,
step size $h>0$, number of steps $N$ with $Nh=1$

\STATE $A_0  \sim \hat\mu(\cdot\mid s)$

\FOR{$n = 0$ {\bfseries to} $N-1$}
  \STATE Set $t_n \gets nh$ and current action $a \gets A_{t_n}$
  \STATE Total leaving rate $\lambda_n \gets \sum_{a' \neq a} u_{\theta,t_n}(a',a\mid s,\omega)$
  \IF{$\lambda_n = 0$}
    \STATE $A_{t_{n+1}} \gets a$
  \ELSE
    \STATE With probability $1-h\lambda_n$, set $A_{t_{n+1}} \gets a$
    \STATE With probability $h\lambda_n$, sample
    \STATE $\displaystyle
    A_{t_{n+1}} \sim
    \frac{u_{\theta,t_n}(\cdot,a\mid s,\omega)}{\lambda_n}$
  \ENDIF
\ENDFOR

\STATE {\bfseries Output:} Action $A_1$
\end{algorithmic}

\vspace{3pt}
\hrule
\vspace{-15pt}
\end{wrapfigure}

After training, the learned rate model $u_{\theta,t}$ implicitly defines a discrete policy
$\pi_\theta(a\mid s,\omega):=\Pr_\theta(A_1=a\mid s,\omega)$ as the terminal distribution of a CTMC over actions. Given a state $s$ and a preference vector $\omega$, actions are generated by simulating this Markov jump process from $t=0$ to $t=1$.

Note that since we already have the sampled different $\omega$ during training (Line~15 of Algorithm~\ref{alg:sketch} that follows), the learned model should fit for any $\omega$ in the inference stage.

% \paragraph{CTMC simulation.}
 Algorithm~\ref{alg:ctmc_inference} (more details in Appendix~\ref{app:ctmc_inference}) implements a discrete-time Euler approximation of the
underlying Continuous-Time Markov Chain defined by the rate model $u_{\theta,t}$.
At each time $t_n = nh$, the total leaving rate $\lambda_n \coloneqq \sum_{a' \neq A_{t_n}} u_{\theta,t_n}(a', A_{t_n} \mid s,\omega)$
% \[
% \lambda_n := \sum\nolimits_{a' \neq A_{t_n}} u_{\theta,t_n}(a', A_{t_n} \mid s,\omega)
% \]
determines the probability of a jump occurring within the next time step.
Specifically, the update
\[
\mathbb{P}(A_{t_{n+1}} = a' | A_{t_n})
=
\begin{cases}
1 - h\lambda_n, & a' = A_{t_n} \\
h\,u_{\theta,t_n}(a', A_{t_n}\mid s,\omega), & a' \neq A_{t_n}
\end{cases}
\]
corresponds to the standard (naive) Euler discretization of a CTMC
\citep[Eq.~(6.6)]{lipman2024flow}.
For bounded and sufficiently regular rates, this first-order Euler scheme has
$\mathcal{O}(h^2)$ one-step truncation error and $\mathcal{O}(h)$ accumulated
error over a fixed time horizon, and it is valid provided
$h\lambda_n \le 1$.

\subsection{Proposed Algorithm}
\label{sec:algo_sketch}

Algorithm~\ref{alg:sketch} summarizes the complete training procedure. Phase~1 performs an \emph{unweighted generative warm-up}, in which the rate network is trained via discrete flow matching to map a simple base distribution over actions to the empirical behavior distribution. 
\begin{wrapfigure}[27]{r}{0.56\linewidth}
\vspace{-1pt}
\centering
% \scriptsize
% \footnotesize
\small
\refstepcounter{algorithm}
\noindent\textbf{Algorithm~\thealgorithm} Q-weighted discrete flow matching for offline RL (sketched version; full algorithm in Appendix~\ref{app.full_algo})
\label{alg:sketch}
\hrule
\vspace{3pt}

\begin{algorithmic}[1]
   \STATE {\bfseries Input:} Offline RL Dataset $\mathcal{D}$, rate function $u_{\theta,t}$, guidance scale $\hat{\beta}>0$, behavior policy $\hat{\mu}(\cdot|s)$ (learned from standard offline RL behavior cloning), support action set size M, number of epochs $K_1, K_2$ and $K_3$

    %\vspace{0.1cm}
    \STATE \textbf{Phase 1: flow matching model warm-up}
    \FOR{warm-up step $k=1,2, \dots, K_1$}

    \STATE Sample batch $(s, a)\subset \mathcal{D}$, preference $\omega \sim \text{Uniform}$,
    \STATE {Update $u_{\theta,t}$ by minimizing the unweighted DFM loss}
    \ENDFOR
    
   \STATE \textbf{Phase 2: critic learning}
   \FOR{critic learning step $k= 1, 2, \dots ,K_2$}
      \STATE Sample batch $(s, a, \mathbf{r}, s')\subset \mathcal{D}$ and preference $\omega \sim \text{Uniform}$.
      \STATE \textbf{Scalarize:} $Q_\omega(s,a) \gets \langle \omega, \mathbf{Q}_\psi(s,a) \rangle$.
      \STATE Update vector critic $\mathbf{Q}_\psi$ via Bellman regression on the scalarized values $Q_\omega$.
   \ENDFOR
   %\vspace{0.1cm}
   
   \STATE \textbf{Phase 3: iterative policy improvement}
   \FOR{policy improvement step $k=1,2, \dots, K_3$}

    \STATE Sample batch $(s, a)$, preference $\omega \sim \text{Uniform}$, and endpoints $\{A_1^{(j)}\}_{j=1}^M$ using $u_{\theta,t}(\cdot, \cdot \mid s,\omega)$ (by running the CTMC inference in Algorithm~\ref{alg:ctmc_inference})
    \STATE Compute guidance weights using the scalarized Q-value: $w_j \propto \exp\!\big(\hat{\beta} \langle \omega, \mathbf{Q}_\psi(s, A_1^{(j)}) \rangle\big)$ (Eq.~\eqref{eq:q_weight}) over the sampled endpoints.
    \STATE Start $A_0 \gets a$ (Dataset)
    \STATE Update $u_{\theta,t}$ by minimizing the Q-weighted conditional DFM loss (Eq.~\eqref{eq:qcdfm}) with weights $\{w_j\}$, biasing the flow toward high value endpoints.

   \ENDFOR

   %\vspace{0.1cm}
   \STATE {\bfseries Output:} Trained rate function $u_{\theta,t}$
\end{algorithmic}

\vspace{1pt}
\hrule
\end{wrapfigure}
This initialization provides a data-supported generative model before value-based guidance is introduced. Phase~2 learns a preference-conditioned vector critic from the offline dataset and can be trained independently of Phase~1. Finally, Phase~3
performs iterative policy improvement by reweighting endpoint actions according to $Q_\omega$ and refining the rate model via weighted conditional flow matching.

Together, the training Algorithm~\ref{alg:sketch} and inference
Algorithm~\ref{alg:ctmc_inference} define the complete QDFM policy.

\smallskip

\subsection{Extension to Multi-Agent RL}
\label{sec:marl_extension}

Our framework can be extended to multi-agent reinforcement learning (MARL), where the main concern is the exponentially large joint action space when the number of agents $G$ is large.
To address this, we use a factorized technique \cite{campbell2022continuous,campbell2024generative} in DFM, in which transitions correspond to asynchronous single-agent updates.

\paragraph{Factorized CTMC rate.}
We define the joint transition rate as
\begin{equation}
u_{\theta,t}(a',a \mid s,\omega)
=
\sum_{i=1}^G
\delta(a'^{(-i)},a^{(-i)})\,
u^{(i)}_{\theta,t}(a'^{(i)},a \mid s,\omega),
\label{eq:marl_factorized_ctmc}
\end{equation}
where $u_{\theta,t}^{(i)}$ is the rate for agent $i$ to change its action conditioned on the current joint configuration $a$, $a^{(i)}$ denotes the action of agent $i$, and $a^{(-i)}$ denotes the actions of all agents except agent $i$.
This construction ensures that only one agent updates at a time and that the number of modeled rates
scales as $\sum_i |\mathcal{A}_i|$, rather than $\prod_i |\mathcal{A}_i|$. 

\paragraph{Factorized conditional path.}
To train the factorized generator, we define the conditional probability path
$p_{t\mid Z}(a)
=
\prod_{i=1}^G
\Big[(1-t)\delta(a^{(i)},a_0^{(i)}) + t\delta(a^{(i)},a_1^{(i)})\Big]$,
which implies that at any intermediate time $t$, the joint state is a mix of agents who have transitioned to their target actions $a_1^{(i)}$ and agents who remain at their start actions $a_0^{(i)}$.
This factorized path is perfectly compatible with the generator in Eq.~\eqref{eq:marl_factorized_ctmc}, allowing us to decompose the training objective into a sum of per-agent flow matching losses. For centralized multi-objective guidance while preserving decentralized execution, we use the QMIX framework \citep{rashid2020monotonic}, which represents the centralized joint action-value as a monotonic mixing of per-agent values. A detailed construction of the factorized conditional path and the corresponding decomposed training objective is provided in Appendix~\ref{sec:marl_appendix}.

% {\color{red}
% \paragraph{Centralized Multi-Objective Guidance.}
% Policy learning is guided by a centralized critic $\mathbf{Q}_\psi(s, a^{(1)}, \dots, a^{(G)})$. Individual agent flows are coordinated by the global preference-conditioned weight. This ensures that while the generative process is factorized, the agents are explicitly coordinated to maximize the joint multi-objective return. Since the objective decomposes as a sum of per-agent Bregman divergences,
% the marginalization and gradient-equivalence results from the single-agent
% setting apply unchanged in the multi-agent case.
% }

% \subsection{Theory: Guided DFM and Optimal Boltzmann Recovery}
\subsection{Theory: Gradient Equivalence and Optimal Policy Recovery}
\label{sec:theory_main}

% In flow matching, the ideal objective is a \emph{marginal} regression loss that matches the (typically intractable) marginal rate field along the true probability path.
% Our training objective is defined through a \emph{weighted conditional} flow-matching loss.
% We must show that minimizing our loss is equivalent to
% optimizing a well-defined \emph{guided marginal} objective, despite the additional complication of our setting that uses endpoint reweighting to bias learning toward high-value actions.
Flow matching ideally matches an intractable marginal rate field along the true probability path. Our objective instead uses a weighted conditional loss with endpoint reweighting. We show that this is gradient-equivalent to optimizing a well-defined guided marginal objective, giving a principled justification for our training loss. 

% Let $x$  and $z$ denote realizations of the random variables $X_t$ and $Z$, respectively.

\paragraph{Guided marginal objective.}
Fix $(s,\omega)$ and consider the conditional path $p_{t\mid Z}$ given in Eq.~\eqref{eq:dfm-mixture-path}, indexed by $Z=(s,\omega,A_0,A_1)$.
Let $w(Z)=w(s,\omega,A_1)$ be a nonnegative endpoint weight independent of $\theta$.
Define the tilted joint distribution at time $t$ as
\begin{equation}
\label{eq:tilted_joint_main}
\tilde p_t(x,z)
:=
\frac{w(z)\,p_t(x,z)}{C_t},
\ 
C_t := \mathbb{E}_{(X_t,Z)\sim p_t}[w(Z)].
\end{equation}
The corresponding guided marginal target rate is
\begin{equation}
\label{eq:utilde_main}
\tilde u_t(\cdot,x)
:=
\mathbb{E}_{Z\sim \tilde p_t(\cdot\mid X_t=x)}
\big[u_t(\cdot,x\mid Z)\big],
\end{equation}
which induces the guided marginal loss
\begin{equation}
\label{eq:marg_loss}
\mathcal{L}_{t,\mathrm{Marg}}^{w}(\theta)
:=
C_t\,\mathbb{E}_{X_t \sim \tilde{p}_t}\!\left[
\big\|\!\big(\tilde{u}_t(\cdot, X_t) -  u_{\theta,t}(\cdot, X_t)\big)\big\|_2^2
\right].
\end{equation}

\begin{theorem}[Gradient Equivalence]
\label{thm:guided_grad_equiv_main}
Assume $w(Z)$ is independent of $\theta$. Then the gradients of the guided conditional objective and
the guided marginal objective coincide:
\[
\nabla_\theta \mathcal{L}_{t}^w(\theta) = \nabla_\theta \mathcal{L}_{t,\mathrm{Marg}}^w(\theta).
\]

and therefore $\nabla_\theta \mathcal{L}_{t}^w(\theta)$ equals the gradient of a
guided marginal DFM objective that targets $\tilde u_t$.
\end{theorem}

% {\color{red}
% \paragraph{Remark (finite-sample implementation).}
% In Algorithm~\ref{alg:sketch}, the weights $w(Z)$ are computed using a self-normalized Monte Carlo estimate with endpoints sampled from the current policy. While this introduces dependence on previous iterates, the sampled endpoints are treated as fixed within each optimization step, and gradients are not taken through the sampling process. Thus, $w(Z)$ is effectively independent of $\theta$ during each update, matching the assumption of Theorem~\ref{thm:guided_grad_equiv_main}.
% }

\begin{corollary}[Oracle Single-Step Recovery of the KL-Regularized Policy]
\label{cor:optimal_policy}
Consider the ideal $\epsilon=0$ oracle population construction in which
endpoint actions are drawn from the fixed behavior policy
$\mu(\cdot\mid s)$ and reweighted by
$\exp(\beta Q_\omega(s,a))$. Under the coverage and realizability assumptions
stated in Appendix~\ref{sec:proof_cor1}, a global minimizer of the corresponding
weighted conditional flow matching objective induces the terminal distribution
\[
\pi^\star_\omega(a\mid s)
\propto
\mu(a\mid s)\exp\!\big(\beta Q_\omega(s,a)\big).
\]
\end{corollary}

% \paragraph{Discussion.}
The corollary is an oracle single-step result and is not a convergence
guarantee for the practical iterative Phase~3 procedure.
We discuss this distinction and approximation errors in
Appendices~\ref{app:theory_algo_gap} and~\ref{app:approximation_errors}.
Full proofs are provided in
Appendix~\ref{sec:proof_thm1} and~\ref{sec:proof_cor1} respectively.

% \paragraph{{\color{red}Recovery of the KL-regularized optimal policy (informal).}}
% Under standard dataset coverage and realizability assumptions, we further show that matching the induced guided marginal dynamics implies that the terminal CTMC distribution coincides with the KL-regularized (Boltzmann) policy.
% A formal statement and proof are provided in Appendix~\ref{sec:proof_thm2}.

% \paragraph{Proof Sketch.}
% The proof relies on the fact that the guided marginal rate $\tilde{u}_t$ generates the tilted probability path $\tilde{p}_t$. Since the conditional path $p_{t|Z}$ is constructed to be a Dirac delta at $t=1$, the terminal distribution $\tilde{p}_1$ becomes a weighted empirical distribution of the dataset endpoints. The weights $w(Z)$ are chosen specifically to match the Boltzmann exponent $\exp(\beta Q)$, ensuring that the aggregated flow converges to the optimal policy.
% Detailed derivations, including the application of Bregman gradient affine invariance, are provided in Appendix~\ref{sec:appendix_proofs}.

\section{Experiments}
\label{sec:experiments}

We evaluate QDFM along three complementary axes: average-return performance
under action quantization, multimodal and preference-conditioned policy
representation in intrinsically discrete settings, and extension to
multi-agent action spaces. Therefore, we treat return
on discretized continuous-control tasks as a measure of competitiveness,
while mode and preference coverage test the representational properties that
motivate a generative discrete policy.

\begin{figure}[ht]
  \centering
  \vspace{-4pt}
  \begin{minipage}[t]{0.75\linewidth}
    \vspace{0pt}
    \centering
    \includegraphics[
      width=\linewidth,
    ]{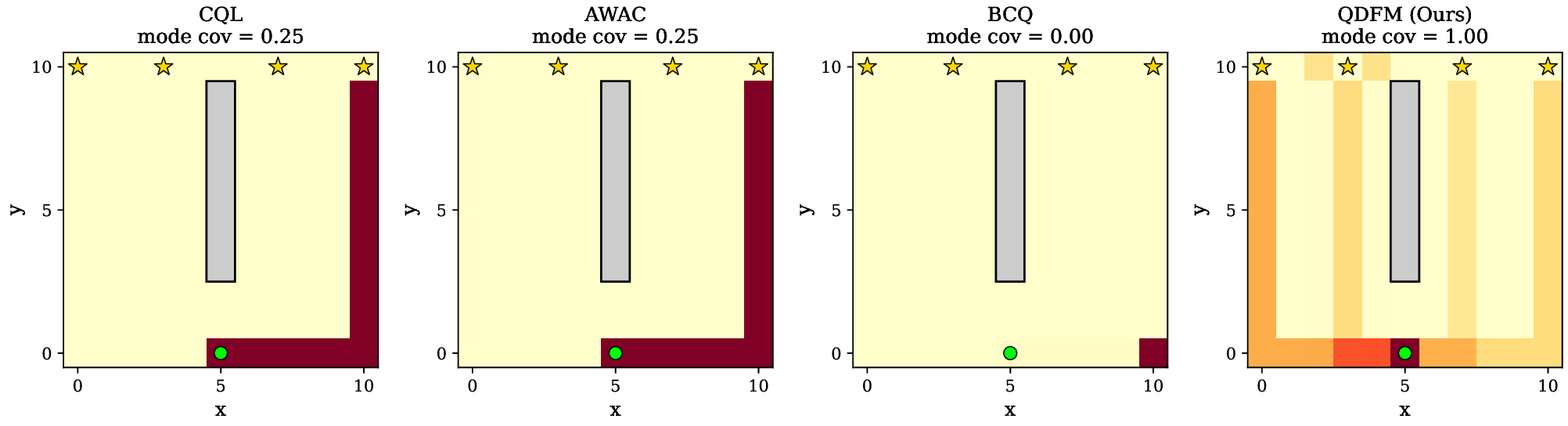}
  \end{minipage}
  \hfill
  \begin{minipage}[t]{0.21\linewidth}
    \vspace{0pt}
    \caption{\small
    Multimodal decision-making on a 4-goal gridworld. Standard high-return
baselines recover only one of the four successful modes, whereas QDFM
recovers all four from the same offline data.}
    \label{fig:multigoal}
  \end{minipage}
  \vspace{-7pt}
\end{figure}

\paragraph{Why use flow matching policies?}
When offline data contains multiple successful strategies, standard methods learn only one of
them because they select actions deterministically. Figure~\ref{fig:multigoal} shows this on
a gridworld with four equally rewarding goals. Strong baselines
obtain comparable return to QDFM but each recovers only one of the four modes
(mode coverage $=0.25$), whereas QDFM recovers all four
(mode coverage $=1.00$). Thus, the distinction is not higher average return,
but the ability of a single learned policy to preserve multiple high-value
behaviors present in the offline data. This ability to preserve diverse
behaviors from data can be useful in practical settings like multi-route planning or personalized decision-making,
where a single strategy is insufficient. Full results for
$K\in\{2,3,4,5\}$ and additional baselines are provided in
Appendix~\ref{app:multigoal}.

\subsection{Discretized MuJoCo Benchmark}
\label{sec:mujoco_discrete_bench}

Here we investigate whether an inherently discrete policy can produce effective behavior on continuous-control tasks through action discretization, using our discrete generative policy as an interface for continuous control. We evaluate on the MuJoCo benchmark from
Minari~\citep{todorov2012mujoco,farama2023gymnasium}, comparing QDFM against
eight baselines: BRAC~\citep{wu2019behavior, luo2023action}, IQL~\citep{kostrikov2022offline}, XQL~\citep{garg2023extreme}, ACTIVE-I~\citep{chen2025active}, BCQ~\citep{fujimoto2019off},
CQL~\citep{kumar2020conservative}, CQL(SA)~\citep{luo2025learning}, and a categorical policy-improvement baseline (BoltzQ). Additional comparisons are provided in
Appendix~\ref{app:comparisons}.
We adapt these baselines to the same discretized action space (action divisions $K_{\text{act}}=256$) for fair comparison.

\begin{table*}[ht]
\centering
\caption{Performance on the discretized MuJoCo benchmark. We report mean
$\pm$ standard deviation over 5 seeds. Evaluation numbers are normalized as
suggested by \citet{fu2020d4rl}. All methods operate on the same discretized
action space with $K_{\text{act}}=256$.}
\label{tab:main_baselines_256}
\small
\setlength{\tabcolsep}{0.35pt}
\renewcommand{\arraystretch}{1.15}
\resizebox{\textwidth}{!}{%
\begin{tabular}{llccccccccc}
\toprule
Dataset & Env
& IQL
& XQL
& ACTIVE-I
& BCQ
& CQL
& CQL(SA)
& BRAC
& BoltzQ
& \textbf{QDFM (Ours)} \\
\midrule

Expert & HalfCheetah
& 6.11$\pm$3.85
& 4.90$\pm$0.52
& 5.51$\pm$1.68
& \textbf{10.33$\pm$5.49}
& -1.19$\pm$0.01
& -0.78$\pm$0.72
& -0.03$\pm$0.01
& 4.84$\pm$2.15
& 1.32$\pm$1.09 \\

Medium & HalfCheetah
& \textbf{22.83$\pm$19.32}
& 15.42$\pm$5.20
& 6.46$\pm$2.63
& 20.23$\pm$8.21
& -0.58$\pm$0.81
& -0.88$\pm$0.35
& 0.42$\pm$1.65
& 4.62$\pm$0.58
& 6.02$\pm$2.00 \\

Expert & Hopper
& 16.94$\pm$13.94
& 16.74$\pm$8.40
& 10.20$\pm$5.12
& 4.95$\pm$4.75
& 0.90$\pm$0.00
& 13.52$\pm$8.69
& 0.81$\pm$0.01
& 0.69$\pm$0.01
& \textbf{17.74$\pm$7.61} \\

Medium & Hopper
& 35.42$\pm$4.77
& 32.44$\pm$3.13
& 28.81$\pm$3.03
& 36.18$\pm$26.22
& 1.13$\pm$0.02
& 28.02$\pm$12.47
& 2.82$\pm$0.04
& 1.84$\pm$0.03
& \textbf{37.79$\pm$16.11} \\

Expert & Walker2d
& 2.89$\pm$2.18
& 2.99$\pm$1.86
& 4.77$\pm$1.82
& 3.26$\pm$2.92
& -0.31$\pm$0.02
& 0.75$\pm$0.45
& 1.97$\pm$0.37
& 0.06$\pm$0.20
& \textbf{4.85$\pm$2.34} \\

Medium & Walker2d
& 13.11$\pm$9.42
& 6.53$\pm$2.15
& 5.49$\pm$4.56
& 9.30$\pm$5.93
& 6.65$\pm$0.25
& 12.25$\pm$3.13
& 3.15$\pm$1.98
& -0.09$\pm$0.01
& \textbf{14.01$\pm$9.88} \\

\bottomrule
\end{tabular}
}
\end{table*}

As shown in Table~\ref{tab:main_baselines_256}, our method is competitive with
standard offline RL baselines on the discretized MuJoCo tasks.  
Interestingly, we find high variance for several methods due to discretized-action setting. When continuous actions are discretized into a finite set, small changes in learned policy can lead to selecting entirely different discrete actions, resulting in abrupt changes in behavior and returns in contrast to continuous policies, which vary smoothly with parameter updates. Since the differences between QDFM and the strongest baselines are often within the reported standard
deviations, we do not interpret these results as evidence of uniform
return superiority. Instead, the MuJoCo experiments show that a native
discrete CTMC policy can achieve competitive return when continuous-control
actions are quantized to a finite codebook. 

Same-hardware measurements show that CTMC sampling is materially slower than
single-pass discrete policies. This additional inference cost is the main computational
trade-off of the generative CTMC policy. However, our method QDFM achieves reduction in action generation time (6.32ms) compared to reported values in \citet{zhang2025energy} for continuous action space diffusion-based (75.05ms, 55.86ms) and optimal-transport-based (27.26ms) methods. We provide same-hardware runtime comparisons on training and inference in Appendix~\ref{app:runtime}.

We also validate CTMC simulation behavior on intrinsically discrete single-objective control tasks and results are reported in Appendix~\ref{app:ctmc_correctness}. Further ablations on different number of action divisions $K_{\text{act}}$, stability, the number of Euler integration steps, the number of sampled actions $M$, and the number of CTMC integration steps $N$, including their runtime trade-offs, are provided in the Appendix~\ref{app:action_divisions},~\ref{app:comparisons},~\ref{app:ablations} and~\ref{app:mn_runtime_ablations}. 

\textbf{Comparison with QIPO and effect of quantization.}
We additionally compare with QIPO~\citep{zhang2025energy}, evaluating the same
trained continuous policy both natively and after projection onto our
$K_{\mathrm{act}}=256$ codebook. Quantization substantially reduces QIPO's
return (by approximately 33--92\% across the six MuJoCo tasks), with the
largest degradation on HalfCheetah. Although, QIPO after projection has
higher mean return than QDFM, these experiments
illustrate the cost of native categorical training versus continuous training
followed by projection. Full results are analyzed in
Appendix~\ref{app:qipo_quantization}.

\subsection{Multi-Objective Discrete Benchmarks}
We evaluate QDFM on DST (2 objectives) and RG (3 objectives)~\cite{vamplew2011empirical} against preference-conditioned scalarized CQL using HV ratio (hypervolume of the
returned solution set normalized by that of the dataset Pareto front), SP (measures variability
of distances between neighboring non-dominated solutions), and number of non-dominated solutions (ND), where higher ND indicates better Pareto-front coverage. Our goal is to show that one preference-conditioned QDFM model can generate high quality actions across preferences without training separate policies for each scalarization. As shown in Table~\ref{tab:morl_results}, QDFM recovers
multiple trade-offs with one preference-conditioned model, while scalarized
CQL collapses to one solution. On D4MORL MO-Hopper, QDFM also achieves higher
HV than MODULI~\citep{yuan2025moduli} and more closely approaches the dataset
Pareto upper bound (Table~\ref{tab:moduli_comparison}). Additional results are in Appendix~\ref{cartpole_morl_appendix},~\ref{app:mofork_heatmaps},~\ref{app:resource_gathering}.

\begin{table}[ht]
\centering

\begin{minipage}[t]{0.58\linewidth}
\centering
\caption{
Multi-objective performance on discrete benchmarks.
We report hypervolume ratio (HV Ratio) with respect to dataset Pareto front, spacing metric (SP), and number of distinct
non-dominated solutions (ND).
Scalarized CQL collapses to single solution (ND=1), while QDFM recovers diverse
approximation of Pareto front.
}
\label{tab:morl_results}
\resizebox{\linewidth}{!}{%
\begin{tabular}{llccc}
\toprule
Environment & Algorithm
& \textbf{HV Ratio} ($\uparrow$)
& \textbf{SP} ($\downarrow$)
& \textbf{ND} ($\uparrow$) \\
\midrule
\multirow{2}{*}{Deep Sea Treasure}
& CQL (Scalar)
& \textbf{0.92 $\pm$ 0.11}
& -- 
& 1 \\
& QDFM (Ours)
& 0.90 $\pm$ 0.03
& 0.18 $\pm$ 0.03
& \textbf{15} \\
\midrule
\multirow{2}{*}{Resource Gathering}
& CQL (Scalar)
& 0.67 $\pm$ 0.10
& -- 
& 1 \\
& QDFM (Ours)
& \textbf{0.82 $\pm$ 0.05}
& 0.07 $\pm$ 0.02
& \textbf{4} \\
\bottomrule
\end{tabular}%
}
\end{minipage}
\hfill
\begin{minipage}[t]{0.39\linewidth}

\centering
\caption{Comparison with MODULI on MO-Hopper. We report hypervolume (HV) and percentage of the dataset Pareto upper bound (UB).}
\label{tab:moduli_comparison}
\begingroup
\setlength{\tabcolsep}{1.5pt}
\renewcommand{\arraystretch}{1.35}
\small
\begin{tabular}{lcc}
\toprule
Method & HV ($\times 10^7 \uparrow$) & \% of UB ($\uparrow$) \\
\midrule
MODULI & 2.025 & 96.02\% \\
QDFM (Ours) & \textbf{2.085} & \textbf{98.88\%} \\
Upper Bound & 2.109 & 100\% \\
\bottomrule
\end{tabular}
\endgroup
\end{minipage}

\end{table}

\subsection{Offline Multi-Agent RL (SMAC)}
\label{main:smac}

We further evaluate our method on StarCraft Multi-Agent Challenge (SMAC)~\citep{samvelyan2019starcraft}, focusing on decentralized control with combinatorial joint action spaces and compare against SOTA offline MARL baselines including CFCQL~\citep{shao2023counterfactual}, OMAR~\citep{pan2022plan}, OMIGA~\citep{wang2023offline}, MACQL~\citep{kumar2020conservative} and BC. QDFM achieves competitive return across the tested maps,
showing that the generative CTMC policy can be used at multi-agent scale with
standard value factorization but we do not claim a return or factorization
advantage over existing MARL methods. The separate benefit demonstrated by
our policy representation is preference-conditioned stochastic coordination,
evaluated in Appendix~\ref{app:marl}. SMAC dataset details and win rates are provided in Appendices~\ref{app:smac_datasets} and
\ref{app:smac_winrates}.

\begin{table*}[ht]
\centering
\caption{Offline MARL performance on SMAC benchmarks (Return $\uparrow$). Results are averaged over 5 seeds.}
\label{tab:smac_returns}
\small
\resizebox{\textwidth}{!}{%
\begin{tabular}{llcccccc}
\toprule
Map & Dataset & BC & MACQL & CFCQL & OMAR & OMIGA & QDFM \\
\midrule

\multirow{1}{*}{2s3z} 
& medium-replay 
& 16.71 $\pm$ 0.09 
& 15.30 $\pm$ 1.27 
& 16.79 $\pm$ 0.27 
& 15.10 $\pm$ 1.31 
& 16.27 $\pm$ 0.73 
& \textbf{17.06 $\pm$ 0.39} \\

% \midrule

% \multirow{1}{*}{3m} 
% & medium 
% & 9.60 $\pm$ 0.16 
% & 11.69 $\pm$ 1.15 
% & 11.75 $\pm$ 0.32 
% & 8.84 $\pm$ 0.07 
% & \textbf{11.92 $\pm$ 0.24} 
% & 11.02 $\pm$ 1.10 \\

\midrule

\multirow{4}{*}{5m\_vs\_6m}
& expert 
& 12.13 $\pm$ 0.23 
& 12.17 $\pm$ 1.47 
& \textbf{15.05 $\pm$ 0.67} 
& 11.32 $\pm$ 0.43 
& 14.95 $\pm$ 1.52 
& 14.99 $\pm$ 0.19 \\

& medium 
& 11.21 $\pm$ 0.08 
& 11.26 $\pm$ 0.20 
& \textbf{11.59 $\pm$ 1.23} 
& 11.01 $\pm$ 1.84 
& 11.48 $\pm$ 0.73 
& 11.58 $\pm$ 1.55 \\

& medium-replay
& 6.89 $\pm$ 0.11
& 6.66 $\pm$ 0.13
& 6.76 $\pm$ 0.11
& 6.83 $\pm$ 0.42
& 6.74 $\pm$ 0.13
& \textbf{6.94 $\pm$ 0.25} \\

& poor 
& 1.30 $\pm$ 0.25 
& 1.54 $\pm$ 0.14 
& 1.41 $\pm$ 0.05 
& \textbf{3.43 $\pm$ 0.15} 
& 1.24 $\pm$ 0.02 
& 1.67 $\pm$ 0.29 \\

\bottomrule
\end{tabular}
}
\end{table*}

\section{Conclusion}
\label{conclusion}

We presented QDFM, a generative policy framework for offline RL that replaces continuous
probability flows with a CTMC over discrete actions. By learning transition rates
via discrete flow matching and guiding them with Q-weighted endpoint reweighting,
our method supports
preference-conditioned and multi-agent extensions and preserves multiple
high value modes in multimodal offline data while achieving competitive
average return on the evaluated benchmarks.

\textbf{Limitations and future work.}
QDFM incurs higher inference cost than single-pass discrete policies because it
requires CTMC simulation. Our exact policy recovery result characterizes an oracle single improvement
step with endpoints drawn from a fixed behavior distribution, whereas the
implemented procedure samples endpoints from the current policy and
iterates the update. Thus,
it is not a convergence guarantee for the full training procedure.
Discretized MuJoCo is also an imperfect proxy for intrinsically discrete control and may be
affected by known offline locomotion confounds such as survival bias. The weaker HalfCheetah results have not been causally isolated.
Finally, our multimodality advantage is shown on a smaller scale and validating this on more complex domains and scaling
to larger discrete action spaces such as language token selection are natural next
steps.

\begin{ack}
This work was supported in part by NVIDIA Academic Grant Program,
NIH award R16GM159671, and NSF grants 2537471 and 2112471.
We thank the University of Kentucky Center for Computational
Sciences and Information Technology Services Research Computing
for their support and use of the Lipscomb Compute Cluster and
associated research computing resources.
The content is solely the responsibility of the authors and
does not necessarily represent the official views of the
funding agencies. The authors declare no competing financial interests.
\end{ack}

\small
\bibliographystyle{plainnat}
\bibliography{references}

%%%%%%%%%%%%%%%%%%%%%%%%%%%%%%%%%%%%%%%%%%%%%%%%%%%%%%%%%%%%

\appendix

% APPENDIX

\newpage
\appendix
\onecolumn

\section{Theory and Proofs}
\label{sec:appendix_proofs}

\subsection{Conditional generator derivation and CTMC validity}
\label{app:ctmc_cdfm_details}

\paragraph{Time derivative of the path.}
For any $a'\in\mathcal A$, differentiating
\eqref{eq:dfm-mixture-path} yields
\begin{equation}
\label{eq:dfm-path-derivative}
\frac{d}{dt}p_{t\mid Z}(a')
=
\delta(a',A_1)-\delta(a',A_0).
\end{equation}
This derivative is fully determined by the chosen path and involves
no modeling choices.

\paragraph{Kolmogorov forward equation.}
Let $u_t(a',a\mid Z)$ denote time-dependent CTMC rates.
The Kolmogorov equation requires
\begin{equation}
\label{eq:dfm-kolmogorov}
\frac{d}{dt}p_{t\mid Z}(a')
=
\sum_{a\in\mathcal A}
u_t(a',a\mid Z)\,p_{t\mid Z}(a).
\end{equation}
Since $p_{t\mid Z}$ has support only on $\{A_0,A_1\}$,
the sum reduces to
\begin{equation}
\label{eq:dfm-reduced-sum}
\frac{d}{dt}p_{t\mid Z}(a')
=
(1-t)\,u_t(a',A_0\mid Z)
+
t\,u_t(a',A_1\mid Z).
\end{equation}

\paragraph{Endpoint-absorbing design choice.}
We impose the standard DFM constraint that the endpoint is absorbing (since no jumping to another point as we are at endpoint):
\[
u_t(a',A_1\mid Z)=0
\qquad \forall\,a'.
\]
Under this constraint, \eqref{eq:dfm-reduced-sum} becomes
\begin{equation}
\label{eq:dfm-reduced}
\frac{d}{dt}p_{t\mid Z}(a')
=
(1-t)\,u_t(a',A_0\mid Z).
\end{equation}

\paragraph{Solving for the conditional rates.}
Equating \eqref{eq:dfm-reduced} with the true derivative
\eqref{eq:dfm-path-derivative} gives
\[
(1-t)\,u_t(a',A_0\mid Z)
=
\delta(a',A_1)-\delta(a',A_0),
\]
which admits the following valid CTMC generator (jump-to-endpoint) choice
\begin{equation}
\label{eq:dfm-conditional-rate-A0}
u_t(a',A_0\mid Z)
=
\frac{1}{1-t}
\Big(
\delta(a',A_1)-\delta(a',A_0)
\Big).
\end{equation}
Combining this with the absorbing condition at $A_1$, the conditional
rate kernel can be written compactly as
\begin{equation}
\label{eq:dfm-conditional-rate}
u_t(a',a\mid Z)
=
\frac{1}{1-t}
\Big(
\delta(a',A_1)-\delta(a',a)
\Big).
\end{equation}

If the chain is currently at $a$, the only positive off-diagonal rate
is a jump toward the endpoint $A_1$.
The diagonal term is negative to ensure each column sums to zero,
as required for valid CTMC rates.
Once $A_1$ is reached, the process remains there.

\paragraph{Regularization near the terminal time.}
The ideal jump-to-endpoint target above scales as $1/(1-t)$ and therefore has
a divergent time-integrated second moment. Indeed, conditional on
$A_0\neq A_1$, the probability that the process has not yet reached $A_1$ at
time $t$ is $1-t$. Since the rate vector then contains two nonzero entries of
magnitude $1/(1-t)$, its conditional second moment scales as
$2/(1-t)$ and is not integrable at $t=1$.

In the implementation we therefore use
\begin{equation}
\label{eq:regularized_conditional_rate}
u_t^{(\epsilon)}(a',a\mid Z)
=
\frac{1-t}{1-t+\epsilon}\,
u_t(a',a\mid Z),
\qquad \epsilon=10^{-4},
\end{equation}
which is exactly Eq.~\eqref{eq:jump_to_endpoint_eps}. Its time-integrated
conditional target second moment is finite:
\begin{align}
\int_0^1
2\frac{1-t}{(1-t+\epsilon)^2}\,dt
&=
2\left[
\log\!\left(\frac{1+\epsilon}{\epsilon}\right)
+\frac{\epsilon}{1+\epsilon}-1
\right]
<\infty .
\label{eq:finite_target_second_moment}
\end{align}
For $\epsilon=10^{-4}$, this quantity is approximately $16.42$.

The corresponding guided marginal target is
\[
\tilde u_t^{(\epsilon)}(\cdot,x)
=
\mathbb{E}\!\left[
u_t^{(\epsilon)}(\cdot,x\mid Z)
\,\middle|\, X_t=x
\right]
=
c_\epsilon(t)\tilde u_t(\cdot,x),
\]
where $c_\epsilon(t)=(1-t)/(1-t+\epsilon)$. By Jensen's inequality,
its second moment is bounded by the corresponding conditional second moment
and is therefore finite as well. Since $c_\epsilon(t)$ is independent of
$\theta$, the marginalization argument used in
Theorem~\ref{thm:guided_grad_equiv_main} applies unchanged to the regularized
target, so the conditional and marginal objectives remain gradient-equivalent.

For fixed $t<1$, $u_t^{(\epsilon)}\to u_t$ as $\epsilon\to0$.
For finite $\epsilon$, the regularized continuous-time jump process leaves
residual non-endpoint probability
\[
\exp\!\left(
-\int_0^1\frac{dt}{1-t+\epsilon}
\right)
=
\frac{\epsilon}{1+\epsilon},
\]
which vanishes as $\epsilon\to0$. Thus exact endpoint recovery is the
$\epsilon\to0$ oracle limit, while the implementation uses a finite
regularization for numerical stability.

\paragraph{Validity of the conditional rate field.}
A central requirement for discrete flow matching is that the conditional rate field
$u_t(\cdot,\cdot\mid Z)$ defines a valid Continuous-Time Markov Chain (CTMC) generator.
In particular, for every conditioning variable $Z$ and every current action $a\in\mathcal{A}$,
the rates must satisfy the generator constraints:
(i) non-negativity of off-diagonal entries, $u_t(a',a\mid Z)\ge 0$ for all $a'\neq a$,
and (ii) conservation of probability mass, $\sum_{a'\in\mathcal{A}} u_t(a',a\mid Z)=0$.
While compact expressions for conditional generators are often convenient, they do not make these constraints explicit and may yield negative off-diagonal entries for some $(a',a)$,
in which case the resulting operator does not correspond to a valid CTMC generator

Therefore, to ensure correctness by construction, we explicitly define the conditional rate field in a
piecewise form in Eq.~\eqref{eq:jump_to_endpoint_main} that enforces the CTMC constraints for all $t\in[0,1]$ and all realizations of $Z$.
To the best of our knowledge, this is the first energy-guided, preference-conditioned discrete
policy optimization framework to explicitly enforce CTMC generator validity at the level of the
conditional rate field, rather than relying on implicit or relaxed formulations.

\paragraph{Conditional CTMC rates.}
To construct a valid conditional probability flow that exactly interpolates between
the source action $A_0$ and the target endpoint $A_1$, we define a time-dependent
CTMC generator whose dynamics concentrate all probability mass at $A_1$ as
$t \to 1$ while satisfying the generator constraints by construction.
Concretely, we use the following jump-to-endpoint rate field:
\[
u_t(a',a \mid Z) =
\begin{cases}
\frac{1}{1-t}, & a \neq A_1,\; a' = A_1, \\
-\frac{1}{1-t}, & a \neq A_1,\; a' = a, \\
0, & \text{otherwise},
\end{cases}
\]
with $A_1$ absorbing.
This generator induces the desired conditional path and guarantees nonnegative
off-diagonal rates and column-wise mass conservation, as required for a valid
CTMC \citep{gat2024discrete,lipman2024flow}. This is the ideal path-generating rate; the optimization in our experiments
uses its $\epsilon$-regularized version in
Eq.~\eqref{eq:regularized_conditional_rate}.

\paragraph{Parameterization of the rate function.}
We enforce the CTMC generator constraints by construction. For each current action $a$, the network outputs unconstrained scores $g_{\theta,t}(a',a \mid s,\omega)$ for all $a' \neq a$. These are mapped through a nonnegative function (softplus) to obtain valid off-diagonal rates:
\[
u_{\theta,t}(a',a \mid s,\omega) = \mathrm{softplus}\big(g_{\theta,t}(a',a \mid s,\omega)\big), \quad a' \neq a.
\]
The diagonal entries are then defined to ensure zero column sum:
\[
u_{\theta,t}(a,a \mid s,\omega) = -\sum_{a' \neq a} u_{\theta,t}(a',a \mid s,\omega).
\]
This guarantees nonnegative off-diagonal rates and that $\sum_{a'} u_{\theta,t}(a',a \mid s,\omega)=0$, so $u_{\theta,t}$ is a valid CTMC generator by construction. The only approximation arises during inference (Algorithm~\ref{alg:ctmc_inference}), where we simulate the continuous-time process using an Euler discretization.

\begin{figure}[t]
\centering

\definecolor{accent}{HTML}{2D6A9F}
\definecolor{accenttwo}{HTML}{D4563A}
\definecolor{nodeblue}{HTML}{E8F0FA}
\definecolor{nodegray}{HTML}{F0F0F0}
\definecolor{bordergray}{HTML}{888888}
\definecolor{darktext}{HTML}{2B2B2B}
\definecolor{softblack}{HTML}{333333}
\definecolor{greenfill}{HTML}{E6F4EA}
\definecolor{greenbd}{HTML}{3A8A5C}
\definecolor{purplefill}{HTML}{EDE7F6}
\definecolor{purplebd}{HTML}{6A4C9C}

\resizebox{\textwidth}{!}{%
\begin{tikzpicture}[
    >=Stealth,
    every node/.style={font=\small, text=darktext},
    state/.style={
        circle, draw=softblack, line width=0.8pt,
        fill=nodeblue, minimum size=22pt,
        inner sep=0pt, font=\normalsize\bfseries
    },
    ratelabel/.style={font=\footnotesize, text=accent, fill=white,
        inner sep=1.5pt, rounded corners=1pt},
    algobox/.style={
        rectangle, rounded corners=3pt, draw=bordergray,
        line width=0.7pt, fill=nodegray, inner sep=5pt,
        font=\footnotesize, align=center, text=darktext
    },
    panellabel/.style={font=\large\bfseries, text=softblack,
        align=center},
    subcap/.style={font=\footnotesize, text=bordergray,
        align=center, text width=4cm},
]

% ==================================================================
%  (a) CTMC rate field
% ==================================================================
\begin{scope}[local bounding box=panelA]
    \node[state] (a)   at (0,0)      {$a$};
    \node[state] (ap)  at (2.6,0)    {$a'$};
    \node[state] (app) at (0,-2.6)   {$a''$};

    \draw[->, line width=0.7pt, color=accent, bend left=18]
        (a) to node[ratelabel, above, yshift=1pt]
            {$u_t(a',a)$} (ap);
    \draw[->, line width=0.7pt, color=accenttwo, bend left=18]
        (ap) to node[ratelabel, below, yshift=-1pt]
            {$u_t(a,a')$} (a);

    \draw[->, line width=0.7pt, color=accent, bend left=15]
        (app) to node[ratelabel, left, xshift=-1pt]
            {$u_t(a''\!,a)$} (a);
    \draw[->, line width=0.7pt, color=accenttwo, bend left=15]
        (a) to node[ratelabel, right, xshift=1pt]
            {$u_t(a,a'')$} (app);

    \node[subcap, anchor=north] at (1.3,-3.5)
        {Total outflow $=$ total inflow\\[1pt](mass conservation)};
\end{scope}

% ==================================================================
%  (b) Jump-to-endpoint field
% ==================================================================
\begin{scope}[xshift=6cm, local bounding box=panelB]
    \node[state, fill=nodegray] (A0) at (0,0) {$A_0$};
    \node[state, draw=greenbd, fill=greenfill] (A1)
        at (2.0,-2.0) {$A_1$};

    % Arrow A0 -> A1 (draw first, label placed separately)
    \draw[->, line width=0.9pt, color=accent] (A0) -- (A1);
    % Label to the right of the arrow, offset so it doesn't touch the line
    \node[ratelabel, anchor=west] at (1.15,-0.7)
        {$\dfrac{1}{1-t}$};

    % Self-loop on A1
    \draw[->, dashed, line width=0.6pt, color=bordergray]
        (A1) to [out=30, in=-30, looseness=4]
        node[right, font=\footnotesize, text=bordergray,
             xshift=2pt] {stay} (A1);

    % "a != A1" text and arrow
    \node[font=\footnotesize, text=accenttwo] (neq)
        at (-0.5,-3.6) {$a \neq A_1$};
    \draw[->, line width=0.7pt, color=accenttwo] (neq) -- (A1);
    % Label placed clearly above the arrow line
    \node[ratelabel, fill=none, anchor=south] at (0.35,-2.65)
        {$\dfrac{1}{1-t}$};

    \node[subcap, anchor=north] at (0.9,-4.4)
    {Eq.~\eqref{eq:jump_to_endpoint_main}: \mbox{probability} mass\\jumps
     directly toward the\\target endpoint $A_1$};
\end{scope}

% ==================================================================
%  (c) Inference step
% ==================================================================
\begin{scope}[xshift=12.5cm, local bounding box=panelC]
    \node[algobox, draw=purplebd, fill=purplefill] (cur) at (0,0)
        {current action $a = A_{t_n}$};

    \node[algobox, below=0.5cm of cur] (lam)
        {$\lambda_n = \displaystyle\sum_{a'\neq a}
          u_{\theta,t_n}(a',a)$};
    \draw[->, line width=0.7pt, color=softblack] (cur) -- (lam);

    \node[algobox, draw=greenbd, fill=greenfill,
          below=1.0cm of lam, xshift=-1.8cm] (stay)
        {stay at $a$\\[1pt]prob.\ $1 - h\lambda_n$};

    \node[algobox, draw=accent, fill=nodeblue,
          below=1.0cm of lam, xshift=1.8cm] (jump)
        {jump to $a'$\\[1pt]prob.\ $h\lambda_n$\\[2pt]
         sample $a' \!\sim\!
           \dfrac{u_{\theta,t_n}(\cdot,\,a)}{\lambda_n}$};

    \draw[->, line width=0.7pt, color=greenbd]
        (lam.south) -- ++(0,-0.3) -| (stay.north);
    \draw[->, line width=0.7pt, color=accent]
        (lam.south) -- ++(0,-0.3) -| (jump.north);

    \node[subcap, text width=5.5cm, anchor=north]
        at ($(stay.south)!0.5!(jump.south)+(0,-0.45)$)
        {Algorithm~\ref{alg:ctmc_inference}: \mbox{Euler} approximation\\
         of the CTMC dynamics};
\end{scope}

% ==================================================================
%  Panel labels — horizontally aligned
% ==================================================================
\coordinate (labely) at (0,-6.6);
\node[panellabel] at (panelA.center |- labely)
    {(a) CTMC rate field};
\node[panellabel] at (panelB.center |- labely)
    {(b) Jump-to-endpoint field};
\node[panellabel] at (panelC.center |- labely)
    {(c) Inference step};

\end{tikzpicture}%
}% end resizebox

\caption{Intuition for the CTMC formulation.
\textbf{(a)}~The rate field specifies how probability mass flows
between actions.
\textbf{(b)}~The jump-to-endpoint target field moves mass directly
toward the endpoint.
\textbf{(c)}~Inference simulates the learned CTMC in discrete time
via stay-or-jump updates.}
\label{fig:ctmc-intuition}
\end{figure}

\paragraph{Why use the jump-to-endpoint rate field?}
Although many rate fields can satisfy the Kolmogorov forward equation and the CTMC generator constraints, our jump-to-endpoint construction provides a simple and stable parameterization with a closed-form conditional target. Instead of learning arbitrary pairwise transitions between all actions, the conditional process moves probability mass directly from the current action toward the endpoint $A_1$. This avoids unnecessary transitions, reduces the complexity of the target rate field, and makes the flow matching regression problem easier to learn in practice.

\paragraph{Meaning of the CTMC rate field and generator constraints.}
The CTMC rate field $u_t(a',a)$ defines the instantaneous rate of transitioning from the current action $a$ to another action $a'$. In other words, it governs how probability mass moves over the discrete action space as time evolves. The Kolmogorov forward equation in Eq.~\eqref{eq:dfm-kolmogorov} describes this evolution of probability mass under the rate field. The constraints given in Section~\ref{sec:dfm_prelims} ensure that the dynamics define a valid CTMC: off-diagonal entries must be nonnegative transition rates, while the diagonal entry balances the total outgoing mass so that each column sums to zero. This zero column sum condition enforces probability conservation, ensuring that total probability is preserved over time.

\paragraph{Connection between the rate field, CTMC validity, and inference.}
Eq.~\eqref{eq:dfm-kolmogorov} defines the probability flow induced by the CTMC rate field, while the generator constraints ensure that this flow corresponds to valid probability dynamics. Algorithm~\ref{alg:ctmc_inference} then simulates these dynamics in discrete time using stay-or-jump updates. At each step, the total outgoing rate determines the probability of leaving the current action, and the normalized off-diagonal rates determine which action is selected after a jump. Thus, the theory and inference procedure are directly connected: the Kolmogorov equation defines the continuous-time flow, the generator constraints make it a valid CTMC, and Algorithm~\ref{alg:ctmc_inference} provides a practical Euler simulation of the resulting action evolution. Figure~\ref{fig:ctmc-intuition} illustrates this relationship.

\paragraph{Alternative valid rate fields.}
Other valid rate-field constructions are possible. One option is a fully parameterized generator, where a neural network outputs a transition rate for every pair of actions $(a,a')$. The scores are mapped to nonnegative off-diagonal rates, and the diagonal entries are set to enforce the zero column sum constraint. While this is general, it requires learning all pairwise transitions and scales poorly with the action-space size, since the number of transition rates grows as $O(|\mathcal A|^2)$.

\paragraph{Energy-based rate fields.}
Another option is an energy-based construction, where transitions are defined indirectly through a learned energy function rather than by modeling each pairwise transition explicitly. For example, transition rates can be induced using a Boltzmann form over candidate next actions. However, valid transition probabilities require normalization over possible actions, which can be expensive for large action spaces and can introduce additional variance compared to directly specified conditional rate fields.

\paragraph{Mathematical advantage of the chosen field.}
The main advantage of our construction is simplicity, which makes the algorithm faster and more stable. Under the Kolmogorov forward dynamics, our generator preserves probability mass by construction: all mass is transported directly toward $A_1$ at rate $1/(1-t)$, and the diagonal term balances this outgoing flow. When the process reaches $A_1$, all transition rates become zero, so $A_1$ is absorbing.

\paragraph{Closed-form target and guaranteed endpoint arrival.}
The jump-to-endpoint construction provides a closed-form conditional target given $Z=(s,\omega,A_0,A_1)$, so the model regresses directly to known transition rates rather than learning the rate-field structure from data. It also avoids the need to learn arbitrary pairwise transitions, reducing the conditional target complexity from $O(|\mathcal A|^2)$ to $O(|\mathcal A|)$. With jump rate $1/(1-t)$, the survival probability of not jumping by time $t$ is
\[
\exp\left(-\int_0^t \frac{1}{1-\tau}\,d\tau\right)
= 1-t.
\]
Therefore, the probability of reaching $A_1$ by time $t$ is $t$, and as $t\to 1$, the conditional process reaches $A_1$ almost surely. Thus, the conditional process reaches the endpoint at terminal time, while the marginal process recovers the endpoint distribution.

\paragraph{Empirical comparison with alternative rate fields.}
To further justify our choice of rate field, we compare it with two alternative valid constructions: a diffuse generator and an energy-based Boltzmann generator. Both alternatives underperform in practice. Energy-based methods require normalization and show higher inference cost, while diffuse dynamics introduce unnecessary transitions and large training overhead. In contrast, our jump-to-endpoint construction directly transports mass to the target, enabling more stable and efficient learning.

\begin{table}[h]
\centering
\small
\setlength{\tabcolsep}{6pt}
\begin{tabular}{lcc}
\toprule
Method & Train Time (s) & Act Time (ms) \\
\midrule
QDFM & 52.41 & 7.27 \\
Energy-based & 52.93 & 14.87 \\
Diffuse-based & 6913.83 & 7.50 \\
\bottomrule
\end{tabular}
\caption{Runtime comparison between the jump-to-endpoint rate field used by QDFM and alternative valid CTMC rate-field constructions.}
\label{tab:rate_field_runtime}
\end{table}

These results show that the chosen jump-to-endpoint rate field achieves a good balance of simplicity, stability, and computational cost. It is mathematically valid by construction, has a closed-form conditional target, guarantees endpoint arrival as $t\to 1$, and avoids the unnecessary complexity of fully parameterized or normalized energy-based rate fields. This rate-field comparison is a matched-budget experiment run separately from our training-time ablations or baseline benchmarks.

\subsection{Critic Backup and In-Support Approximation Details}
\label{app:critic_details}

This section provides additional details on the critic backup used in our
multi-objective offline reinforcement learning setup.
The material here is included for completeness and to clarify implementation
choices; the main text only relies on the resulting update rule.

\paragraph{Soft Bellman backup under KL regularization.}
For a fixed state $s$ and preference vector $\omega$, the KL-regularized optimal
policy has the Boltzmann form
\[
\pi^\star_\omega(a \mid s)
\propto
\mu(a \mid s)\exp\!\big(\beta Q_\omega(s,a)\big),
\]
where $\mu(\cdot \mid s)$ is the behavior policy and $\beta > 0$ controls the
strength of value guidance.
Under this policy, the corresponding soft value function is given by the
expectation
\[
V_\omega(s)
=
\mathbb{E}_{a \sim \pi^\star_\omega(\cdot \mid s)}
\big[ Q_\omega(s,a) \big].
\]

Substituting the Boltzmann form of $\pi^\star_\omega$ yields
\begin{equation}
\label{eq:soft_value_exact}
V_\omega(s)
=
\frac{
\sum_{a \in \mathcal{A}} \mu(a \mid s)\exp\!\big(\beta Q_\omega(s,a)\big)\,Q_\omega(s,a)
}{
\sum_{a \in \mathcal{A}} \mu(a \mid s)\exp\!\big(\beta Q_\omega(s,a)\big)
}.
\end{equation}
This expression corresponds to a soft (log-sum-exp–weighted) backup and is
standard in KL-regularized reinforcement learning
(e.g., \citet{todorov2006linearly,haarnoja2017reinforcement}).

\paragraph{In-support approximation.}
In offline reinforcement learning with large or discrete action spaces,
computing the full sums in Eq.~\eqref{eq:soft_value_exact} is typically infeasible.
Moreover, evaluating $Q_\omega(s,a)$ for actions outside the support of the
dataset can lead to severe extrapolation error.

To address this, we adopt an \emph{in-support approximation} in which the
expectations are estimated using a finite set of actions sampled from the
behavior policy.
Specifically, for a given next state $s'$, we sample a support set
\[
\mathcal{A}_{\mathrm{supp}}(s')
=
\{ a'_1, \dots, a'_M \}
\sim \mu(\cdot \mid s'),
\]
and approximate the soft value as
\begin{equation}
\label{eq:soft_value_support}
V_\omega(s')
\;\approx\;
\frac{
\sum_{j=1}^M \exp\!\big(\hat{\beta} Q_\omega(s',a'_j)\big)\,Q_\omega(s',a'_j)
}{
\sum_{j=1}^M \exp\!\big(\hat{\beta} Q_\omega(s',a'_j)\big)
}.
\end{equation}

This approximation preserves the KL-regularized structure of the backup while
ensuring that all evaluated actions lie within the empirical support of the
offline dataset. Here $\hat{\beta}$ denotes the finite-sample guidance temperature used in the
self-normalized Monte Carlo approximation in Eq.~\eqref{eq:q_weight}.
In the ideal infinite-support limit it coincides with the theoretical inverse temperature $\beta$.

\paragraph{Bellman regression target.}
Using the approximation above, the scalarized Bellman target used to train the
critic is
\[
y
=
\langle \omega, \mathbf{r} \rangle
+
\gamma V_\omega(s'),
\]
where $\mathbf{r}$ is the observed reward vector.
This target is used in a standard squared regression loss for the scalarized
critic values, while the underlying vector-valued critic
$\mathbf{Q}_\psi$ is shared across preferences.

\subsection{CTMC Inference: Test-Time Sampling for Arbitrary Preferences}
\label{app:ctmc_inference}

After training, the learned CTMC induces a discrete policy
$\pi_\theta(\cdot\mid s,\omega)$ as its terminal distribution.
This section describes the Euler-type simulation procedure used to sample actions
at test time for arbitrary preference vectors.

Given a state $s$ and a user preference $\omega$, we generate one action
by simulating the CTMC from $t=0$ to $t=1$.

\paragraph{Inputs.}
A trained rate model $u_{\theta,t}(a',a\mid s,\omega)$ satisfying the CTMC
rate constraints, behavior policy
$\mu(\cdot\mid s)$, and a time step $h>0$.

\paragraph{Euler-type CTMC simulation.}
We discretize time $n=0,1,\dots,N \quad\text{with } Nh=1
$ and
simulate one trajectory $(A_{t_n})_{n=0}^N$ as follows:
\begin{enumerate}
  \item \textbf{Initialize.} Sample an initial action
  \[
  A_{0}\sim \mu(\cdot\mid s)
  \]
  \item \textbf{Iterate.} For $n=0,1,\dots,N-1$, given the current action
  $A_{t_n}=a$, compute the \emph{total leaving rate}
  \[
  \lambda_n \;:=\; \sum_{a'\neq a} u_{\theta,t_n}(a',a\mid s,\omega),
  \]
  and form the corresponding \emph{jump distribution}
  \[
  \rho_n(a'\mid a)
  \;:=\;
  \frac{u_{\theta,t_n}(a',a\mid s,\omega)}{\lambda_n},
  \qquad a'\neq a,
  \]
  (if $\lambda_n=0$, define $\rho_n$ arbitrarily and the chain stays).
  Then perform one Euler step:
  \begin{enumerate}
    \item With probability $1-h\lambda_n$, set $A_{t_{n+1}}=a$
    (no jump).
    \item With probability $h\lambda_n$, draw $A_{t_{n+1}}\sim \rho_n(\cdot\mid a)$
    (jump to a new action).
  \end{enumerate}
  \item \textbf{Output.} Return $A_{1}:=A_{t_N}$ as the sampled action.
\end{enumerate}
The resulting $A_1$ is a sample from the learned policy
\(
A_1\sim \pi_\theta(\cdot\mid s,\omega).
\)
\paragraph{Step-size condition.}
The Euler simulation is valid provided the step size $h>0$ satisfies
$h\,\lambda_n \le 1$ for all $n$, ensuring that the jump probability
$h\lambda_n$ lies in $[0,1]$. In practice, this can be enforced by choosing
$h$ sufficiently small (we use this) or by clamping the total leaving rate.

\paragraph{Remarks.}
This simulation is the discrete analogue of integrating an ODE in continuous
flow-based policies. The probability of a jump in each step is proportional
to the total leaving rate $\lambda_n$; the destination of the jump is chosen
proportionally to the off-diagonal rates $u_{\theta,t_n}(a',a\mid s,\omega)$.
Smaller $h$ yields a more accurate approximation of the continuous-time chain.

\subsection{Proof of Proposition~\ref{prop:bregman_affine}}
\label{sec:proof_bregman_affine}

\textbf{Bregman divergences:} Let $\phi_x$ be a differentiable and strictly convex function. The Bregman divergence $D_x(y, v)$ is defined as the difference between the value of $\phi_x$ at $y$ and its first-order Taylor approximation at $v$:
\begin{equation}
\label{eq:bregman}
    D_x(y, v) := \phi_x(y) - \Big( \phi_x(v) + \langle \nabla \phi_x(v), y - v \rangle \Big).
\end{equation}
Geometrically, this measures the vertical distance between the convex function $\phi_x$ and its tangent hyperplane at $v$, evaluated at the point $y$.

\begin{proposition}[Affine invariance of the Bregman gradient]
\label{prop:bregman_affine}
For any integrable random variable $Y$ and any fixed vector $v$, the expected gradient of the Bregman divergence with respect to $v$ is equal to the gradient of the divergence at the expected value of $Y$:
\begin{equation}
    \E_Y [\nabla_v D_x(Y, v)] = \nabla_v D_x(\E[Y], v).
\end{equation}
\end{proposition}

\begin{proof}
\textbf{Step 1: Compute the gradient $\nabla_v D_x(y, v)$.} \\
We differentiate the definition of $D_x(y, v)$ with respect to $v$, treating $y$ as a constant.
\begin{align*}
    \nabla_v D_x(y, v) &= \nabla_v \Big( \phi_x(y) - \phi_x(v) - \langle \nabla \phi_x(v), y - v \rangle \Big) \\
    &= 0 - \nabla \phi_x(v) - \nabla_v \Big( \langle \nabla \phi_x(v), y - v \rangle \Big).
\end{align*}
Using the product rule for the inner product term (or simply expanding the gradient of $\langle f(v), g(v) \rangle$):
\[
\nabla_v \langle \nabla \phi_x(v), y - v \rangle
=
\nabla^2 \phi_x(v)(y - v) - \nabla \phi_x(v).
\]
Substituting this back:
\begin{align*}
    \nabla_v D_x(y, v) &= -\nabla \phi_x(v) - \Big( \nabla^2 \phi_x(v)(y - v) - \nabla \phi_x(v) \Big) \\
    &= -\nabla \phi_x(v) - \nabla^2 \phi_x(v)(y - v) + \nabla \phi_x(v) \\
    &= \nabla^2 \phi_x(v)(v - y).
\end{align*}

\textbf{Step 2: Take the expectation over $Y$.} \\
Since $v$ is fixed, the Hessian $\nabla^2 \phi_x(v)$ is constant with respect to the expectation. By the linearity of expectation:
\begin{align*}
    \E_Y [\nabla_v D_x(Y, v)] &= \E_Y \Big[ \nabla^2 \phi_x(v)(v - Y) \Big] \\
    &= \nabla^2 \phi_x(v) \Big( v - \E[Y] \Big).
\end{align*}

\textbf{Step 3: Recognize the form.} \\
The result $\nabla^2 \phi_x(v)(v - \E[Y])$ is exactly the formula derived in Step 1, but with $\E[Y]$ replacing $y$. Therefore:
\[
    \nabla^2 \phi_x(v)(v - \E[Y]) = \nabla_v D_x(\E[Y], v).
\]
\end{proof}

\subsection{Chain Rule for the Loss (used in Theorem~\ref{thm:guided_grad_equiv_main})}
\label{sec:prop_chain_rule}

\begin{proposition}[Chain rule]
\label{prop:chain_rule}
Let the loss function be $J(\theta) := D_x(y, u_{\theta,t}(\cdot, x))$, where $y$ is a fixed target independent of $\theta$. Then:
\begin{equation}
    \nabla_\theta J(\theta) = \nabla_v D_x(y, v) \Big|_{v=u_{\theta,t}(\cdot,x)} \cdot \nabla_\theta u_{\theta,t}(\cdot, x).
\end{equation}
\end{proposition}

\begin{proof}
Let $v(\theta) := u_{\theta,t}(\cdot, x)$. We define $g(v) := D_x(y, v)$ so that $J(\theta) = g(v(\theta))$.
Applying the standard multivariate chain rule:
\[
    \nabla_\theta J(\theta) = \frac{\partial g}{\partial v} \cdot \frac{\partial v}{\partial \theta} = \nabla_v D_x(y, v) \Big|_{v=v(\theta)} \cdot \nabla_\theta u_{\theta,t}(\cdot, x).
\]
Even though $D_x$ takes two arguments, the derivative with respect to the first argument $y$ is multiplied by $\frac{\partial y}{\partial \theta}$. Since the target $y$ comes from the dataset and is independent of $\theta$, $\frac{\partial y}{\partial \theta} = 0$, so that term vanishes.
\end{proof}

\subsection{Proof of Theorem~\ref{thm:guided_grad_equiv_main}: Gradient Equivalence}
\label{sec:proof_thm1}

Here we prove that the gradient of the \textbf{Guided Conditional Loss} is equivalent to the gradient of the \textbf{Guided Marginal Loss}.

Assume the guidance weight $w(Z)$ is independent of $\theta$. Then:
\[
\nabla_\theta \mathcal{L}_{t}^w(\theta) = \nabla_\theta \mathcal{L}_{t,\mathrm{Marg}}^w(\theta).
\]

\begin{proof}
We prove the equality for a fixed time $t$. The full result follows by linearity of expectation over $t$.
We use the $\ell_2$ loss for the experiments and when defining the conditional (Eq.~\eqref{eq:qcdfm}) and marginal (Eq.~\eqref{eq:marg_loss}) losses, which is a special case of a Bregman divergence (Eq.~\eqref{eq:bregman}).
Therefore, all theoretical results stated for general Bregman divergences apply directly.

\subsection*{Step 1: Change of Measure}
We start with the definition of the guided conditional loss represented as a Bregman divergence adapted from \cite{lipman2024flow}.
\[
\mathcal{L}_t^w(\theta)
=
\E_{(X_t, Z) \sim p_t}
\Big[
w(Z)\,
D_{X_t}\big(u_t(\cdot, X_t \mid Z),\, u_{\theta,t}(\cdot, X_t)\big)
\Big].
\]
Using the definition of the tilted distribution $\tilde{p}_t(x,z) = \frac{w(z)p_t(x,z)}{C_t}$, we can substitute $w(z)p_t(x,z) = C_t \tilde{p}_t(x,z)$. This allows us to rewrite the expectation under $\tilde{p}_t$:
\begin{equation}
    \mathcal{L}_t^w(\theta)
    =
    C_t \cdot
    \E_{(X_t, Z) \sim \tilde{p}_t}
    \Big[
    D_{X_t}\big(u_t(\cdot, X_t \mid Z),\, u_{\theta,t}(\cdot, X_t)\big)
    \Big].
\end{equation}

\subsection*{Step 2: Differentiate}
We apply the gradient operator $\nabla_\theta$. Note that $C_t$, the distribution $\tilde{p}_t$, and the conditional target $u_t(\cdot, X_t \mid Z)$ do not depend on $\theta$. Thus, we can move the gradient inside the expectation and apply the Chain Rule (Proposition~\ref{prop:chain_rule}):
\begin{align}
    \nabla_\theta \mathcal{L}_t^w(\theta)
    &=
    C_t \cdot
    \E_{(X_t, Z) \sim \tilde{p}_t}
    \Big[
    \nabla_\theta D_{X_t}\big(u_t(\cdot, X_t \mid Z),\, u_{\theta,t}(\cdot, X_t)\big)
    \Big] \nonumber \\
    &=
    C_t \cdot
    \E_{(X_t, Z) \sim \tilde{p}_t}
    \Big[
    \underbrace{\nabla_v D_{X_t}\big(u_t(\cdot, X_t \mid Z), v\big)\Big|_{v=u_{\theta,t}(\cdot, X_t)}}_{\text{Gradient w.r.t output}}
    \cdot
    \underbrace{\nabla_\theta u_{\theta,t}(\cdot, X_t)}_{\text{Jacobian}}
    \Big].
\end{align}

\subsection*{Step 3: Marginalization (Tower Property)}
We use the law of iterated expectations (Tower Property) to split the joint expectation over $(X_t, Z)$ into a marginal expectation over $X_t$ and a conditional expectation over $Z|X_t$.
Note that the Jacobian term $\nabla_\theta u_{\theta,t}(\cdot, X_t)$ depends only on $X_t$, so it can be taken out of the inner expectation.
\begin{equation}
    \nabla_\theta \mathcal{L}_t^w(\theta)
    =
    C_t \cdot
    \E_{X_t \sim \tilde{p}_t}
    \Bigg[
    \Bigg(
    \E_{Z \sim \tilde{p}_t(\cdot|X_t)}
    \Big[
    \nabla_v D_{X_t}\big(u_t(\cdot, X_t \mid Z),\, u_{\theta,t}(\cdot, X_t)\big)
    \Big]
    \Bigg)
    \cdot
    \nabla_\theta u_{\theta,t}(\cdot, X_t)
    \Bigg].
\end{equation}

\subsection*{Step 4: Affine Invariance of the Bregman Gradient}
We focus on the inner expectation term:
\[
    \E_{Z \sim \tilde{p}_t(\cdot|X_t)}
    \Big[
    \nabla_v D_{X_t}\big(u_t(\cdot, X_t \mid Z),\, u_{\theta,t}(\cdot, X_t)\big)
    \Big].
\]
Let the random variable $Y = u_t(\cdot, X_t \mid Z)$ and the fixed vector $v = u_{\theta,t}(\cdot, X_t)$. By Proposition~\ref{prop:bregman_affine} (Affine Invariance), we can push the expectation inside the divergence gradient:
\begin{align}
    \E_{Z}
    \Big[
    \nabla_v D_{X_t}\big(u_t(\cdot, X_t \mid Z),\, u_{\theta,t}(\cdot, X_t)\big)
    \Big]
    &=
    \nabla_v D_{X_t}
    \Big(
    \underbrace{\E_{Z \sim \tilde{p}_t(\cdot|X_t)} [u_t(\cdot, X_t \mid Z)]}_{\text{Definition of } \tilde{u}_t(\cdot, X_t)},
    \,u_{\theta,t}(\cdot, X_t)
    \Big) \nonumber \\
    &=
    \nabla_v D_{X_t}
    \big(
    \tilde{u}_t(\cdot, X_t),
    \,u_{\theta,t}(\cdot, X_t)
    \big).
\end{align}
Substituting this back into Eq.~(28):
\begin{equation}
    \nabla_\theta \mathcal{L}_t^w(\theta)
    =
    C_t \cdot
    \E_{X_t \sim \tilde{p}_t}
    \Big[
    \nabla_v D_{X_t}\big(\tilde{u}_t(\cdot, X_t),\, u_{\theta,t}(\cdot, X_t)\big)
    \cdot
    \nabla_\theta u_{\theta,t}(\cdot, X_t)
    \Big].
\end{equation}

\subsection*{Step 5: Reconstruction (Reverse Chain Rule)}
The term inside the brackets is exactly the expansion of the gradient of the divergence with respect to $\theta$ (using Proposition~\ref{prop:chain_rule} in reverse):
\[
    \nabla_v D_{X_t}\big(\tilde{u}_t(\cdot, X_t),\, u_{\theta,t}(\cdot, X_t)\big)
    \cdot
    \nabla_\theta u_{\theta,t}(\cdot, X_t)
    =
    \nabla_\theta D_{X_t}\big(\tilde{u}_t(\cdot, X_t),\, u_{\theta,t}(\cdot, X_t)\big).
\]
Therefore, we can rewrite the full expression as:
\begin{align}
    \nabla_\theta \mathcal{L}_t^w(\theta)
    &=
    C_t \cdot
    \E_{X_t \sim \tilde{p}_t}
    \Big[
    \nabla_\theta D_{X_t}\big(\tilde{u}_t(\cdot, X_t),\, u_{\theta,t}(\cdot, X_t)\big)
    \Big] \nonumber \\
    &=
    \nabla_\theta
    \Bigg(
    C_t \cdot
    \E_{X_t \sim \tilde{p}_t}
    \Big[
    D_{X_t}\big(\tilde{u}_t(\cdot, X_t),\, u_{\theta,t}(\cdot, X_t)\big)
    \Big]
    \Bigg) \nonumber \\
    &=
    \nabla_\theta \mathcal{L}_{t,\text{Marg}}^w(\theta).
\end{align}
This holds because $C_t$ and $\tilde{p}_t$ are independent of $\theta$, allowing us to pull the gradient operator outside the entire expectation.

Since the gradients match for every time $t$, taking the expectation over $t \sim U[0,1]$ yields:
\[
    \nabla_\theta \mathcal{L}_{t}^w(\theta)
    =
    \E_t [\nabla_\theta \mathcal{L}_t^w(\theta)]
    =
    \E_t [\nabla_\theta \mathcal{L}_{t,\text{Marg}}^w(\theta)]
    =
    \nabla_\theta \mathcal{L}_{t,\mathrm{Marg}}^w(\theta).
\]
\end{proof}

\subsection{Proof of Corollary~\ref{cor:optimal_policy}: Recovery of Boltzmann Policy}
\label{sec:proof_cor1}

\paragraph{Theoretical Guarantees.}
To establish the consistency of our method, we rely on the following standard assumptions regarding the dataset coverage and model capacity.

\begin{assumption}[Full Support]
\label{ass:support}
The behavior policy $\mu(\cdot \mid s)$ has full support on the action space $\mathcal{A}$, or at minimum covers the support of the optimal policy $\pi^\star_\omega$. That is, $\mu(a \mid s) > 0$ for all $a$ such that $\pi^\star_\omega(a \mid s) > 0$.
\end{assumption}

\begin{assumption}[Realizability and Global Optimality]
\label{ass:realizability}
For the ideal $\epsilon=0$ construction, the parameterized rate model class
$\{u_{\theta,t}\}_{\theta\in\Theta}$ is sufficiently expressive to represent
the guided marginal rate field $\tilde{u}_t$ for all $t\in[0,1)$.
Moreover, optimization reaches a global minimizer $\theta^\star$ of the
corresponding ideal weighted conditional objective.
\end{assumption}

For clarity, we denote by $\mathcal{L}_{t}^{w,(0)}$ the ideal weighted
conditional objective obtained by setting $\epsilon=0$ in the target of
Eq.~\eqref{eq:qcdfm}, and by $\mathcal{L}_{t}^{w,(\epsilon)}$ the
regularized objective used in the implementation. The exact recovery result
below concerns $\mathcal{L}_{t}^{w,(0)}$. For finite $\epsilon>0$, the same
conditional-to-marginal gradient-equivalence argument applies to
$\tilde{u}_t^{(\epsilon)}$, but the induced terminal distribution is only an
approximation to the ideal endpoint distribution; the approximation converges
to the ideal construction as $\epsilon\to0$.

\textbf{Why zero conditional loss is not required.}
Because the conditional target $u_t(\cdot,A_t\mid Z)$ depends on the latent
conditioning variable $Z$, whereas the learned rate
$u_{\theta,t}(\cdot,A_t\mid s,\omega)$ does not observe all components of $Z$,
the conditional regression generally has irreducible variance and its minimum
need not be zero. The recovery result instead relies on the fact that the
weighted conditional and guided marginal objectives have the same minimizers.
Under Assumption~\ref{ass:realizability}, the ideal guided marginal objective
is uniquely minimized at $u_{\theta,t}=\tilde u_t$. For the finite-$\epsilon$
objective, the analogous minimizer is the regularized guided target
$\tilde u_t^{(\epsilon)}$.

A discussion on the practical implications of these assumptions is provided in Appendix~\ref{app:theory_practice_gap}.

% \begin{theorem}[Oracle Single-Step Recovery of the Boltzmann Policy]
% \label{thm:boltzmann_recovery_main}
% Under Assumptions~\ref{ass:support} and~\ref{ass:realizability}, let
% $\theta^\star$ be a global minimizer of the ideal weighted conditional
% objective $\mathcal{L}_{t}^{w,(0)}$. Then
% \[
% u_{\theta^\star,t}=\tilde{u}_t,
% \qquad t\in[0,1),
% \]
% and in the ideal terminal limit the corresponding CTMC satisfies
% \[
% \pi_{\theta^\star}(a\mid s,\omega)
% =
% \pi^\star_\omega(a\mid s)
% \propto
% \mu(a\mid s)\exp\!\big(\beta Q_\omega(s,a)\big).
% \]
% \end{theorem}

\begin{proof}[Proof of Corollary~\ref{cor:optimal_policy}]
\textbf{Step 1: The Guided Marginal Endpoint.}
The learned rate field $\tilde{u}_t$ generates the marginal probability path of the tilted distribution $\tilde{p}_t$. At $t=1$, the conditional path is a delta function at the target endpoint $A_1$, i.e., $p_{1|Z}(a|z) = \delta(a, A_1)$. Also, any joint distribution \( p_t(a,z) \) is defined as the product of the context prior \( p(z) \) and the conditional path \( p_{t|Z}(a|z) \):
$$
p_t(a,z) = p(z) \cdot p_{t|Z}(a|z)
$$
Thus, the terminal marginal distribution is the marginal of $A_1$ under the tilted measure (from Eq.~\eqref{eq:tilted_joint_main} and ~\eqref{eq:dfm-mixture-path}):
\[
\tilde{p}_1(a) = \sum_z \tilde{p}_1(a, z) = \sum_{z} \frac{w(z) p(z)}{C_1} \delta(a, A_1).
\]

\textbf{Step 2: Expanding the Weights.}
Recall that $Z$ includes endpoints sampled from $\mu(\cdot|s)$ and $w(Z) \propto \exp(\beta Q_\omega(s, A_1))$. Marginalizing over the latent variables $A_0$ and $s$ (conditioned on a specific state $s$ for inference), the distribution of $A_1$ in the dataset is $\mu(a|s)$. The guidance reweights this by $\exp(\beta Q_\omega)$. Therefore:
\[
\tilde{p}_1(a \mid s) \propto \mu(a \mid s) \exp(\beta Q_\omega(s, a)).
\]

This is the exact definition of the KL-regularized (Boltzmann) policy \( \pi_{\omega}^{\star} \).

By Theorem~\ref{thm:guided_grad_equiv_main}, the weighted conditional and
guided marginal objectives have identical gradients and hence share the same
global minimizers. Under Assumption~\ref{ass:realizability}, the guided marginal
objective is minimized at $\tilde{u}_t$. Therefore,
$u_{\theta^\star,t}=\tilde{u}_t$, and the induced CTMC follows the guided
marginal path $\tilde{p}_t$.

\textbf{Step 3: Conclusion.}
Under Assumption~\ref{ass:realizability}, a global minimizer of the ideal
weighted conditional objective $\mathcal{L}_{t}^{w,(0)}$ satisfies
$u_{\theta^\star,t}=\tilde{u}_t$ for $t<1$. Therefore, the corresponding CTMC generates
the guided marginal path $\tilde{p}_t$, whose terminal distribution derived
above is the KL-regularized Boltzmann policy. Hence,
$\pi_{\theta^\star}=\pi^\star_\omega$.

For the regularized objective used in practice,
$u_{\theta,t}$ instead targets $\tilde{u}_t^{(\epsilon)}$; as shown in
Appendix~\ref{app:ctmc_cdfm_details}, the resulting construction converges to
the ideal one as $\epsilon\to0$, while finite $\epsilon$ removes the divergent
target variance near $t=1$.
\end{proof}

\subsection{Practical Implications of Theoretical Assumptions}
\label{app:theory_practice_gap}

Corollary~\ref{cor:optimal_policy} relies on standard assumptions in offline reinforcement learning, including full support of the behavior policy and realizability of the function class. While these assumptions are idealized, they are commonly used to characterize the target solution and connect the objective to KL-regularized (Boltzmann) policy improvement \citep{nair2020awac}.

In practice, these conditions are not required to hold exactly. The critic is learned from finite data, the action support is approximated using sampled candidate sets, and optimization is performed using stochastic gradient methods. Under these approximations, the resulting objective can be interpreted as a soft, Q-weighted policy improvement step that biases the policy toward higher value actions while remaining grounded in the dataset distribution.

This form of approximate policy improvement is known to be robust in offline RL, as it avoids extrapolation outside the data support while still enabling meaningful improvement. Our empirical results support this, showing that the method remains effective despite these idealized assumptions.

\subsection{Multi-Objective Contrastive Energy Prediction}

To gain theoretical insights, we first formulated a preference-conditioned contrastive energy objective and present our model here.
Contrastive Energy Prediction (CEP) \citep{lu2023contrastive} provides a principled way to learn an intermediate energy model whose
induced distribution matches the KL-regularized optimal policy.

Given a state $s$, preference $\omega$, and critic $\mathbf{Q}_\psi$,
we define the target energy
\[
E_{\text{target}}(s,a,\omega) := -\beta Q_\omega(s,a).
\]

To avoid computing the partition function over $\mathcal{A}$, we adopt
the standard \emph{in-support approximation}.
For each state $s$, we sample a support set
\[
\mathcal{A}_{\text{supp}}(s) = \{ \hat a^{(1)}, \dots, \hat a^{(M)} \}
\sim \mu(\cdot \mid s).
\]

We train a preference-conditioned energy model
$f_\phi(s,a,t,\omega)$ by minimizing the cross-entropy
between the target Boltzmann distribution and the model distribution
restricted to the support set:
\begin{align}
\mathcal{L}_{\text{MO-CEP}}(\phi)
=
\E_{s,\omega,t}
\Bigg[
-
\sum_{i=1}^M
\tilde{w}_i
\log
\frac{\exp(f_\phi(s,\hat a^{(i)},t,\omega))}
{\sum_{j=1}^M \exp(f_\phi(s,\hat a^{(j)},t,\omega))}
\Bigg],
\end{align}
where
\[
\tilde{w}_i :=
\frac{\exp(\hat{\beta} Q_\omega(s,\hat a^{(i)}))}
{\sum_{j=1}^M \exp(\hat{\beta} Q_\omega(s,\hat a^{(j)}))}.
\]

This objective can train an energy model whose induced distribution
matches the preference-conditioned Boltzmann policy within the dataset support.

\subsection{Extension to Discrete Multi-Agent Reinforcement Learning}
\label{sec:marl_appendix}

Our framework also supports discrete multi-agent reinforcement learning (MARL) by
exploiting the factorized structure of Continuous-Time Markov Chains.
We describe this extension to illustrate the generality of the proposed formulation.

\paragraph{Joint action space.}
Consider a multi-agent setting with $G$ agents, where each agent $i$ selects a discrete
action $a^{(i)} \in \mathcal{A}_i$.
The joint action is
\[
a = (a^1,\dots,a^G) \in \mathcal{A}_1 \times \cdots \times \mathcal{A}_G =: \mathcal{A}.
\]
Directly modeling a generative policy over $\mathcal{A}$ is intractable when the number of agents
grows, as $|\mathcal{A}|$ scales exponentially with $G$.

\paragraph{Factorized CTMC policy.}
We model the joint policy $\pi_\theta(\cdot\mid s,\omega)$ as the terminal distribution of a
CTMC over joint actions, using a \emph{factorized generator} \cite{campbell2022continuous,campbell2024generative}
\begin{equation}
\label{eq:app_marl_factorized_ctmc}
u_{\theta,t}(a',a \mid s,\omega)
=
\sum_{i=1}^G
\delta(a'^{(-i)},a^{(-i)})\,
u_{\theta,t}^{(i)}(a'^{(i)},a \mid s,\omega),
\end{equation}
where $a^{(-i)}$ denotes the actions of all agents except agent $i$.
The term $u^{(i)}_{\theta,t}(a'^{(i)},a \mid s,\omega)$ represents the transition rate for agent $i$
to change its action from $a^{(i)}$ to $a'^{(i)}$, conditioned on the full joint action $a$.

For each agent $i$ where $b^{(i)}$ denotes a candidate action for agent $i$, we assume
$u^{(i)}_{\theta,t}(b^{(i)},a\mid s,\omega)\ge 0$ for all $b^{(i)}\neq a^{(i)}$ and define the diagonal entry
\[
u^{(i)}_{\theta,t}(a^{(i)},a\mid s,\omega)
:=
-\sum_{b^{(i)}\neq a^{(i)}} u^{(i)}_{\theta,t}(b^{(i)},a\mid s,\omega).
\]
Under this construction, the global diagonal is given by
\[
u_{\theta,t}(a,a\mid s,\omega)
:=
-\sum_{a'\neq a} u_{\theta,t}(a',a\mid s,\omega),
\]
ensuring that $u_{\theta,t}$ defines a valid time-inhomogeneous Markov jump process.

This factorization permits transitions that modify the action of a single agent at a time while
keeping the remaining agents fixed.
Consequently, the number of modeled rates scales as $\sum_i |\mathcal{A}_i|$, rather than
$\prod_i |\mathcal{A}_i|$.

\paragraph{Interpretation.}
The induced CTMC corresponds to an asynchronous multi-agent update process in which agents revise
their actions sequentially according to learned transition rates.
This preserves the discrete semantics of joint actions and avoids the need for continuous
relaxations or synchronized multi-agent jumps.

\paragraph{Training objective.}
Let $Z=(s,\omega,a_0,a_1)$ denote the conditioning variable, where $a_0$ and $a_1$ are joint actions
sampled from the behavior dataset.
We consider a factorized conditional probability path
\[
p_{t\mid Z}(a)
=
\prod_{i=1}^G
\big[(1-t)\delta(a^{(i)},a_0^{(i)}) + t\delta(a^{(i)},a_1^{(i)}\big],
\]
which allows intermediate joint actions in which some agents have transitioned to their terminal
actions $a_1^{(i)}$ while others remain at their initial actions $a_0^{(i)}$.
Such partially transitioned joint actions are reachable under the factorized CTMC dynamics,
which permit single-agent updates at each jump.

For this factorized conditional path, we adopt the standard factorized-velocity conditional
discrete flow matching objective, in which per-agent conditional generators
$u_t^{(i)}(\cdot,\cdot\mid Z)$ are chosen to generate the corresponding per-agent marginals.
The resulting training loss is
\begin{equation}
\label{eq:marl_cdfm}
\mathcal{L}_{t}^w(\theta)
=
\mathbb{E}_{t,Z,A_t}
\Bigg[
w(s,\omega,a_1)
\sum_{i=1}^G
D_{A_t^i}\Big(
u_t^{(i)}(\cdot,A_t \mid Z),
u_{\theta,t}^{(i)}(\cdot,A_t \mid s,\omega)
\Big)
\Bigg],
\end{equation}
where $t \sim \mathrm{Unif}[0,1]$, the intermediate action $A_t$ is sampled from $p_{t|Z}$ and $D_{A_t^i}$ denotes a Bregman divergence over valid rate vectors for agent $i$, and
$w(s,\omega,a_1)\propto\exp(\beta Q_\omega(s,a_1))$ is the same Boltzmann guidance weight
used in the single-agent case.

Under the standard regularity assumptions of discrete flow matching,
the factorized construction preserves the marginalization and gradient-equivalence properties.

\paragraph{Centralized value guidance.}
Policy learning in this extension is guided by a centralized critic
$Q_\omega(s,a^1,\dots,a^G)$, while the policy itself factorizes through the CTMC dynamics.
This corresponds to centralized training with a structured generative policy and does not require
decentralized execution assumptions.
% We leave empirical evaluation in multi-agent domains to future work.

This factorized CTMC formulation demonstrates that the proposed discrete,
preference-conditioned flow matching framework naturally generalizes to multi-agent action spaces.
The extension requires no modification to the underlying theory and highlights the flexibility of our discrete flow-based generative policy in structured decision-making problems.

\subsection{Additional Theory}

\paragraph{Conditional paths and marginalization trick.}
DFM defines a \emph{conditional} probability path $p_{t\mid Z}(\cdot\mid z)$ between endpoints
(often delta endpoints) and then marginalizes:
\begin{equation}
\label{eq:marginalize_path_prelim}
p_t(x)=\mathbb{E}_{Z\sim p_Z}\big[p_{t\mid Z}(x\mid Z)\big].
\end{equation}
If a conditional rate field $u_t(\cdot,\cdot\mid z)$ generates $p_{t\mid Z}(\cdot\mid z)$, 
% via \eqref{eq:kolmogorov_prelim}
 then the \emph{marginal} rate field
\begin{equation}
\label{eq:marginal_rate_prelim}
u_t(y,x) = \mathbb{E}\big[u_t(y,X_t\mid Z)\,\big|\,X_t=x\big]
\end{equation}
generates the marginal path $p_t$ \citep{lipman2024flow}.

\paragraph{Learning objective (DFM and conditional DFM).}
DFM learns a parametric rate model $u_{\theta,t}$ by regressing rates using a Bregman divergence.
Define the convex ``rate simplex'' at state $x$:
\[
\Omega_x := \left\{v\in \mathbb{R}^{S}:\ v(y)\ge 0\ \forall y\neq x,\ \ v(x)=-\sum_{y\neq x}v(y)\right\}.
\]
DFM minimizes the (intractable) marginal loss
\begin{equation}
\label{eq:dfm_loss_prelim}
\mathcal{L}_{\mathrm{DFM}}(\theta)
=
\mathbb{E}_{t}\,\mathbb{E}_{X_t\sim p_t}
\Big[
D_{X_t}\big(u_t(\cdot,X_t),\,u_{\theta,t}(\cdot,X_t)\big)
\Big],
\end{equation}
where $D_x(\cdot,\cdot)$ is a Bregman divergence on $\Omega_x$.
Since the marginal rate $u_t$ is often intractable, DFM uses the tractable \emph{conditional} loss
\begin{equation}
\label{eq:cdfm_loss_prelim}
\mathcal{L}_{\mathrm{CDFM}}(\theta)
=
\mathbb{E}_{t}\,\mathbb{E}_{Z}\,\mathbb{E}_{X_t\sim p_{t\mid Z}}
\Big[
D_{X_t}\big(u_t(\cdot,X_t\mid Z),\,u_{\theta,t}(\cdot,X_t)\big)
\Big].
\end{equation}
A key result in DFM is that optimizing the conditional loss yields the correct marginal gradient
\citep{lipman2024flow}.

\paragraph{Theorem~\ref{thm:guided_grad_equiv_main} and the iterative algorithm.}
\label{app:theory_algo_gap}

Theorem~\ref{thm:guided_grad_equiv_main} analyzes a population objective with
a fixed endpoint source distribution and guidance weights that are independent
of the optimized model parameters. Corollary~\ref{cor:optimal_policy} further
specializes this construction to a single oracle policy-improvement step in
which endpoints are drawn from the behavior policy $\mu(\cdot\mid s)$ and
reweighted by $\exp(\beta Q_\omega(s,a))$. This gives
\[
\pi^\star_\omega(a\mid s)
\propto
\mu(a\mid s)\exp\!\big(\beta Q_\omega(s,a)\big).
\]

The practical Phase~3 procedure differs because its candidate endpoints are
drawn from the current CTMC policy. To make this distinction explicit, consider
an idealized population iteration with a fixed critic and exact optimization.
If $\pi_k$ denotes the endpoint distribution before iteration $k$, then the
reweighted target distribution is
\begin{equation}
\label{eq:population_iterative_update}
\pi_{k+1}(a\mid s,\omega)
=
\frac{
\pi_k(a\mid s,\omega)
\exp\!\big(\beta Q_\omega(s,a)\big)
}{
\sum_{a'}
\pi_k(a'\mid s,\omega)
\exp\!\big(\beta Q_\omega(s,a')\big)
}.
\end{equation}
Starting from $\pi_0=\mu$, repeated exact updates therefore give
\begin{equation}
\label{eq:population_iterative_closed_form}
\pi_k(a\mid s,\omega)
\propto
\mu(a\mid s)
\exp\!\big(k\beta Q_\omega(s,a)\big).
\end{equation}
Thus, unlike the fixed-source oracle update, repeated policy-sourced updates
accumulate guidance and progressively sharpen the distribution toward
highest-value actions in the initial support. Fixed points of
Eq.~\eqref{eq:population_iterative_update} are distributions whose support is
contained in actions having the same $Q_\omega$ value; with a unique
highest-value action, the exact iteration converges toward a point mass on
that action.

If instead endpoints are resampled from the fixed behavior source $\mu$ at
every iteration, each update has the same target
\[
\pi_{k+1}(a\mid s,\omega)
\propto
\mu(a\mid s)\exp\!\big(\beta Q_\omega(s,a)\big),
\]
so the update is idempotent and recovers the oracle distribution characterized
by Corollary~\ref{cor:optimal_policy}.

Algorithm~\ref{alg:sketch} uses the policy-sourced version because refreshed
candidate sets can allocate more of the finite candidate budget to actions that
have been emphasized by earlier updates. However, the implemented procedure
uses finite $M$, self-normalized weights, finite optimization steps, a learned
critic, and an approximate CTMC simulator. Consequently,
Corollary~\ref{cor:optimal_policy} is a consistency result for the oracle
single-step construction and is not an exact convergence theorem for the full
iterative Phase~3 algorithm.

\paragraph{Fixed-source stability alternative.}
We also evaluate a variant that resamples Phase~3 endpoints from the fixed
behavior estimate $\hat{\mu}$ rather than from the current CTMC policy. This
variant more closely matches the oracle construction in
Corollary~\ref{cor:optimal_policy} and avoids the cumulative sharpening of the
idealized policy-sourced population update. Table~\ref{tab:fixed_mu_endpoint}
shows that the two variants have similar mean return on the two evaluated
tasks, while the fixed-$\hat{\mu}$ variant has substantially smaller empirical
standard deviation. We therefore view fixed-$\hat{\mu}$ sampling as a useful
lower-variance alternative, but these two tasks are insufficient to establish
uniform superiority over policy-sourced sampling.

\begin{table}[h]
\centering
\small
\setlength{\tabcolsep}{8pt}
\begin{tabular}{lcc}
\toprule
Phase~3 endpoint source & Hopper-Medium & Walker2d-Expert \\
\midrule
Current policy (default) & $37.79 \pm 16.11$ & $4.85 \pm 2.34$ \\
Fixed $\hat{\mu}$        & $36.58 \pm 3.77$  & $4.24 \pm 1.42$ \\
\bottomrule
\end{tabular}
\caption{Phase~3 endpoint-source ablation over 5 seeds. Fixed-$\hat{\mu}$
sampling gives similar mean return with lower empirical variance on these two
tasks.}
\label{tab:fixed_mu_endpoint}
\end{table}

\paragraph{Why retain policy-sourced endpoints?}
The practical motivation is finite-$M$ adaptive proposal sampling. If a
high-value action has probability $p$ under $\hat{\mu}(\cdot\mid s)$, then
with $M$ independent endpoint draws it is proposed at least once with
probability
\[
1-(1-p)^M.
\]
If earlier policy-improvement updates increase its probability to $q_k>p$,
policy-sourced refreshes increase this inclusion probability to
$1-(1-q_k)^M$. Thus, with a finite candidate budget, resampling from the
current policy can allocate more candidates to promising actions emphasized by
earlier updates.

This finite-$M$ behavior also differs from simply applying a larger guidance
scale to fixed-source candidates. Although $k$ exact population updates are
equivalent to a single fixed-source tilt with effective scale $k\beta$, at
finite $M$ a large one-step tilt can further concentrate the self-normalized
weights on the same limited candidate set. In our experiments the effective
sample size of the endpoint weights is typically only $3$--$5$ out of
$M=64$. Iterative refreshes instead modify the proposal before applying the
next moderate tilt. We therefore retain policy-sourced endpoints as the
default while reporting fixed-$\hat{\mu}$ sampling as a lower-variance
alternative.

\subsection{Approximation Errors in the Practical Algorithm}
\label{app:approximation_errors}

The oracle result in Corollary~\ref{cor:optimal_policy} assumes an exact
population objective, the true behavior distribution, an exact critic, and
exact CTMC dynamics. Algorithm~\ref{alg:sketch} introduces four corresponding
sources of approximation error.

\paragraph{Finite candidate set.}
Phase~3 replaces the population endpoint expectation by a self-normalized
Monte Carlo estimate over $M$ sampled candidates. Under standard bounded-moment
conditions, the stochastic error of this estimator decreases at the usual
Monte Carlo rate $O_p(M^{-1/2})$, while self-normalization introduces a
finite-sample bias that vanishes as $M\to\infty$. Thus, finite $M$ does not
change the population target asymptotically, but small candidate sets can
provide poor coverage of high value actions and produce concentrated
importance weights. In our MuJoCo experiments we use $M=64$; during Phase~3,
the effective sample size of the normalized endpoint weights is typically only
3--5 out of 64, indicating that the practical update can be dominated by a
small number of sampled candidates. We therefore treat finite-candidate
quality as an important practical approximation rather than assuming that the
population reweighting is exactly realized.

\paragraph{Behavior-policy estimation.}
The oracle construction in Corollary~\ref{cor:optimal_policy} assumes access
to the true behavior policy $\mu(\cdot\mid s)$. A fixed-source implementation
using an estimated behavior model $\hat{\mu}(\cdot\mid s)$ would instead induce
the one-step target
\[
\hat{\pi}_\omega(a\mid s)
\propto
\hat{\mu}(a\mid s)\exp\!\big(\beta Q_\omega(s,a)\big),
\]
rather than
$\mu(a\mid s)\exp(\beta Q_\omega(s,a))$.
Thus, behavior-policy estimation error perturbs the source distribution of the
oracle construction.

The practical Phase~3 procedure in Algorithm~\ref{alg:sketch} draws endpoint
candidates from the current CTMC policy rather than directly from
$\hat{\mu}$. Nevertheless, $\hat{\mu}$ enters the practical method through
the learned behavior model and the initialization of CTMC sampling, and hence
affects the support and proposal distribution used by the algorithm.
Our initialization ablation in Appendix~\ref{app:muhat_ablation}
(Table~\ref{tab:muhat_ablation}) shows that this choice has a substantial
empirical effect, so we do not assume behavior-policy estimation error to be
negligible.

\paragraph{Critic error.}
The guidance weights are computed using the learned critic
$\hat Q_\omega$ rather than the exact value function. If
$\hat Q_\omega(s,a)=Q_\omega(s,a)+e(s,a)$ with
$|e(s,a)|\leq\delta$, then each unnormalized guidance weight differs from its
oracle counterpart by a multiplicative factor between
$\exp(-\beta\delta)$ and $\exp(\beta\delta)$. Thus, critic errors are amplified
by the guidance scale $\beta$ and can alter the ranking and concentration of
the finite candidate weights. This effect can be particularly important when
candidate Q-values are close or the effective sample size is small. We do not
claim a finite-sample critic-error guarantee; the recovery result assumes an
exact critic and the experiments quantify performance under the learned one.

\paragraph{CTMC discretization.}
Inference uses the Euler approximation in
Algorithm~\ref{alg:ctmc_inference} rather than exact simulation of the
continuous-time Markov chain. For bounded and sufficiently regular
time-dependent rates, the Euler transition approximation has
$O(h^2)$ one-step truncation error and $O(h)$ accumulated error over a fixed
unit-time horizon as $h\to0$. The implementation uses $h=0.05$ in the main
MuJoCo experiments. Appendix~\ref{app:ablations} empirically studies the
number of Euler steps and shows that performance stabilizes as the simulation
resolution is increased. Thus the exact terminal-distribution statements in
the theory refer to the continuous-time process, while practical sampling
contains a controllable discretization error.

Taken together, these approximations explain why
Corollary~\ref{cor:optimal_policy} should be interpreted as a population-level
characterization of the oracle update rather than a finite-sample convergence
guarantee for Algorithm~\ref{alg:sketch}.

\clearpage

\section{Additional Experimental Details \& Results}

\subsection{Implementation Details}
\label{ap:imp_details}

\paragraph{Model architecture.}
The transition-rate model $u_{\theta}$ is parameterized as a lightweight MLP with two hidden layers of size 256, comparable to critic networks used in standard offline RL. Our method does not rely on large model capacity; performance gains primarily arise from the flow-based policy improvement mechanism rather than increased model size.

\paragraph{Generative policy vs. static policy.}
Our method does not learn a static categorical policy or an action-to-action value table. Instead, it parameterizes a time-dependent CTMC generator $u_{\theta,t}$, which defines a stochastic process over actions. Starting from an initial action distribution, the model gradually shifts probability mass through stochastic transitions toward higher value actions. This trajectory-based representation can capture structured multimodal action distributions without explicitly normalizing over all actions as in categorical Boltzmann policies. As shown in Table~\ref{tab:main_baselines_256}, QDFM  improves over Boltzmann-Q even when using the same critic, suggesting that the generative CTMC policy provides a stronger action representation than a static one-step policy.

\paragraph{Why use $\hat{\mu}$ in Algorithm~\ref{alg:ctmc_inference}?}
We use $\hat{\mu}(\cdot\mid s)$ in Algorithm~\ref{alg:ctmc_inference} so that
the initialization reflects the support of the offline dataset while still
providing a diverse set of plausible actions. Using only the dataset action
collapses the initialization support to a single point, while uniform
initialization can introduce out-of-support actions. We study these alternatives
in Appendix~\ref{app:muhat_ablation}; Table~\ref{tab:muhat_ablation} shows that
the data-informed initialization performs substantially better than either
single-action or uniform initialization.

% \paragraph{Why use $\hat{\mu}$ in Algorithm~\ref{alg:ctmc_inference}?}
% We use $\hat{\mu}(\cdot\mid s)$ in Algorithm~\ref{alg:ctmc_inference} so that the initialization reflects the support of the offline dataset while still providing a diverse set of plausible actions. Using only the dataset action collapses the initialization support to a single point, which limits diversity and leads to weaker policy improvement. To validate this design choice, we replace $\hat{\mu}$ with two alternatives: initializing from the single dataset action, denoted single-a, and initializing uniformly over actions. As shown in Table~\ref{tab:mu_ablation}, both alternatives perform substantially worse than QDFM, showing that $\hat{\mu}$ is important for maintaining in-support diversity and improving the policy.

% \begin{table}[t]
% \centering
% \caption{$\hat{\mu}$ ablation on MuJoCo datasets.}
% \label{tab:mu_ablation}
% \small
% \setlength{\tabcolsep}{8pt}
% \begin{tabular}{lccc}
% \toprule
% Dataset & QDFM & single-a & uniform \\
% \midrule
% Hopper-Medium
% & $\mathbf{35.81 \pm 18.80}$
% & $0.75 \pm 0.01$
% & $1.82 \pm 0.03$ \\
% Walker2d-Medium
% & $\mathbf{7.90 \pm 5.65}$
% & $0.94 \pm 0.81$
% & $3.90 \pm 2.98$ \\
% \bottomrule
% \end{tabular}
% \end{table}

\subsection{MuJoCo Results: Varying the number of action divisions.}
\label{app:action_divisions}
We check how the number of action divisions $K_{\text{act}}$ affects performance.
Table~\ref{tab:ablation_N} reports performance as we vary $K_{\text{act}} \in \{8,16,32,64,256\}$.
Overall, although moderate discretization appears sufficient to capture dominant
modes of the offline action distribution, performance is strongest at $K_{\text{act}}$ = 256 as we can also see from Table~\ref{tab:ablation_N_aggregate} where we report aggregated performance across tasks as a function of $K_{\text{act}}$.

\begin{table}[h]
\centering
\begin{minipage}[t]{0.73\columnwidth}
  \centering
  \caption{Effect of the number of action divisions $K_{\text{act}}$ on performance.
  Results are mean $\pm$ standard deviation across seeds.
  Best results per row (highest return) are shown in \textbf{bold}.}
  \label{tab:ablation_N}
  \scriptsize
  \setlength{\tabcolsep}{1.0pt}
  \renewcommand{\arraystretch}{1.1}
  \begin{tabular}{lccccc}
  \toprule
  Env/Dataset 
  & $K_{\text{act}}{=}8$ 
  & $K_{\text{act}}{=}16$ 
  & $K_{\text{act}}{=}32$ 
  & $K_{\text{act}}{=}64$ 
  & $K_{\text{act}}{=}256$ \\
  \midrule
  halfcheetah/expert
  & -142.49$\pm$53.51
  & -224.44$\pm$70.61
  & -202.31$\pm$73.72
  & -157.69$\pm$84.00
  & \textbf{-116.30$\pm$135.33} \\
  halfcheetah/medium
  & -4.49$\pm$125.64
  & 37.69$\pm$312.38
  & 98.24$\pm$207.09
  & 187.53$\pm$237.86
  & \textbf{467.21$\pm$248.30} \\
  hopper/expert
  & 118.94$\pm$73.35
  & 193.74$\pm$163.84
  & 261.94$\pm$245.81
  & 189.08$\pm$131.07
  & \textbf{557.09$\pm$247.67} \\
  hopper/medium
  & 31.25$\pm$11.20
  & 349.19$\pm$167.19
  & 8.28$\pm$0.99
  & 7.18$\pm$1.14
  & \textbf{1209.63$\pm$524.31} \\
  walker2d/expert
  & -24.43$\pm$17.33
  & \textbf{383.12$\pm$192.18}
  & 149.77$\pm$165.94
  & 162.74$\pm$172.15
  & 224.28$\pm$107.42 \\
  walker2d/medium
  & 332.03$\pm$441.95
  & 98.96$\pm$132.03
  & 203.75$\pm$244.01
  & 322.46$\pm$301.45
  & \textbf{644.78$\pm$453.56} \\
  \bottomrule
  \end{tabular}
\end{minipage}
\hfill
\begin{minipage}[t]{0.24\columnwidth}
  \centering
  \caption{Aggregated performance vs.\ discretization $K_{\text{act}}$.
  We average episodic return across the six MuJoCo benchmarks, and report mean $\pm$ std across seeds.}
  \label{tab:ablation_N_aggregate}
  \small
  \setlength{\tabcolsep}{3pt}
  \renewcommand{\arraystretch}{1.1}
  \begin{tabular}{lc}
  \toprule
  $K_{\text{act}}$ & Avg.\ return \\
  \midrule
  $8$   & $51.80 \pm 75.03$ \\
  $16$  & $139.71 \pm 92.24$ \\
  $32$  & $86.61 \pm 99.29$ \\
  $64$  & $118.55 \pm 93.95$ \\
  $256$ & $\mathbf{497.78 \pm 286.10}$ \\
  \bottomrule
  \end{tabular}
\end{minipage}
\end{table}

For our comparisons we fix a single discretization and select $K_{\text{act}}=256$ because it performs strongly on most environments giving the highest returns.
To verify that this choice does not obscure trends, we next provide a full baseline comparison at $K_{\text{act}}=16$ in the next subsection.

\subsection{Performance on discretized MuJoCo tasks}
\label{app:comparisons}

Here we report actual mean values of the algorithms on the 6 benchmark problems for $K_{\text{act}}$ = 16 and $K_{\text{act}}$ = 256. We report episodic return (sum of environment rewards per episode; higher is better)
as mean $\pm$ standard deviation over five training seeds. All methods use a batch size of 256 and are evaluated over 10 episodes. For baseline methods, we train for 1M gradient
steps.
Our method uses a three-phase training schedule with $(K_1, K_2, K_3) = (150k, 500k, 350k)$,
support size $M=64$, guidance scale $\hat{\beta}=20$, and CTMC step size $h=0.05$.

It is worth noting that because actions are discretized into a finite set, absolute normalized returns are not directly comparable to continuous-action policies. Our goal is to evaluate relative performance among discrete offline methods under a fixed discretization.

\begin{table*}[h]
\centering
\caption{Performance comparison at $K_{\text{act}}=256$.
Mean $\pm$ standard deviation across seeds.
Best results per row (highest return) are in \textbf{bold}.}
\label{tab:main_baselines_N256}
\small
\resizebox{\textwidth}{!}{%
\begin{tabular}{llcccccc}
\toprule
Dataset & Environment & AWAC & AWBC & BCQ & CQL & GreedyQ & Ours \\
\midrule
Expert & HalfCheetah
& 29.14$\pm$306.94
& -367.81$\pm$55.60
& \textbf{1002.09$\pm$682.03}
& -428.26$\pm$1.14
& -319.23$\pm$0.40
& -116.30$\pm$135.33 \\

Medium & HalfCheetah
& 2164.79$\pm$2021.35
& -5.00$\pm$152.87
& \textbf{2231.27$\pm$1019.49}
& -351.66$\pm$100.36
& -308.52$\pm$6.93
& 467.21$\pm$248.30 \\

Expert & Hopper
& 10.63$\pm$0.13
& 13.30$\pm$0.87
& 140.93$\pm$154.72
& 8.88$\pm$0.08
& 8.88$\pm$0.10
& \textbf{557.09$\pm$247.67} \\

Medium & Hopper
& 29.92$\pm$0.97
& 211.61$\pm$5.18
& 1157.34$\pm$853.44
& 16.48$\pm$0.68
& 15.96$\pm$0.57
& \textbf{1209.63$\pm$524.31} \\

Expert & Walker2d
& 71.89$\pm$32.31
& -7.50$\pm$0.93
& 151.19$\pm$133.83
& -12.77$\pm$0.89
& -12.20$\pm$1.21
& \textbf{	224.28$\pm$107.42} \\

Medium & Walker2d
& -6.85$\pm$0.72
& 115.42$\pm$111.01
& 428.60$\pm$272.39
& 307.01$\pm$11.26
& 303.94$\pm$11.29
& \textbf{644.78$\pm$453.56} \\
\bottomrule
\end{tabular}
}
\end{table*}

\begin{table*}[h]
\centering
\caption{Performance comparison at $K_{\text{act}}=16$.
Mean $\pm$ standard deviation across seeds.
Best results per row (highest return) are in \textbf{bold}.}
\label{tab:main_baselines}
\small
\resizebox{\textwidth}{!}{%
\begin{tabular}{llcccccc}
\toprule
Dataset & Environment & AWAC & AWBC & BCQ & CQL & GreedyQ & Ours \\
\midrule
Expert & HalfCheetah
& -232.53$\pm$74.42
& -337.04$\pm$37.28
& -384.38$\pm$21.48
& -263.23$\pm$95.31
& -415.24$\pm$44.37
& \textbf{-224.44$\pm$70.61} \\

Medium & HalfCheetah
& \textbf{318.29$\pm$458.39}
& -5.80$\pm$388.09
& -198.59$\pm$82.80
& -442.33$\pm$69.31
& -143.53$\pm$148.09
& 37.69$\pm$312.38 \\

Expert & Hopper
& 7.06$\pm$0.41
& 8.42$\pm$0.19
& -0.77$\pm$0.34
& 105.71$\pm$31.11
& 1.04$\pm$0.19
& \textbf{193.74$\pm$163.84} \\

Medium & Hopper
& 37.42$\pm$1.32
& 198.77$\pm$76.68
& 43.95$\pm$4.19
& 29.17$\pm$0.80
& 28.52$\pm$0.58
& \textbf{349.19$\pm$167.19} \\

Expert & Walker2d
& 79.51$\pm$78.24
& 12.27$\pm$10.29
& \textbf{444.78$\pm$221.17}
& 68.31$\pm$81.92
& 5.67$\pm$20.73
& 383.12$\pm$192.18 \\

Medium & Walker2d
& \textbf{319.66$\pm$218.78}
& 46.59$\pm$102.14
& 293.99$\pm$72.79
& 31.02$\pm$107.75
& 90.36$\pm$2.39
& 98.96$\pm$132.03 \\
\bottomrule
\end{tabular}
}
\end{table*}

\paragraph{Additional baselines from Table~\ref{tab:main_baselines_256}.}
For completeness, we report AWAC, AWBC, and GreedyQ results
omitted from the main paper table for readability.
These are D4RL-normalized returns (mean $\pm$ standard
deviation over 5 seeds), using the same $K_{\text{act}}=256$
codebook and evaluation protocol as Table~\ref{tab:main_baselines_256}.

\begin{table*}[h]
\centering
\caption{Additional discrete offline RL baselines on MuJoCo
($K_{\text{act}}=256$). Normalized return (mean $\pm$
standard deviation over 5 seeds).}
\label{tab:additional_baselines_256}
\small
\begin{tabular}{llccc}
\toprule
Dataset & Env & AWAC & AWBC & GreedyQ \\
\midrule
Expert & HalfCheetah
& 2.49$\pm$2.47
& -0.71$\pm$0.45
& -0.31$\pm$0.00 \\

Medium & HalfCheetah
& 19.69$\pm$16.28
& 2.22$\pm$1.23
& -0.23$\pm$0.06 \\

Expert & Hopper
& 0.95$\pm$0.00
& 1.03$\pm$0.03
& 0.90$\pm$0.00 \\

Medium & Hopper
& 1.54$\pm$0.03
& 7.12$\pm$0.16
& 1.11$\pm$0.02 \\

Expert & Walker2d
& 1.53$\pm$0.70
& -0.20$\pm$0.02
& -0.30$\pm$0.03 \\

Medium & Walker2d
& -0.18$\pm$0.02
& 2.48$\pm$2.42
& 6.59$\pm$0.25 \\
\bottomrule
\end{tabular}
\end{table*}

\paragraph{Recent baseline implementations.}
For the recent comparisons in Table~\ref{tab:main_baselines_256} that we added later,
we require all methods to select actions directly from the same
flat codebook of $K_{\mathrm{act}}=256$ actions.

\textbf{ACTIVE-I \citep{chen2025active}.}
We retain ACTIVE-I's in-sample value ensemble, ensemble based
critic targets, advantage-weighted regression, and adaptive
cloning temperature. The actor is categorical rather than Gaussian:
the expected-Q term is computed exactly as
$\sum_a \pi(a\mid s)Q(s,a)$, and the weighted policy objective
uses categorical log-likelihoods $\log\pi(a\mid s)$.
We use five value-ensemble members and the published
medium settings $(\tau,c)=(0.9,1.0)$ and medium-expert
settings $(0.7,0.0)$ for our medium and expert datasets,
respectively. The categorical target log-likelihood $H_D$
is calibrated using an initial IVE-I run.

\textbf{CQL(SA) \citep{luo2025learning}.}
We implement state-adaptive regularization with a categorical
SAC-style actor. The CQL regularization and policy expectations
are evaluated exactly over all 256 actions. Since the original
Gaussian trust-region threshold does not directly transfer to a
categorical policy, we use the following probability-mass analogue.
Let $\pi_{(1)}(s)\ge\cdots\ge\pi_{(K_{\mathrm{act}})}(s)$
be the sorted action probabilities, and define
\[
p_n=\operatorname{erf}(n/\sqrt{2}),\qquad
k_n(s)=\min\left\{k:
\sum_{j=1}^{k}\pi_{(j)}(s)\ge p_n\right\},
\qquad
C_n(s)=\log\pi_{(k_n(s))}(s).
\]
Thus, $C_n(s)$ is the log-probability of the least likely
action in the smallest highest-probability set containing
at least mass $p_n$. We retain the state-adaptive coefficient
network and selective regularization mechanism, with
$\beta_{\mathrm{init}}=1$, $n_{\mathrm{start}}=1$, and
$n_{\mathrm{end}}=3$ (1.5 for expert datasets). The categorical
actor uses target entropy $0.1\log K_{\mathrm{act}}$.

\textbf{Matched CQL comparison.}
The original CQL baseline in Table~\ref{tab:main_baselines_256}
uses a DQN-style critic without an explicit actor. Since
CQL(SA) requires a stochastic actor, we also evaluate a
SAC-style CQL backbone with a fixed regularization coefficient
(CQL-SAC). This matched comparison isolates the effect of
state-adaptive regularization. All new baselines use the same
codebook and evaluation protocol: $10^6$ gradient steps,
batch size 256, and 10 evaluation episodes per seed.
We report D4RL-normalized returns over five seeds.

\begin{table}[h]
\centering
\small
\setlength{\tabcolsep}{3pt}
\caption{Matched CQL-SAC and CQL(SA) comparison on discretized
MuJoCo ($K_{\mathrm{act}}=256$). Normalized return
(mean $\pm$ standard deviation over five seeds).}
\label{tab:cqlsa_matched}
\begin{tabular}{llcc}
\toprule
Dataset & Env & CQL-SAC & CQL(SA) \\
\midrule
Expert & HalfCheetah
& -0.23$\pm$0.97 & -0.78$\pm$0.72 \\
Medium & HalfCheetah
& -0.80$\pm$0.71 & -0.88$\pm$0.35 \\
Expert & Hopper
& 10.56$\pm$3.37 & 13.52$\pm$8.69 \\
Medium & Hopper
& 25.33$\pm$10.53 & 28.02$\pm$12.47 \\
Expert & Walker2d
& 1.80$\pm$1.85 & 0.75$\pm$0.45 \\
Medium & Walker2d
& 3.30$\pm$3.34 & 12.25$\pm$3.13 \\
\bottomrule
\end{tabular}
\end{table}

CQL(SA) improves over its matched CQL-SAC backbone on
both Hopper datasets and Walker2d-Medium, but not on
HalfCheetah or Walker2d-Expert. Thus, its benefits are
task-dependent under this discrete adaptation.

\paragraph{Comparison with BraVE and combinatorial action spaces.}
BraVE~\citep{landers2025brave} addresses large combinatorial
action spaces through tree-structured value estimation.
Unlike our flat discrete-action benchmarks, its CoNE environment
contains exponentially many joint actions with dependencies
between sub-actions. QDFM can represent small CoNE action spaces
directly, but scaling to larger instances would require efficient
factorized CTMC rates and joint-action value estimation.
We leave direct evaluation on CoNE as future work.

\subsection{Euler steps and rate-scale ablations}
\label{app:ablations}

\paragraph{CTMC simulation details and hyperparameters}

We simulate CTMC sampling over $t\in[0,1]$ with Euler discretization using $K_{\text{steps}}$ steps
(step size $h=1/K_{\text{steps}}$).
A global rate scale $\alpha$ multiplies all transition rates, controlling the expected number of jumps within the unit-time horizon.
Unless otherwise stated, we use the same $K_{\text{steps}}$ and $\alpha$ reported in the corresponding figure/table captions.
We also report the source initialization used for sampling, since it affects mixing speed in short horizons.

\begin{figure}[h]
  \centering
  \begin{minipage}[t]{0.48\linewidth}
    \centering
    \includegraphics[width=\linewidth]{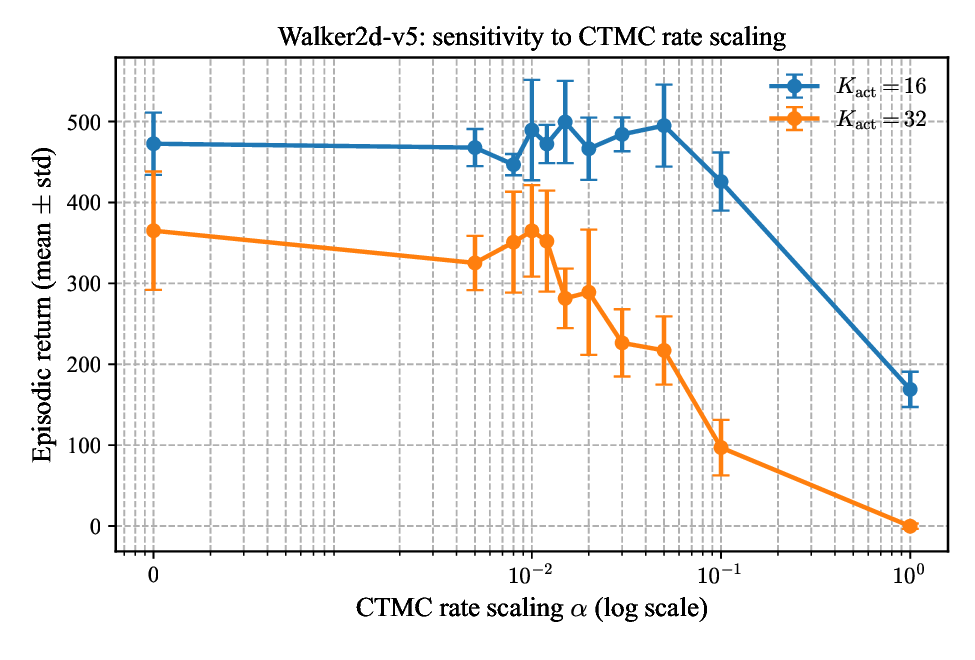}
    \captionof{figure}{
      \texttt{Walker2d-v5}: sensitivity to CTMC rate scaling $\alpha$ for
      discretizations $K_{\mathrm{act}}\in\{16,32\}$.
      We report mean episodic return across $5$ seeds
      (100 episodes/seed); error bars indicate standard deviation.
      All runs use the same Euler simulation horizon with
      $K_{\mathrm{steps}}=15$.
    }
    \label{fig:alpha_sweep}
  \end{minipage}
  \hfill
  \begin{minipage}[t]{0.48\linewidth}
    \centering
    \includegraphics[width=\linewidth]{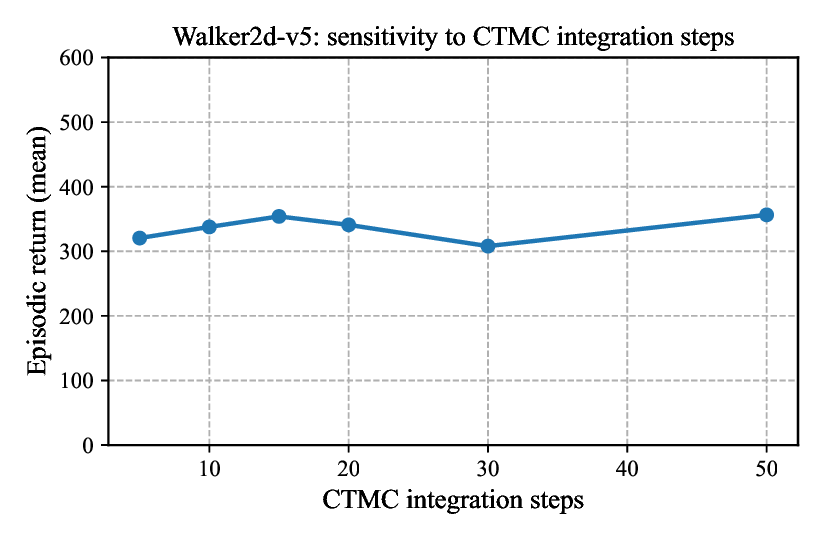}
    \captionof{figure}{
      Effect of the number of Euler steps used to simulate the CTMC at
      inference time on \texttt{Walker2d-v5}.
      Performance remains stable across a wide range of discretization
      steps.
    }
    \label{fig:steps_sweep}
  \end{minipage}
\end{figure}

\paragraph{Stability of discrete sampling.}
Figure~\ref{fig:alpha_sweep} illustrates how performance varies with the CTMC rate
scale $\alpha$.
For $K_{\text{act}}=16$, performance remains high across a broad range of values, indicating that
the discrete sampler is robust to the amount of stochasticity introduced during
inference.
For $K_{\text{act}}=32$, performance is more sensitive and degrades as $\alpha$ increases, reflecting
the interaction between finer discretization and higher effective jump rates.
In both cases, very large rate scales lead to unstable trajectories.
Overall, these trends show that CTMC-based inference exhibits predictable and
interpretable behavior, with clear stability regimes governed by discretization and
sampling intensity.

\paragraph{Effect of the number of integration steps.}
As can be seen from Figure~\ref{fig:steps_sweep}, varying the number of Euler integration steps $K_{\text{steps}}$ with the rate scale fixed yields stable performance across a wide range of values. Returns saturate beyond approximately $15$ steps, indicating that accurate inference can be achieved with modest simulation resolution.

\subsection{Ablations on $M$, $N$, and Runtime Trade-offs}
\label{app:mn_runtime_ablations}

We examine how the sampled endpoint count $M$ and CTMC integration steps $N$ affect measured runtime. Within the reported $M$ sweep, training time increases with $M$, while action generation time changes relatively little. Within the $N$ sweep, action generation time
increases approximately linearly with $N$.
These are within-sweep observations; the absolute timings should not be treated as directly comparable across separate experiments.

\begin{table}[h]
\centering
\small
\begin{tabular}{lcccc}
\toprule
Dataset & $M=16$ & $M=32$ & $M=64$ & $M=128$ \\
\midrule
Hopper-Medium   & 3158.7 $\pm$ 637.2 & 3583.4 $\pm$ 359.6 & 4719.3 $\pm$ 104.4 & 7247.1 $\pm$ 118.8 \\
Walker2d-Medium & 3152.7 $\pm$ 607.6 & 3588.1 $\pm$ 361.6 & 4664.4 $\pm$ 9.5   & 7346.0 $\pm$ 12.4 \\
\bottomrule
\end{tabular}
\caption{Ablation on $M$: training time (seconds).}
\label{tab:m_train_ablation}
\end{table}

\begin{table}[h]
\centering
\small
\begin{tabular}{lcccc}
\toprule
Dataset & $M=16$ & $M=32$ & $M=64$ & $M=128$ \\
\midrule
Hopper-Medium   & 8.86 $\pm$ 2.06 & 8.86 $\pm$ 1.96 & 8.93 $\pm$ 2.12 & 8.86 $\pm$ 2.01 \\
Walker2d-Medium & 8.91 $\pm$ 2.22 & 8.91 $\pm$ 2.18 & 10.02 $\pm$ 0.09 & 10.08 $\pm$ 0.10 \\
\bottomrule
\end{tabular}
\caption{Ablation on $M$: action generation time (ms).}
\label{tab:m_action_ablation}
\end{table}

\begin{table}[h]
\centering
\small
\begin{tabular}{lcccc}
\toprule
Dataset & $N=5$ & $N=10$ & $N=20$ & $N=40$ \\
\midrule
Hopper-Medium   & 2286.7 $\pm$ 13.3 & 2381.3 $\pm$ 21.6 & 2429.6 $\pm$ 43.9 & 2637.3 $\pm$ 7.7 \\
Walker2d-Medium & 2314.8 $\pm$ 28.9 & 2382.0 $\pm$ 77.5 & 2480.2 $\pm$ 62.5 & 2683.2 $\pm$ 84.8 \\
\bottomrule
\end{tabular}
\caption{Ablation on $N$: training time (seconds).}
\label{tab:n_train_ablation}
\end{table}

\begin{table}[H]
\centering
\small
\begin{tabular}{lcccc}
\toprule
Dataset & $N=5$ & $N=10$ & $N=20$ & $N=40$ \\
\midrule
Hopper-Medium   & 1.89 $\pm$ 0.02 & 3.60 $\pm$ 0.23 & 6.29 $\pm$ 0.20 & 12.69 $\pm$ 0.84 \\
Walker2d-Medium & 1.92 $\pm$ 0.03 & 3.48 $\pm$ 0.18 & 6.57 $\pm$ 0.48 & 12.59 $\pm$ 0.75 \\
\bottomrule
\end{tabular}
\caption{Ablation on $N$: action generation time (ms).}
\label{tab:n_action_ablation}
\end{table}

\subsection{Training Time Scaling with $|\mathcal{A}|$ and $M$}
\label{app:runtime_scaling}

\begin{table}[H]
\centering
\small
\begin{tabular}{c|ccc}
\toprule
$|\mathcal{A}|$ & $M=32$ & $M=64$ & $M=128$ \\
\midrule
64  & 1527 & 1469 & 1481 \\
256 & 1411 & 1448 & 1578 \\
512 & 1420 & 1510 & 1869 \\
\bottomrule
\end{tabular}
\caption{Training time (seconds) while varying $|\mathcal{A}|$ and $M$ on Hopper-Medium.}
\label{tab:runtime_scaling}
\end{table}

We additionally report a separate runtime sweep over the action-space size $|\mathcal{A}|$ and sampled endpoint count $M$ on Hopper-Medium. Table~\ref{tab:runtime_scaling} presents the measured training times for these configurations.

Within this sweep, training time shows relatively modest variation over the tested settings, although the increase is more pronounced at $|\mathcal{A}|=512$. The weaker dependence on $M$ observed here differs from the trend in Appendix~\ref{app:mn_runtime_ablations}.

\subsection{Runtime Analysis}
\label{app:runtime}

Here we report the same-hardware runtime comparisons for both training and inference across methods. QDFM is faster in training than BCQ, IQL, and AWAC, and remains comparable to CQL across environments. The higher inference time of QDFM reflects the use of CTMC-based sampling,
which enables a more expressive generative policy compared to single-pass
baselines.

\begin{table}[h]
\centering
\resizebox{\textwidth}{!}{%
\small
\begin{tabular}{lcccccc}
\toprule
Dataset & GreedyQ & CQL & BCQ & AWAC & IQL & QDFM \\
\midrule
HalfCheetah-M & 1380.1 $\pm$ 13.2 & 1729.0 $\pm$ 24.8 & 2364.8 $\pm$ 85.3 & 2606.1 $\pm$ 16.6 & 3620.3 $\pm$ 11.2 & 2209.4 $\pm$ 19.7 \\
HalfCheetah-E & 1351.0 $\pm$ 10.5 & 1751.5 $\pm$ 42.2 & 2324.4 $\pm$ 9.8 & 2606.7 $\pm$ 32.6 & 3611.9 $\pm$ 38.1 & 2232.0 $\pm$ 27.9 \\
Hopper-M & 1403.7 $\pm$ 38.0 & 1761.5 $\pm$ 26.5 & 2364.6 $\pm$ 69.1 & 2602.9 $\pm$ 27.8 & 3633.2 $\pm$ 24.5 & 2208.3 $\pm$ 16.6 \\
Hopper-E & 1378.8 $\pm$ 30.3 & 1722.0 $\pm$ 26.1 & 2416.9 $\pm$ 68.9 & 2607.3 $\pm$ 61.2 & 3674.4 $\pm$ 134.1 & 2214.6 $\pm$ 31.4 \\
Walker2d-M & 1374.8 $\pm$ 10.2 & 1742.9 $\pm$ 36.0 & 2389.5 $\pm$ 74.3 & 2598.9 $\pm$ 13.9 & 3683.5 $\pm$ 102.1 & 2233.5 $\pm$ 17.3 \\
Walker2d-E & 1354.9 $\pm$ 8.1 & 1720.9 $\pm$ 14.2 & 2401.2 $\pm$ 61.1 & 2619.7 $\pm$ 5.2 & 3617.7 $\pm$ 33.7 & 2249.1 $\pm$ 26.9 \\
\bottomrule
\end{tabular}
}
\caption{Same-hardware training time (seconds).}
\label{tab:runtime_train}
\end{table}

\begin{table}[h]
\resizebox{\textwidth}{!}{%
\centering
\small
\begin{tabular}{lcccccc}
\toprule
Dataset & GreedyQ & CQL & BCQ & AWAC & IQL & QDFM \\
\midrule
HalfCheetah-M & 0.22 $\pm$ 0.00 & 0.21 $\pm$ 0.01 & 0.40 $\pm$ 0.03 & 0.22 $\pm$ 0.00 & 0.21 $\pm$ 0.00 & 6.28 $\pm$ 0.04 \\
HalfCheetah-E & 0.21 $\pm$ 0.00 & 0.23 $\pm$ 0.03 & 0.39 $\pm$ 0.00 & 0.22 $\pm$ 0.02 & 0.22 $\pm$ 0.00 & 6.38 $\pm$ 0.02 \\
Hopper-M & 0.29 $\pm$ 0.00 & 0.30 $\pm$ 0.02 & 0.45 $\pm$ 0.05 & 0.28 $\pm$ 0.01 & 0.23 $\pm$ 0.00 & 6.26 $\pm$ 0.02 \\
Hopper-E & 0.33 $\pm$ 0.01 & 0.35 $\pm$ 0.04 & 0.48 $\pm$ 0.09 & 0.34 $\pm$ 0.03 & 0.24 $\pm$ 0.03 & 6.28 $\pm$ 0.09 \\
Walker2d-M & 0.26 $\pm$ 0.02 & 0.24 $\pm$ 0.01 & 0.44 $\pm$ 0.01 & 0.24 $\pm$ 0.01 & 0.23 $\pm$ 0.01 & 6.39 $\pm$ 0.08 \\
Walker2d-E & 0.23 $\pm$ 0.00 & 0.23 $\pm$ 0.00 & 0.46 $\pm$ 0.05 & 0.24 $\pm$ 0.00 & 0.24 $\pm$ 0.00 & 6.35 $\pm$ 0.12 \\
\bottomrule
\end{tabular}
}
\caption{Same-hardware action generation time (milliseconds).}
\label{tab:runtime_infer}
\end{table}

\subsection{QIPO Comparison and the Effect of Action Quantization}
\label{app:qipo_quantization}

To separate the effect of action quantization from the choice of discrete
policy model, we compare with QIPO~\citep{zhang2025energy}. We evaluate the
same trained QIPO policy both in its native continuous action space
(QIPO-cont.) and after projecting its output onto the same
$K_{\mathrm{act}}=256$ codebook used by QDFM (QIPO-quant.).

\begin{table}[h]
\centering
\caption{Effect of action quantization on QIPO compared with QDFM.
QIPO-cont.\ acts in the original continuous action space, whereas QIPO-quant.\
projects the output of the continuously trained policy onto the same
$K_{\mathrm{act}}=256$ codebook used by QDFM. Results are mean
$\pm$ standard deviation.}
\label{tab:qipo_quantization}
\small
\setlength{\tabcolsep}{4pt}
\begin{tabular}{lccc}
\toprule
Dataset & QIPO-cont. & QIPO-quant. & QDFM \\
\midrule
HalfCheetah-Expert
& 48.08$\pm$33.13 & 4.00$\pm$1.60 & 1.32$\pm$1.09 \\
HalfCheetah-Medium
& 92.42$\pm$19.43 & 8.59$\pm$2.64 & 6.02$\pm$2.00 \\
Hopper-Expert
& 44.06$\pm$14.61 & 25.20$\pm$10.66 & 17.74$\pm$7.61 \\
Hopper-Medium
& 73.09$\pm$22.10 & 49.01$\pm$9.82 & 37.79$\pm$16.11 \\
Walker2d-Expert
& 9.90$\pm$3.68 & 6.26$\pm$2.61 & 4.85$\pm$2.34 \\
Walker2d-Medium
& 79.61$\pm$27.86 & 34.74$\pm$18.83 & 14.01$\pm$9.88 \\
\bottomrule
\end{tabular}
\end{table}

Quantization substantially reduces QIPO's return, with the largest degradation
on HalfCheetah. Nevertheless, QIPO-quant.\ has a higher mean return than QDFM
on all six tasks. This comparison therefore separates the cost of native
categorical training from the intrinsically discrete settings targeted by
QDFM: QIPO is trained continuously and projected only at evaluation, whereas
QDFM is trained directly over the categorical action set.

\paragraph{Interpreting the HalfCheetah gap.}
The weaker QDFM results on HalfCheetah have not been causally isolated.
In particular, QIPO after quantized projection still outperforms QDFM, so
discretization alone does not explain the gap. Possible contributors include
the difference between continuous training followed by projection and native
optimization over a flat categorical codebook, finite candidate support in
Phase~3, and sensitivity of candidate ranking to critic error. We treat these
as hypotheses rather than established causes; the finite-$M$ and critic-error
effects are discussed further in Appendix~\ref{app:approximation_errors}.

\subsection{Role of Initialization Distribution $\hat{\mu}$}
\label{app:muhat_ablation}

We study the effect of the initialization distribution $\hat{\mu}$ used in Algorithm~\ref{alg:ctmc_inference}. The role of $\hat{\mu}$ is to define the initial distribution over actions before the CTMC evolution.

We compare three choices:
(i) $\hat{\mu}$ constructed from the dataset,
(ii) a single-action initialization using the dataset action,
(iii) a uniform distribution over actions.

Using a single action collapses the initial support, preventing the model from exploring alternative high value actions. In contrast, uniform initialization introduces many out-of-support actions, which can degrade performance in offline RL.

As shown in Table~\ref{tab:muhat_ablation}, both alternatives perform significantly worse than $\hat{\mu}$, demonstrating that a data-informed initialization is critical for balancing support and diversity in offline RL.

\begin{table}[h]
\centering
\small
\begin{tabular}{lccc}
\toprule
Dataset & QDFM ($\hat{\mu}$) & single-a & uniform \\
\midrule
Hopper-Medium   & \textbf{37.79 $\pm$ 16.11} & 0.75 $\pm$ 0.01 & 1.82 $\pm$ 0.03 \\
Walker2d-Medium & \textbf{14.01 $\pm$ 9.88}  & 0.94 $\pm$ 0.81 & 3.90 $\pm$ 2.98 \\
\bottomrule
\end{tabular}
\caption{Ablation on the initialization distribution $\hat{\mu}$. Using a data-informed initialization significantly improves performance compared to single-action or uniform initialization.}
\label{tab:muhat_ablation}
\end{table}

\subsection{Atari Evaluation}
\label{app:atari}

We provide preliminary results on two intrinsically discrete Atari
environments, Breakout and Pong, using their original discrete action
spaces. All methods use the same fixed offline dataset for each
environment, $84\times84$ grayscale observations with four stacked
frames, a convolutional encoder, and matched training budgets and
evaluation protocols.

We use three independent training seeds, evaluating each
trained policy over 10 episodes. We report the mean and standard
deviation across the three seed-level average episodic returns.

\begin{table}[htbp]
\centering
\small
\caption{Atari results (episodic return, mean $\pm$
standard deviation across three seeds). Higher is better.}
\label{tab:atari_prelim}
\begin{tabular}{lcccc}
\toprule
Environment & QDFM & IQL & BC & CQL \\
\midrule
Breakout
& $32.3\pm4.9$
& $28.5\pm5.7$
& $25.4\pm3.8$
& $22.7\pm6.1$ \\
Pong
& $-7.1\pm2.5$
& $-7.3\pm2.9$
& $-13.2\pm2.1$ 
& $-12.8\pm3.2$ \\
\bottomrule
\end{tabular}
\end{table}

QDFM achieves the highest reported mean return on Breakout, while
its performance on Pong is comparable to IQL within the observed
variability. These results provide preliminary evidence on native
discrete-action environments rather than a comprehensive Atari
benchmark.

\subsection{CTMC Sampler Correctness in Intrinsically Discrete Control}
\label{app:ctmc_correctness}
This section isolates the sampling mechanism to prove that the CTMC converges to the target distribution, independent of potential offline RL training errors. We validate the CTMC sampler independently of continuous-control artifacts
by evaluating it in environments with inherently discrete action spaces.

\paragraph{Setup.}
We consider \texttt{CartPole-v1} and \texttt{Acrobot-v1}.
Given a reference categorical policy $\pi(\cdot\mid s)$, we construct a CTMC with
rates
\[
u(a',a\mid s)=\alpha\,\pi(a'\mid s), \qquad a'\neq a,
\]
and simulate the chain over $t\in[0,1]$ using Euler discretization.
% We study the effects of jump intensity $\alpha$, integration resolution
% $K_{\text{steps}}$, and the source distribution $p_0$.

\paragraph{Baselines.}
We compare CTMC inference against:
(i) a random policy,
(ii) direct sampling from the reference categorical policy $\pi(\cdot\mid s)$, and
(iii) the deterministic argmax policy induced by $\pi$.
These baselines isolate the effect of CTMC-based action generation from the
quality of the underlying policy.

\paragraph{CartPole.}
When initialized from target policy ($p_0=\pi$), CTMC sampling closely matches direct sampling
from the target policy across a wide range of $\alpha$ and $K_{\text{steps}}$.
When initialized from a fixed action, performance improves monotonically with
$\alpha$, reflecting stronger mixing toward the target distribution.

\begin{table}[h]
\centering
\begin{minipage}[t]{0.56\textwidth}
  \vspace{0pt}
  \centering
  \caption{
  Baselines on \texttt{CartPole-v1} (200 episodes).
  The solved fraction denotes the proportion of episodes achieving the task
  success threshold.
  }
  \label{tab:cartpole_baselines}
  \vspace{1mm}
  \small
  \setlength{\tabcolsep}{3pt}
  \resizebox{\linewidth}{!}{%
  \begin{tabular}{lcc}
  \toprule
  Policy & Return (mean $\pm$ std) & Solved fraction \\
  \midrule
  Random & $22.46 \pm 13.05$ & $0.00$ \\
  $\pi(\cdot\mid s)$ sampling & $400.25 \pm 100.94$ & $0.41$ \\
  Argmax$(\pi)$ & $500.00 \pm 0.00$ & $1.00$ \\
  \bottomrule
  \end{tabular}
  }
\end{minipage}
\hfill
\begin{minipage}[t]{0.39\textwidth}
  \vspace{0pt}
  \centering
  \caption{
  CTMC sampling on \texttt{CartPole-v1} (200 episodes).
  With $p_0=\pi$, CTMC closely matches direct sampling across rate scales, while
  fixed-action initialization benefits monotonically from increasing $\alpha$
  due to faster mixing.
  }
  \label{tab:cartpole_ctmc}
  \vspace{1mm}
  \small
  \setlength{\tabcolsep}{4pt}
  \resizebox{\linewidth}{!}{%
  \begin{tabular}{c|cc}
  \toprule
  Rate scale $\alpha$ & Init = $\pi$ & Init = fixed \\
  \midrule
  $1.0$  & $410.5$ & $42.1$ \\
  $2.0$  & $408.8$ & $209.9$ \\
  $5.0$  & $406.4$ & $402.5$ \\
  $10.0$ & $403.1$ & $394.4$ \\
  \bottomrule
  \end{tabular}
  }
\end{minipage}
\end{table}

% \begin{table}[h]
% \centering
% \caption{
% Baselines on \texttt{CartPole-v1} (200 episodes).
% The solved fraction denotes the proportion of episodes achieving the task success
% threshold.
% }
% \label{tab:cartpole_baselines}
% \begin{tabular}{lccc}
% \toprule
% Policy & Return (mean $\pm$ std) & Solved fraction \\
% \midrule
% Random & $22.46 \pm 13.05$ & $0.00$ \\
% $\pi(\cdot\mid s)$ sampling & $400.25 \pm 100.94$ & $0.41$ \\
% Argmax$(\pi)$ & $500.00 \pm 0.00$ & $1.00$ \\
% \bottomrule
% \end{tabular}
% \end{table}

% \begin{table}[h]
% \centering
% \begin{tabular}{c|cc}
% \toprule
% Rate scale $\alpha$ & Init = $\pi$ & Init = fixed \\
% \midrule
% $1.0$  & $410.5$ & $42.1$ \\
% $2.0$  & $408.8$ & $209.9$ \\
% $5.0$  & $406.4$ & $402.5$ \\
% $10.0$ & $403.1$ & $394.4$ \\
% \bottomrule
% \end{tabular}
% \caption{
% CTMC sampling on \texttt{CartPole-v1} (200 episodes).
% With $p_0=\pi$, CTMC closely matches direct sampling across rate scales, while
% fixed-action initialization benefits monotonically from increasing $\alpha$ due
% to faster mixing.
% }
% \label{tab:cartpole_ctmc}
% \end{table}

\paragraph{Acrobot.}
On Acrobot, CTMC sampling substantially outperforms random actions and closely
tracks the reference categorical policy.
Increasing $\alpha$ and/or $K_{\text{steps}}$ improves robustness to poor
initialization, consistent with faster mixing within the fixed simulation
horizon.

\begin{table}[h]
\centering
\begin{minipage}[t]{0.42\columnwidth}
  \centering
  \caption{
    CTMC sampling on \texttt{Acrobot-v1}.
    CTMC outperforms random actions and remains robust to poor initialization as
    the rate scale or number of steps increases. Reported CTMC results correspond
    to the best-performing configuration from the hyperparameter sweep.
  }
  \label{tab:acrobot_ctmc}
  \vspace{1mm}
  \small
  \setlength{\tabcolsep}{4pt}
  \renewcommand{\arraystretch}{1.1}
  \begin{tabular}{lcc}
    \toprule
    Policy & Mean return & Std \\
    \midrule
    Random & $-498.0$ & $19.9$ \\
    Argmax heuristic & $-172.6$ & $106.6$ \\
    Sample $\pi$ & $-156.9$ & $82.4$ \\
    CTMC (init=$\pi$) & $-147.16$ & $71.06$ \\
    CTMC (init=uniform) & $-164.69$ & $84.03$ \\
    \bottomrule
  \end{tabular}
\end{minipage}
\hfill
\begin{minipage}[t]{0.55\columnwidth}
  \centering
  \caption{
    CTMC ablations on \texttt{Acrobot-v1}.
    We sweep the rate scale $\alpha$ and number of Euler steps $K_{\text{steps}}$
    under two initializations.
    Higher rate scales and increased integration steps improve robustness to
    uniform initialization and yield performance comparable to initialization
    from $\pi$.
  }
  \label{tab:acrobot_ctmc_sweep}
  \vspace{1mm}
  \small
  \setlength{\tabcolsep}{3pt}
  \renewcommand{\arraystretch}{1.1}
  \begin{tabular}{cccc}
    \toprule
    $\alpha$ & $K_{\text{steps}}$ & init & Return (mean $\pm$ std) \\
    \midrule
    1.0 & 20 & $\pi$ & $-151.26 \pm 61.00$ \\
    1.0 & 20 & uniform & $-214.72 \pm 110.40$ \\
    1.0 & 50 & $\pi$ & $-165.85 \pm 92.80$ \\
    1.0 & 50 & uniform & $-214.95 \pm 106.44$ \\
    \midrule
    2.0 & 20 & $\pi$ & $-152.91 \pm 91.67$ \\
    2.0 & 20 & uniform & $-166.50 \pm 92.61$ \\
    2.0 & 50 & $\pi$ & $-161.34 \pm 94.17$ \\
    2.0 & 50 & uniform & $-181.06 \pm 97.78$ \\
    \midrule
    5.0 & 20 & $\pi$ & $\mathbf{-147.16 \pm 71.06}$ \\
    5.0 & 20 & uniform & $-164.69 \pm 84.03$ \\
    5.0 & 50 & $\pi$ & $-157.01 \pm 91.54$ \\
    5.0 & 50 & uniform & $-170.28 \pm 88.09$ \\
    \bottomrule
  \end{tabular}
\end{minipage}
\end{table}

Across both environments, the CTMC sampler exhibits predictable and interpretable
behavior: jump intensity controls convergence, initialization effects diminish
with stronger mixing, and Euler discretization remains numerically stable.
These experiments validate the correctness of the discrete CTMC construction
independently of offline RL training.

\subsection{CartPole: Toy two-objective preference conditioning}
\label{cartpole_morl_appendix}

We construct a simple two-objective variant of \texttt{CartPole} to check
preference conditioning in a fully discrete setting.
In addition to the task objective $R_1$ (episode length), we define a smoothness
objective $R_2$ that penalizes action switches.
The policy is conditioned on a continuous preference parameter
$\omega\in[0,1]$, and a single model is evaluated across preferences without
retraining.

Figure~\ref{fig:cartpole_morl} shows that varying $\omega$ induces a smooth
traversal of the Pareto front between task performance and smoothness.
This demonstrates that preference signals can be injected directly into the
endpoint distribution and respected by CTMC-based sampling, enabling zero-shot
preference generalization in discrete action spaces. The task return measures standard CartPole performance (episode length), while the smoothness return penalizes action switching; together they form a conflicting two-objective control problem used to verify that the preference-conditioned CTMC policy smoothly traverses the Pareto frontier as the preference parameter varies.

\begin{figure*}[h]
\centering
\begin{minipage}[t]{0.60\textwidth}
  \vspace{0pt}
  \centering
  \includegraphics[width=\linewidth]{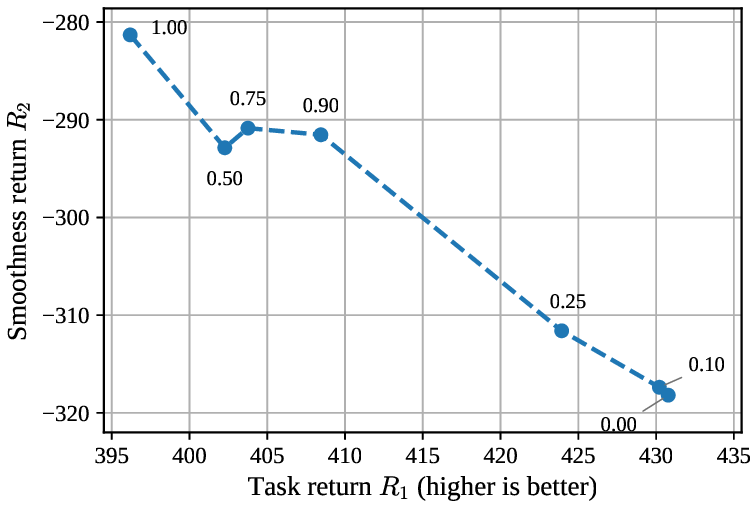}
  \caption{
    Toy two-objective \texttt{CartPole}.
    Sweeping the preference parameter $\omega$ induces a smooth traversal of the
    Pareto front without retraining. Points correspond to different $\omega$ values
    and are connected for visualization. This curve demonstrates that our method
    learns not only two extreme behaviors but also intermediate trade-offs.
  }
  \label{fig:cartpole_morl}
\end{minipage}
\hfill
\begin{minipage}[t]{0.37\textwidth}
  \vspace{0pt}
  \centering
  % \captionsetup{type=table}
  \caption{
    Toy preference sweep on \texttt{CartPole-v1} using the preference-conditioned
    CTMC policy.
    We report task return $R_1$, smoothness objective $R_2$ (negative number of action
    switches), scalarized return $R_\omega$, and the fraction of episodes achieving
    $R_1\ge 475$.
    All results use $\alpha=5$, $K_{\mathrm{steps}}=20$, and $k_{\pi}=20$ with 200 evaluation episodes.
  }
  \label{tab:cartpole_morl_sweep}
  \vspace{1mm}
  \scriptsize
  \setlength{\tabcolsep}{2.7pt}
  \renewcommand{\arraystretch}{1.3}
  \begin{tabular}{crrrr}
    \toprule
    $\omega$ & \shortstack[r]{$R_1$ \\ (task)} & \shortstack[r]{$R_2$ \\ (smooth)} & $R_\omega$ & Solved frac. \\
    \midrule
    0.00 & 430.78 & -318.19 & -318.19 & 0.55 \\
    0.10 & 430.21 & -317.38 & -242.62 & 0.52 \\
    0.25 & 423.93 & -311.60 & -127.72 & 0.54 \\
    0.50 & 402.27 & -292.87 & 54.70   & 0.43 \\
    0.75 & 403.76 & -290.86 & 230.10  & 0.43 \\
    0.90 & 408.45 & -291.54 & 338.46  & 0.50 \\
    1.00 & 396.19 & -281.32 & 396.19  & 0.41 \\
    \bottomrule
  \end{tabular}
\end{minipage}
\end{figure*}

% \begin{figure}[h]
% \centering
% \includegraphics[width=0.55\linewidth]{figures/pareto_frontier.eps}
% \caption{
% Toy two-objective \texttt{CartPole}.
% Sweeping the preference parameter $\omega$ induces a smooth traversal of the
% Pareto front without retraining. Points correspond to different $\omega$ values and are connected for visualization. This curve demonstrates that our method not only learns two extreme behaviors but it learns every possible trade-off in between."
% }
% \label{fig:cartpole_morl}
% \end{figure}

% \label{sec:cartpole_morl_appendix}
% \begin{table}[h]
% \centering
% \begin{tabular}{crrrr}
% \toprule
% $\omega$ & $R_1$ (task) & $R_2$ (smooth) & $R_\omega$ & Solved frac. \\
% \midrule
% 0.00 & 430.78 & -318.19 & -318.19 & 0.55 \\
% 0.10 & 430.21 & -317.38 & -242.62 & 0.52 \\
% 0.25 & 423.93 & -311.60 & -127.72 & 0.54 \\
% 0.50 & 402.27 & -292.87 & 54.70 & 0.43 \\
% 0.75 & 403.76 & -290.86 & 230.10 & 0.43 \\
% 0.90 & 408.45 & -291.54 & 338.46 & 0.50 \\
% 1.00 & 396.19 & -281.32 & 396.19 & 0.41 \\
% \bottomrule
% \end{tabular}
% \caption{
% Toy preference sweep on \texttt{CartPole-v1} using the preference-conditioned
% CTMC policy.
% We report task return $R_1$, smoothness objective $R_2$ (negative number of action
% switches), scalarized return $R_\omega$, and the fraction of episodes achieving
% $R_1\ge 475$.
% All results use $\alpha=5$, $K_{\mathrm{steps}}=20$, $k_{\pi}=20$, and
% $c_{\mathrm{sw}}=1$ with 200 evaluation episodes.
% }
% \label{tab:cartpole_morl_sweep}
% \end{table}

\subsection{Multi-Objective Fork: Qualitative Evidence of Discrete Multimodality}
\label{app:mofork_heatmaps}

To provide a self-contained qualitative visualization of discrete multimodality
and preference control, we include an additional toy gridworld, \texttt{Multi-Objective Fork}. The offline dataset contains two disjoint successful strategies: trajectories that
go left to \emph{Goal A} (high reward in objective~1) and trajectories that go right
to \emph{Goal B} (high reward in objective~2). Importantly, the ``average'' path
goes through a central \emph{trap} state that incurs a catastrophic penalty in both
objectives; no dataset trajectory enters the trap.
This setting isolates the key failure mode of naive action averaging in discrete
spaces, while directly testing whether our CTMC policy maintains multiple modes.

Figure~\ref{fig:mofork_heatmaps} visualizes the induced state occupancy of the learned
CTMC policy, where brightness reflects how frequently a state is visited during
sampling. Each panel corresponds to a different preference vector $\omega$:
left objective ($\omega=(1,0)$), balanced ($\omega=(0.5,0.5)$), and right objective
($\omega=(0,1)$).
Bright regions indicate states that the policy consistently traverses, while dark
regions indicate avoided states.
Notably, the central trap state remains unvisited even under balanced preferences,
demonstrating that the learned discrete flow preserves multiple modes rather than
collapsing to an invalid average behavior.

\begin{figure}[h]
  \centering
  \includegraphics[width=\linewidth]{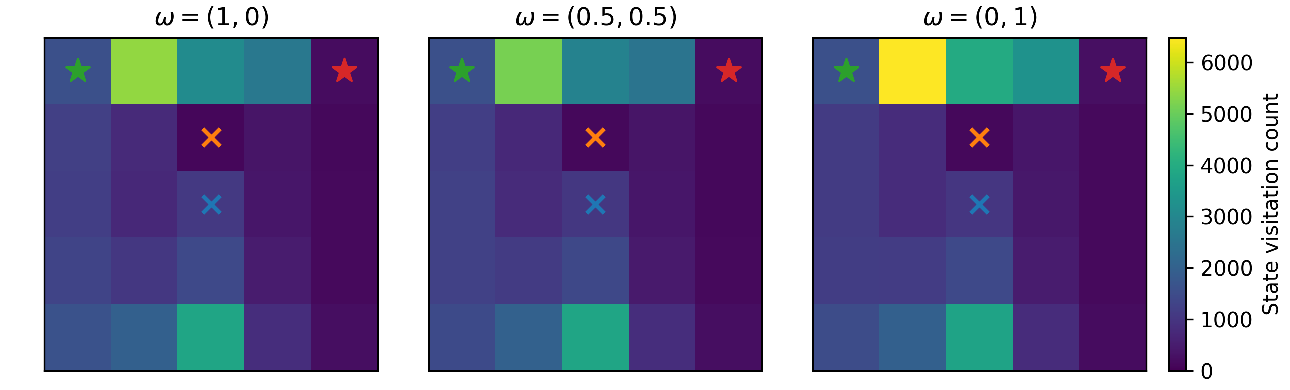}
  \caption{\textbf{Multi-Objective Fork: state visitation heatmaps under preference conditioning.}
  Each panel shows state visitation counts aggregated over rollouts from the learned
  CTMC policy under a different preference vector $\omega$:
  left objective ($\omega=(1,0)$), balanced ($\omega=(0.5,0.5)$), and right objective
  ($\omega=(0,1)$). The key qualitative signature is \emph{trap avoidance}: the trap
  state remains unvisited (dark) even under balanced preferences, while trajectories
  concentrate along the left or right corridor depending on $\omega$.
  This indicates that the learned discrete flow maintains \emph{multiple modes}
  instead of collapsing to an ``average'' (trap-entering) behavior, while still
  allowing controllable traversal between modes via $\omega$.}
  \label{fig:mofork_heatmaps}
\end{figure}

This Fork visualization complements the main benchmark results by offering a direct,
human-interpretable demonstration that (i) the learned discrete CTMC policy is
\emph{multimodal} and does not collapse into invalid interpolations, and (ii) the
preference signal $\omega$ meaningfully controls which mode is selected at test time.

\subsection{Preference-Conditioned Multi-objective Offline
Reinforcement Learning on Resource-Gathering}
\label{app:resource_gathering}

We evaluate the proposed discrete CTMC framework in a fully multi-objective,
preference-conditioned offline reinforcement learning setting using the
\texttt{Resource-Gathering} benchmark from MO-Gymnasium.
This environment features an intrinsically discrete control problem with three
conflicting objectives (e.g., collecting different resources while avoiding
unsafe regions), making it a natural testbed for assessing whether a single
learned model can adapt its behavior to different objective preferences at
test time.

\paragraph{Environment.}
In \texttt{Resource-Gathering}, an agent navigates a gridworld to collect
different types of resources.
Each episode yields a three-dimensional return vector consisting of
(i) collected gold,
(ii) collected diamonds, and
(iii) a safety or survival signal that penalizes hazardous behavior.
These objectives are inherently conflicting: aggressive resource collection
typically improves gold or diamond returns at the cost of safety.

\paragraph{Offline data and training.}
We construct an offline dataset using a mixture of heuristic behavior policies
(e.g., gold-seeking, diamond-seeking, safety-oriented, and random strategies),
ensuring broad coverage of the state--action space while inducing a nontrivial
support mismatch across objectives.
A single preference-conditioned CTMC policy is trained on this dataset.
At evaluation time, the policy is queried with a continuous preference vector
$\omega \in \Delta^{2}$ without retraining.

\paragraph{Coverage of achievable behaviors.}
Figure~\ref{fig:rg_pareto} compares the outcomes achieved by the learned CTMC
model against the behavior policy and a random baseline.
Each point corresponds to the average episodic outcome obtained under a
different preference vector.
Across preferences, the CTMC-based method consistently attains a wider range of
objective outcomes than the behavior policy, including combinations that achieve
higher resource collection while maintaining reasonable safety.
In contrast, the behavior policy remains confined to a narrower set of outcomes
reflecting the limitations of the offline data, and the random baseline performs
poorly across all objectives.

% \begin{figure}[h]
%   \centering
%   \includegraphics[width=0.55\columnwidth]{figures/rg_pareto_overlay.eps}
%   \caption{
%   Achieved objective outcomes on \texttt{Resource-Gathering} at the best validation
%   checkpoint.
%   Each point corresponds to a different preference vector $\omega$.
%   Colors indicate the achieved safety level.
%   The CTMC-based method reaches a broader range of outcomes than the behavior
%   policy and random baseline, including combinations with higher resource
%   collection.
%   }
%   \label{fig:rg_pareto}
% \end{figure}

\paragraph{Response to preference changes.}
To illustrate how the learned model reacts to changes in user preference,
Figure~\ref{fig:rg_pref_sweep} fixes the weight on safety and varies the relative
importance assigned to the two resource objectives.
As the preference is adjusted, the resulting behavior changes smoothly and
predictably: the agent collects more of the favored resource while reducing
emphasis on the other, and the scalarized return peaks near the intended tradeoff.
At extreme preferences, performance degrades in an interpretable way (e.g.,
overly aggressive behavior leads to early termination), indicating that the
model captures meaningful tradeoffs rather than arbitrary interpolation.

% \begin{figure}[h]
%   \centering
%   \includegraphics[width=0.55\columnwidth]{figures/rg_preference_sweep.eps}
%   \caption{
%   Effect of varying preferences on \texttt{Resource-Gathering} with fixed safety
%   weight.
%   As the relative importance of gold versus diamond changes, the learned CTMC
%   model adjusts its behavior smoothly, achieving outcomes aligned with the
%   specified preference.
%   Extreme preferences lead to predictable failure modes due to excessive risk-taking.
%   }
%   \label{fig:rg_pref_sweep}
% \end{figure}

\begin{figure}[h]
  \centering
  \begin{minipage}[t]{0.48\columnwidth}
    \centering
    \includegraphics[width=\linewidth]{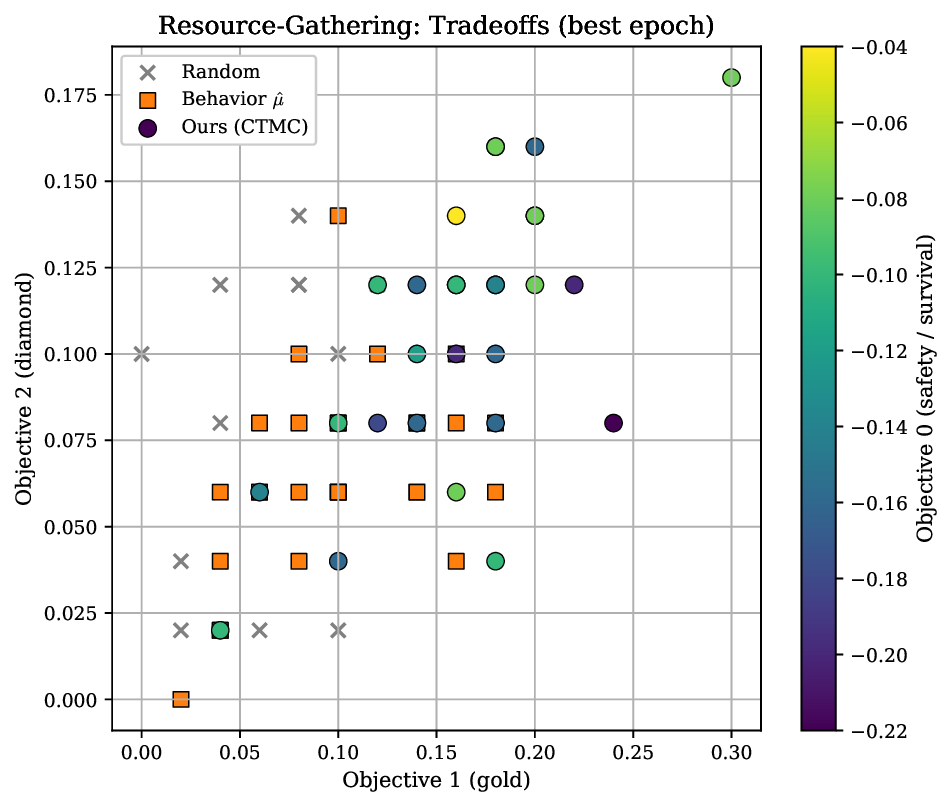}
    \captionof{figure}{
      Achieved objective outcomes on \texttt{Resource-Gathering} at the best
      validation checkpoint.
      Each point corresponds to a different preference vector $\omega$.
      Colors indicate the achieved safety level.
      The CTMC-based method reaches a broader range of outcomes than the behavior
      policy and random baseline, including combinations with higher resource
      collection.
    }
    \label{fig:rg_pareto}
  \end{minipage}
  \hfill
  \begin{minipage}[t]{0.48\columnwidth}
    \centering
    \includegraphics[width=\linewidth]{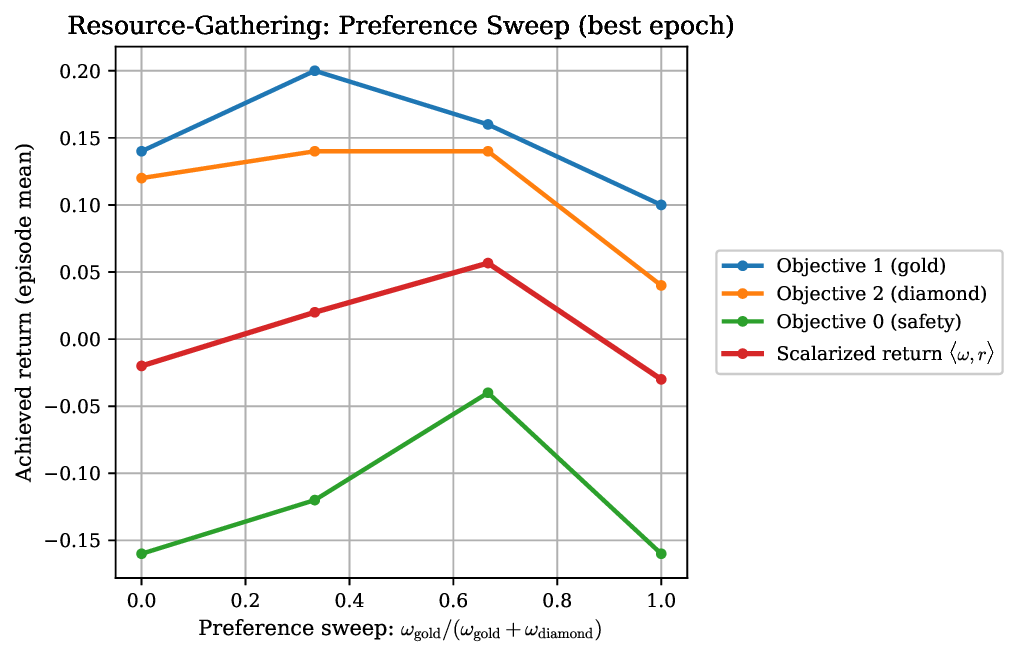}
    \captionof{figure}{
      Effect of varying preferences on \texttt{Resource-Gathering} with fixed
      safety weight.
      As the relative importance of gold versus diamond changes, the learned
      CTMC model adjusts its behavior smoothly, achieving outcomes aligned with
      the specified preference.
      Extreme preferences lead to predictable failure modes due to excessive
      risk-taking.
    }
    \label{fig:rg_pref_sweep}
  \end{minipage}
\end{figure}

These results show that the proposed discrete CTMC framework enables
effective preference-conditioned policy generation in multi-objective offline
reinforcement learning.

\subsection{Discussion on Hypervolume vs Coverage}
\label{app:hv_discussion}

Hypervolume (HV) is sensitive to extreme points on the Pareto frontier. Methods such as scalarized CQL, which train independent policies for each preference, may better optimize such extremes leading to slightly higher HV as we can see in Table~\ref{tab:morl_results}. However, this comes at the cost of learning multiple policies and potentially reduced coverage of intermediate trade-offs.

In contrast, QDFM learns a single preference-conditioned policy that models the entire Pareto front. This leads to improved diversity and coverage, as reflected in metrics such as the number of non-dominated solutions (ND) and spacing (SP), while maintaining competitive HV.

\subsection{Multi-Agent Reinforcement Learning: Matrix Game}
\label{app:marl}

To empirically validate the multi-agent extension in Sec.~\ref{sec:marl_extension}, we consider a two-agent coordination game with discrete joint actions and conflicting objectives.
Each episode draws a context state $s\in\{0,1\}$ and both agents select actions $a^1,a^2\in\{0,1,2\}$ (one-step horizon).
Rewards are two-dimensional, $\mathbf r=(r_1,r_2)$, where $r_1$ encourages coordinated choices while $r_2$ penalizes costly coordination, inducing a natural trade-off controlled by the preference $\omega$.

\paragraph{Offline dataset.}
We construct an offline dataset by rolling out a mixture of two coordinated behavior policies:
one prefers joint action $(0,0)$ and the other prefers $(1,1)$, producing diverse but structured joint behavior.
We train (i) independent behavior cloning (factorized per-agent BC), (ii) centralized behavior cloning (BC on joint actions), and (iii) our factorized CTMC policy with centralized multi-objective guidance.

\paragraph{Results.}
Fig.~\ref{fig:marl_pareto} shows that our method produces a broad, preference-controllable set of Pareto points, whereas BC baselines yield a collapsed front with weak dependence on $\omega$. Table~\ref{tab:marl_matrix} shows that our factorized CTMC policy with centralized multi-objective guidance achieves near-perfect coordination when preferences favor task reward, and smoothly trades off coordination for safety as preferences change. In contrast, both independent and centralized behavior cloning gives a collapsed trade-off and is insensitive to preferences. Our method achieves substantially higher hypervolume and spread with coordination being actively modulated by $\omega$ which shows near-perfect coordination when $\omega$ emphasizes task reward, and reduced coordination when $\omega$ emphasizes the safety objective, demonstrating that centralized preference-conditioned guidance can coordinate factorized CTMC dynamics.

% \noindent%
\begin{figure}[ht]
\centering
\includegraphics[width=0.75\linewidth]{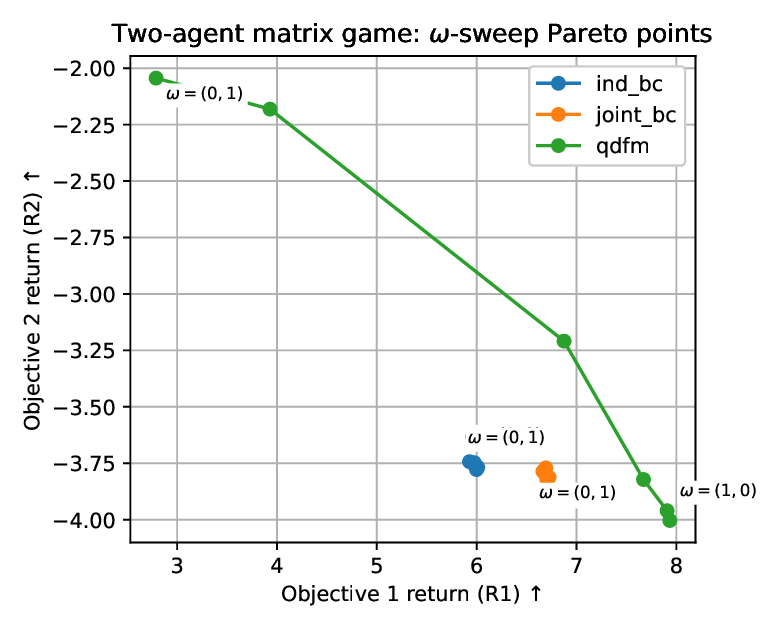}
\caption{
Two-agent matrix game (offline).
Mean Pareto points (point in objective space is Pareto if no other achievable point improves one objective without degrading another) obtained by sweeping the preference vector $\omega$,
averaged over three seeds, curve traces the trade-off between objective~1 (task reward) and
objective~2 (safety/cost).
Our factorized CTMC policy with centralized multi-objective guidance (QDFM)
produces a broad, preference-controllable Pareto front, whereas independent
and centralized behavior cloning yield collapsed fronts with weak dependence
on $\omega$ with the endpoints for $\omega=(1,0)$ and $\omega=(0,1)$
coinciding, reflecting the fact that these methods are insensitive to preference
conditioning.
}
\label{fig:marl_pareto}
\end{figure}

\begin{table}[ht]
\centering
\caption{In a two-agent offline matrix game we compare 2D hypervolume (higher is better), spread (higher indicates broader coverage here), and Coord which is the coordination rate at $\omega=(0,1)$ (lower indicates the policy deliberately
trades coordination for safety).}
\label{tab:marl_matrix}
\resizebox{\columnwidth}{!}{%
\begin{tabular}{lccc}
\toprule
Method & HV $\uparrow$ & Spread $\uparrow$ & Coord@$\omega{=}(0,1)$ \\
\midrule
Independent BC & 30.145 $\pm$ 0.704 & 0.043 $\pm$ 0.041 & 0.748 $\pm$ 0.025 \\
Centralized BC & 32.786 $\pm$ 0.284 & 0.039 $\pm$ 0.015 & 0.824 $\pm$ 0.012 \\
QDFM (ours)    & \textbf{47.667 $\pm$ 0.529} & \textbf{0.616 $\pm$ 0.072} & 0.283 $\pm$ 0.028 \\
\bottomrule
\end{tabular}%
}
\end{table}

We next provide additional details for the SMAC experiments in Section~\ref{main:smac}, including the construction and quality of the offline datasets and the corresponding win-rate results.

\subsection{SMAC Dataset Details}
\label{app:smac_datasets}

Here we provide additional details on the offline SMAC datasets used in the experiments of Section~\ref{main:smac}. Following prior work~\citep{shao2023counterfactual}, the datasets are generated using QMIX policies at different stages of training to obtain trajectories with varying levels of quality.

For 2s3z, we use the medium-replay dataset containing 877 episodes with an average return of 7.03. For 3m, we use the medium dataset containing 4000 episodes with an average return of 10.06. For 5m\_vs\_6m, we generate expert, medium, medium-replay, and poor datasets, each containing 1000 episodes, with average returns of 18.29, 11.83, 6.84, and 1.41, respectively.

\subsection{SMAC Win Rates}
\label{app:smac_winrates}

In addition to the episodic returns reported in Table~\ref{tab:smac_returns}, we report win rates on the same SMAC benchmarks in Table~\ref{tab:app_smac}. These results provide a complementary measure of task success. QDFM achieves the highest win rate on 2s3z medium-replay and remains competitive with the strongest offline MARL baselines across the other evaluated settings. Consistent with the return results in Section~\ref{main:smac}, these results demonstrate that the CTMC-based generative policy can operate effectively in multi-agent settings, without claiming a systematic performance advantage over specialized offline MARL methods.

\begin{table}[ht]
\centering
\caption{Win rate comparison on SMAC benchmarks.}
\label{tab:app_smac}
\small
\setlength{\tabcolsep}{4pt}
\begin{tabular}{llcccccc}
\toprule
Map & Dataset & BC & MACQL & CFCQL & OMAR & OMIGA & QDFM \\
\midrule

\multirow{1}{*}{2s3z}
& medium-replay 
& 0.57 $\pm$ 0.03 
& 0.42 $\pm$ 0.10 
& 0.52 $\pm$ 0.05 
& 0.45 $\pm$ 0.10 
& 0.48 $\pm$ 0.12 
& \textbf{0.59 $\pm$ 0.04} \\

\midrule

\multirow{1}{*}{3m}
& medium 
& 0.18 $\pm$ 0.02 
& 0.31 $\pm$ 0.09 
& 0.31 $\pm$ 0.03 
& 0.13 $\pm$ 0.00 
& \textbf{0.33 $\pm$ 0.03} 
& 0.29 $\pm$ 0.08 \\

\midrule

\multirow{2}{*}{5m\_vs\_6m}
& expert 
& 0.43 $\pm$ 0.03 
& 0.41 $\pm$ 0.17 
& \textbf{0.51 $\pm$ 0.07} 
& 0.29 $\pm$ 0.06 
& 0.49 $\pm$ 0.15 
& 0.48 $\pm$ 0.09 \\

& medium 
& 0.17 $\pm$ 0.01 
& 0.16 $\pm$ 0.04 
& \textbf{0.20 $\pm$ 0.12} 
& 0.17 $\pm$ 0.15 
& 0.19 $\pm$ 0.08 
& 0.19 $\pm$ 0.06 \\

\bottomrule
\end{tabular}
\end{table}

\subsection{Multi-Goal Gridworld: Full Analysis}
\label{app:multigoal}

\paragraph{Environment.}
An $11 \times 11$ gridworld with start position $(5,0)$, a trap column at $x{=}5$
for $y \in \{3,\dots,9\}$, and $K$ goal cells on the top row ($y{=}10$).
Actions are \{up, down, left, right\}. Reaching a goal gives $+10$ reward,
entering the trap gives $-10$ and ends the episode. The offline dataset
contains 250 expert trajectories per goal, balanced across all $K$ modes.
We train each method once and evaluate over 200 episodes with 5 seeds.

\paragraph{Results.}
Figure~\ref{fig:multigoal_app} and Table~\ref{tab:multigoal} show results
across $K \in \{2,3,4,5\}$. Strong baselines achieve comparable return to QDFM,
confirming that our method does not sacrifice performance. However, every strong offline RL
baseline here reaches exactly one out of $K$ goals (mode coverage $= 1/K$)
because they select actions via argmax, which is inherently unimodal.
QDFM consistently recovers nearly all modes, and the gap grows as $K$
increases. Notably, BCQ and GreedyQ frequently enter the trap because their Q-values for out-of-distribution actions (going up at the branching state) are poorly estimated without conservative penalties, further illustrating the risks of non-generative approaches in multimodal settings.

\begin{figure}[ht]
\centering
\includegraphics[width=\linewidth]{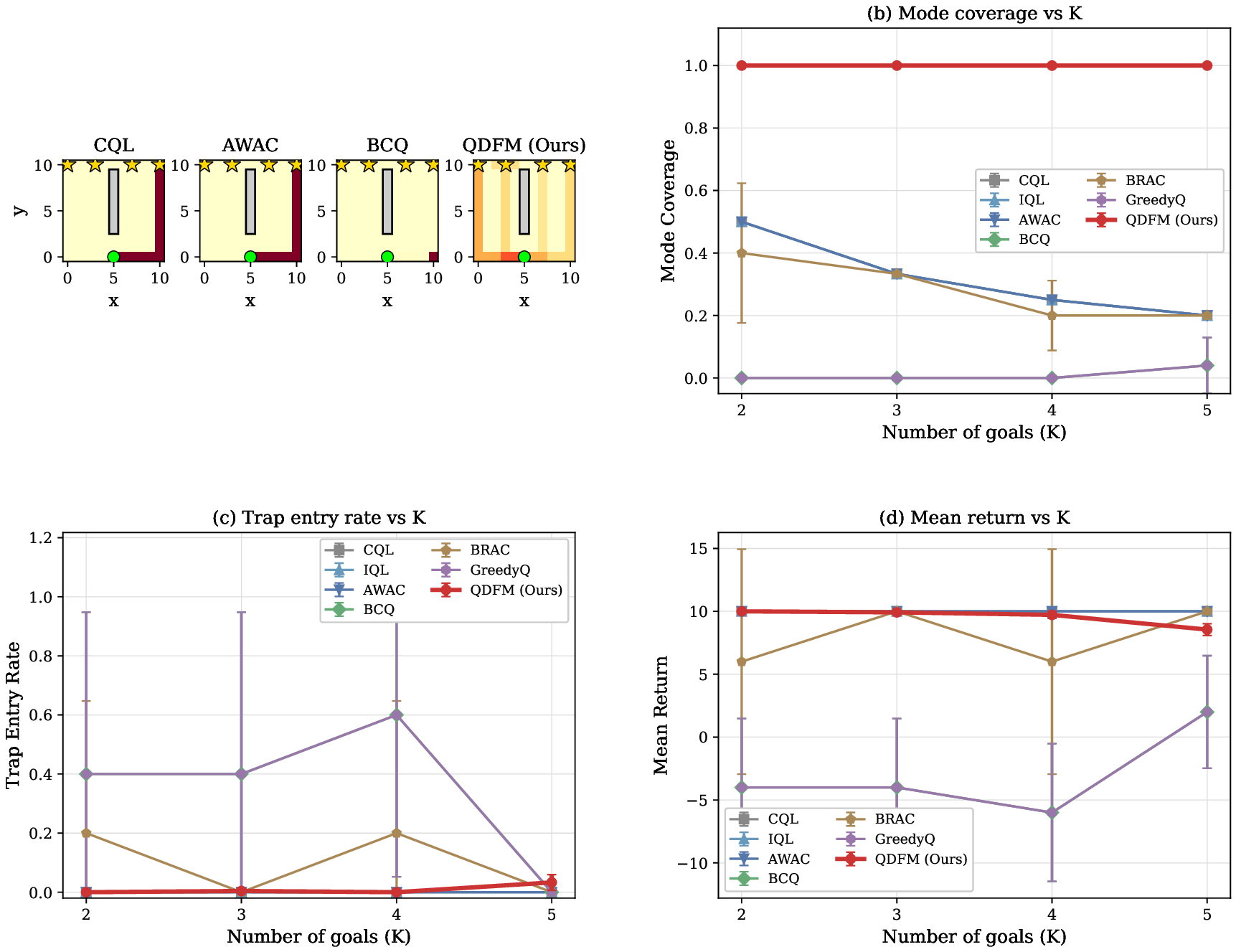}
\caption{Multi-goal gridworld analysis. (a)~State visitation heatmaps at $K{=}4$.
(b)~Mode coverage vs.\ number of goals. (c)~Trap entry rate vs.\ number of goals.
(d)~Mean return vs.\ number of goals. All values are mean $\pm$ std over 5 seeds.}
\label{fig:multigoal_app}
\end{figure}

\begin{table}[ht]
\centering
\caption{Multi-goal gridworld results. Mean $\pm$ std over 5 seeds, 200 eval episodes each.
Mode coverage is the fraction of distinct goals reached by a single trained policy.}
\label{tab:multigoal}
\small
\setlength{\tabcolsep}{4pt}
\resizebox{\linewidth}{!}{%
\begin{tabular}{lcccccccc}
\toprule
 & \multicolumn{2}{c}{$K=2$} & \multicolumn{2}{c}{$K=3$}
 & \multicolumn{2}{c}{$K=4$} & \multicolumn{2}{c}{$K=5$} \\
\cmidrule(lr){2-3}\cmidrule(lr){4-5}\cmidrule(lr){6-7}\cmidrule(lr){8-9}
Method & Return & Mode Cov & Return & Mode Cov & Return & Mode Cov & Return & Mode Cov \\
\midrule
CQL     & 10.00 & 0.50 & 10.00 & 0.33 & 10.00 & 0.25 & 10.00 & 0.20 \\
IQL     & 10.00 & 0.50 & 10.00 & 0.33 & 10.00 & 0.25 & 10.00 & 0.20 \\
XQL     & 10.00 & 0.50 & 10.00 & 0.33 & 10.00 & 0.25 & 10.00 & 0.20 \\
AWAC    & 10.00 & 0.50 & 10.00 & 0.33 & 10.00 & 0.25 & 10.00 & 0.20 \\
BCQ     & $-$4.00{\scriptsize$\pm$5.5} & 0.00 & $-$4.00{\scriptsize$\pm$5.5} & 0.00 & $-$6.00{\scriptsize$\pm$5.5} & 0.00 & 2.00{\scriptsize$\pm$4.5} & 0.04 \\
BRAC    & 6.00{\scriptsize$\pm$8.9} & 0.40 & 10.00 & 0.33 & 6.00{\scriptsize$\pm$8.9} & 0.20 & 10.00 & 0.20 \\
GreedyQ & $-$4.00{\scriptsize$\pm$5.5} & 0.00 & $-$4.00{\scriptsize$\pm$5.5} & 0.00 & $-$6.00{\scriptsize$\pm$5.5} & 0.00 & 2.00{\scriptsize$\pm$4.5} & 0.04 \\
\midrule
QDFM \textbf{(Ours)} & \textbf{10.00} & \textbf{1.00} & \textbf{9.92}{\scriptsize$\pm$0.1} & \textbf{1.00} & \textbf{9.72}{\scriptsize$\pm$0.3} & \textbf{1.00} & \textbf{8.55}{\scriptsize$\pm$0.5} & \textbf{1.00} \\
\bottomrule
\end{tabular}
}
\end{table}

\paragraph{Practical relevance.}
Recovering multiple modes from offline data matters whenever a single
strategy is insufficient. In \emph{multi-route navigation}, an autonomous
vehicle trained on historical driving logs should retain awareness of
alternative routes rather than committing to one. For example, if the preferred route is
blocked, a multimodal policy can immediately switch to a known alternative
without retraining. In \emph{personalized treatment planning}, clinical
datasets often contain multiple successful protocols for different patient
subgroups. In such cases, a mode-collapsing policy serves only one subgroup well, while a
multimodal policy can generate diverse treatment options for downstream
selection by a clinician.

\section{Experimental Protocol and Reproducibility}
\label{app:exp_protocol}

This section gives the experimental details needed to reproduce the results in
this paper.

\subsection{Benchmarks}

We evaluate QDFM on four groups of tasks. The first is discretized MuJoCo from
Minari, where continuous actions are converted into a finite set and all methods
use the same discretization. We construct the discrete action codebook
using k-means clustering on the offline dataset's joint action vectors,
with the cluster centroids defining the available discrete actions. The main comparison in
Sec.~\ref{sec:mujoco_discrete_bench} uses $K_{\mathrm{act}}=256$, following the
ablation in Appendix~\ref{app:action_divisions}. The second is intrinsically
discrete multi-objective control, including Deep Sea Treasure and Resource
Gathering (Appendix~\ref{app:resource_gathering}). These tasks test whether one
preference-conditioned policy can adapt to different objective weights at test
time. The third is offline multi-agent RL on SMAC
(Appendix~\ref{app:smac_datasets}), where the goal is to test whether the
factorized CTMC scales to combinatorial joint action spaces. The fourth is the
multi-goal gridworld diagnostic (Appendix~\ref{app:multigoal}), which isolates
the multimodal decision-making advantage of generative policies over
deterministic baselines.

\subsection{Baselines}

For discretized MuJoCo we compare against IQL, XQL, ACTIVE-I, BCQ,
CQL, CQL(SA), BRAC, and BoltzQ in Table~\ref{tab:main_baselines_256}.
Additional comparisons with AWAC, AWBC, GreedyQ, and the matched
CQL-SAC baseline are provided in Appendix~\ref{app:comparisons}.
All methods use the same discrete action set ($K_{\text{act}}=256$). BoltzQ shares the same critic as
QDFM and applies categorical Boltzmann action selection, which isolates the
benefit of the CTMC generative policy. For multi-objective discrete tasks we
compare against scalarized CQL and MODULI. For SMAC we compare against BC,
MACQL, CFCQL, OMAR, and OMIGA under the same dataset setting. For the
multi-goal gridworld we compare against CQL, IQL, XQL, AWAC, BCQ, BRAC, and GreedyQ.

\subsection{Training Setup}

The rate model $u_{\theta,t}$ is an MLP with two hidden layers of width 256, as
described in Appendix~\ref{ap:imp_details}. QDFM uses the three-phase training
procedure in Appendix~\ref{app.full_algo}. Phase 1 learns an unweighted discrete
flow model, Phase 2 learns the critic, and Phase 3 performs Q-weighted policy
improvement. For the MuJoCo comparison all methods use batch size 256. Baselines
are trained for 1M gradient steps. QDFM uses
$(K_1, K_2, K_3) = (150\mathrm{k}, 500\mathrm{k}, 350\mathrm{k})$, support
size $M=64$, guidance scale $\hat{\beta}=20$, and CTMC step size $h=0.05$. For
the multi-goal gridworld we use $\hat{\beta}=5$ (adjusted for the smaller
Q-value scale), $M=16$, and train baselines for 30k steps and QDFM for
$(K_1, K_2, K_3) = (3\mathrm{k}, 20\mathrm{k}, 10\mathrm{k})$. All models are
optimized with Adam and the same optimizer settings are kept fixed across seeds
for each method.

\subsection{Evaluation Setup}

For MuJoCo, each trained policy is evaluated in the corresponding environment
and returns are normalized following the standard D4RL protocol. The main table
reports mean and standard deviation over 5 seeds. For CTMC sampler diagnostics
we evaluate 100 episodes per seed when reporting sensitivity to the rate scale
and Euler steps. For the multi-goal gridworld in
Appendix~\ref{app:multigoal}, each method is evaluated over 200 episodes with 5
seeds. For SMAC, returns and win rates are averaged over 5 seeds.

\subsection{Randomness and Seeds}

Each seed controls model initialization, minibatch sampling, endpoint sampling,
preference sampling, CTMC sampling, and environment evaluation randomness. For
benchmarks with fixed offline data the dataset is shared across methods and only
the training and evaluation randomness changes across seeds. For generated
datasets such as the multi-goal gridworld and selected SMAC splits, the data
generation procedure is deterministic and shared across all methods and seeds.

\subsection{Statistical Reporting}

All tables and plots report mean $\pm$ sample standard deviation unless stated
otherwise. For a metric value $x_i$ from seed $i$ we compute
\[
\bar{x} = \frac{1}{n}\sum_{i=1}^{n} x_i,
\qquad
s = \sqrt{\frac{1}{n-1}\sum_{i=1}^{n}(x_i - \bar{x})^2}.
\]
We report the sample standard deviation $s$, not the standard error of the mean.
We do not report formal $p$-values because several experiments use a small number
of seeds and we do not assume normally distributed returns. We interpret
overlapping error bars conservatively and focus on consistent trends across
tasks, baselines, and ablations.

\subsection{Compute Environment}

Experiments were run on a Slurm-managed GPU cluster. Each run used one GPU
unless otherwise stated. The main MuJoCo and runtime experiments used NVIDIA
GPUs with 8 to 16 CPU cores and 16GB to 64GB memory per job. The multi-goal
gridworld and Resource Gathering experiments can run on a single GPU in under
one hour. SMAC experiments used one GPU together with CPU workers for the
StarCraft II environment process.

\begin{table}[h]
\centering
\caption{Compute resources for the reported experiments.}
\label{tab:compute_resources}
\small
\setlength{\tabcolsep}{4pt}
\begin{tabular}{p{0.25\linewidth}p{0.22\linewidth}p{0.15\linewidth}p{0.28\linewidth}}
\toprule
Experiment group & GPU type & Memory & Typical run time per seed \\
\midrule
Discretized MuJoCo & A100 80GB class & 32 to 64GB & 0.6 to 2.1 hours \\
Multi-objective gridworlds & CPU or one GPU & 16 to 32GB & Under one hour \\
Resource Gathering & CPU or one GPU & 16 to 32GB & One to a few hours \\
Multi-goal gridworld & One GPU & 8 to 16GB & Under 10 minutes \\
SMAC & One GPU with CPU workers & 32 to 64GB & Map dependent \\
\bottomrule
\end{tabular}
\end{table}

\subsection{Code and Data}

The experiments use public benchmark environments and offline datasets when
available. The code submitted as supplementary zip file includes scripts for data loading, action
discretization, QDFM training, baseline training, evaluation, and plotting. For
generated datasets including the multi-goal gridworld and SMAC splits the
file includes the generation scripts and the dataset statistics reported
in the corresponding appendix sections. The implementation is publicly available at
\url{https://github.com/FairozNowerKhan/QDFM}.

\section{Broader Impacts}
\label{app:broader_impacts}

This work introduces a general purpose offline RL algorithm and does not target
a specific deployed application. The ability to recover multimodal policies from
offline data could benefit domains such as personalized medicine and autonomous
navigation, where maintaining diverse strategies improves robustness and
adaptability. As with any policy learning method trained on historical data,
there is a risk that biases present in the offline dataset are reproduced by the
learned policy. Practitioners should audit offline datasets for
representativeness before deployment. We do not foresee risks specific to our
method beyond those common to offline RL and generative modeling research.

\clearpage

\section{Complete Training procedure}
\label{app.full_algo}

For completeness, we provide pseudocode for the full Q-weighted conditional
DFM training pipeline used in our experiments.

% \vspace{-1pt}

\begin{algorithm}[H]
\small
\caption{Q-weighted discrete flow matching for offline RL}
\label{alg:discrete_algo}
\begin{algorithmic}[1]
\STATE {\bfseries Input:} Rate network $u_{\theta,t}(a',a\mid s,\omega)$, guidance scale $\hat{\beta}$, batch size $B$
\STATE {\bfseries Input:} Regularized target rate $u_t^{(\epsilon)}$ ($\epsilon=10^{-4}$; Eq.~\eqref{eq:jump_to_endpoint_eps}) and conditional path $p_{t\mid Z}$ (Eq.~\ref{eq:dfm-mixture-path})
\STATE {\bfseries Input:} Offline dataset $\mathcal{D}=\{(s,a,\mathbf r,s')\}$, learned behavior policy $\hat\mu(a\mid s)$ trained on $\mathcal D$\\
\STATE {\bfseries Input:} Base distribution $p_0$ over $\mathcal A$, Epochs $K_1$ (DFM warm-up), $K_2$ (critic warm-up), $K_3$ (policy improvement), support size $M$, renew freq. $K_{\mathrm{renew}}$

\vspace{2pt}
\STATE \textit{Phase 1: flow matching model warm-up}
\FOR{$k=1$ {\bfseries to} $K_1$}
  \FOR{batch $\{(s_i,a_i)\}_{i=1}^B \sim \mathcal{D}$}
    \STATE Sample $\omega_i\sim Uniform$, sample $A_{i,0}\sim p_0$,\; sample $A_{i,1}\sim \hat\mu(\cdot\mid s_i)$
    \STATE Sample $t_i\sim \mathrm{Unif}[0,1]$ and $A_{i,t} \sim p_{t_i\mid Z_i}$ where $Z_i=(s_i,\omega_i,A_{i,0},A_{i,1})$
    \STATE Minimize 
    \[
      \mathcal{L}_{\mathrm{DFM}}(\theta)
      =
      \frac{1}{B}\sum_{i=1}^B \big\|\!\Big(
      u_{t_i}^{(\epsilon)}(\cdot, A_{i,t}\mid Z_i) - u_{\theta,t_i}(\cdot, A_{i,t}\mid s_i,\omega_i)
      \Big)\big\|_2^2.
    \]
  \ENDFOR
\ENDFOR

\vspace{2pt}
\STATE \textit{Phase 2: critic learning (Boltzmann backup on in-support actions)}
\FOR{$k=1$ {\bfseries to} $K_2$}
  \FOR{batch $\{(s_i,a_i,\mathbf r_i,s'_i)\}_{i=1}^B \sim \mathcal{D}$}
    \STATE Sample $\omega_i\sim Uniform$. Define $Q_{\omega_i}(s,a):=\langle \omega_i,\mathbf Q_\psi(s,a)\rangle$.
    \IF{$k \bmod K_{\mathrm{renew}} = 1$}
      \STATE Sample support $\{a'_{ij}\}_{j=1}^M \sim \hat\mu(\cdot\mid s'_i)$
    \ENDIF
    \STATE Compute scalar Q-values: $q'_{ij} \gets Q_{\omega_i}(s'_i, a'_{ij})$ for $j=1\dots M$.
    \STATE $V_{\omega_i}(s'_i) \gets \frac{\sum_{j=1}^M \exp(\hat{\beta} q'_{ij})\,q'_{ij}}{\sum_{j=1}^M \exp(\hat{\beta} q'_{ij})}$
    \STATE $y_i \gets \langle \omega_i,\mathbf r_i\rangle + \gamma V_{\omega_i}(s'_i)$
    \STATE Update $\psi$ by minimizing $\mathcal{L}_Q(\psi)=\frac{1}{B}\sum_{i=1}^B (Q_{\omega_i}(s_i,a_i)-y_i)^2$
  \ENDFOR
\ENDFOR

\vspace{2pt}
\STATE \textit{Phase 3: Q-weighted iterative policy improvement (Eq.~\ref{eq:qcdfm} with in-support approximation)}
\FOR{$k=1$ {\bfseries to} $K_3$}
  \FOR{each batch $\{(s_i,a_i)\}_{i=1}^B \sim \mathcal{D}$}
    \STATE Sample $\omega_i\sim Uniform$, set $A_{i,0} \gets a_i$
    \IF{$k \bmod K_{\mathrm{renew}} = 1$}
      \STATE Sample endpoints $\{A_{i,j,1}\}_{j=1}^M$ using the rate model
$u_{\theta,t}(\cdot,\cdot\mid s_i,\omega_i)$ by running the CTMC inference procedure
(Algorithm~\ref{alg:ctmc_inference}).
    \ENDIF
    \STATE Compute scalar Q-values: $ q_{ij} \gets Q_{\omega_i}(s_i,A_{i,j,1})
      =
      \langle \omega_i, \mathbf{Q}_\psi(s_i,A_{i,j,1}) \rangle$ for $j=1\dots M$.
    \STATE $w(s_i,\omega_i,A_{i,j,1}) \gets \frac{\exp(\hat{\beta} q_{ij})}{\sum_{\ell=1}^M \exp(\hat{\beta} q_{i\ell})}$
    \STATE Sample $t_{i,j}\sim \mathrm{Unif}[0,1]$ and $A_{i,j,t}\sim p_{t_{i,j}\mid Z_{ij}}$ where $Z_{ij}=(s_i,\omega_i,A_{i,0},A_{i,j,1})$
    \STATE Update $\theta$ by minimizing
    \[
      \mathcal{L}_{\mathrm{QDFM}}(\theta)
      =
      \frac{1}{B}\sum_{i=1}^B\sum_{j=1}^M
      w(s_i,\omega_i,A_{i,j,1})\,
     \big\| \Big(
      u_{t_{i,j}}^{(\epsilon)}(\cdot,A_{i,j,t}\mid Z_{ij}) - u_{\theta,t_{i,j}}(\cdot,A_{i,j,t}\mid s_i,\omega_i)
      \Big)\big\|_2^2.
    \]
  \ENDFOR
\ENDFOR

\STATE {\bfseries Output:} terminal CTMC distribution induced by $u_{\theta,t}$.
\end{algorithmic}
\end{algorithm}

%%%%%%%%%%%%%%%%%%%%%%%%%%%%%%%%%%%%%%%%%%%%%%%%%%%%%%%%%%%%

\newpage
\section*{NeurIPS Paper Checklist}

\begin{enumerate}

\item {\bf Claims}
    \item[] Question: Do the main claims made in the abstract and introduction accurately reflect the paper's contributions and scope?
    \item[] Answer: \answerYes{}
    \item[] Justification: All claims are supported by Theorem~\ref{thm:guided_grad_equiv_main}, Corollary~\ref{cor:optimal_policy}, and experiments in Sec.~\ref{sec:experiments}.

\item {\bf Limitations}
    \item[] Question: Does the paper discuss the limitations of the work performed by the authors?
    \item[] Answer: \answerYes{}
    \item[] Justification: Sec.~\ref{conclusion} discusses the limitations including inference cost, scale of experiments and diagnostic scope.

\item {\bf Theory assumptions and proofs}
    \item[] Question: For each theoretical result, does the paper provide the full set of assumptions and a complete (and correct) proof?
    \item[] Answer: \answerYes{}
    \item[] Justification: Assumptions~\ref{ass:support} and~\ref{ass:realizability} are stated formally. Full proofs are in Appendix~\ref{sec:proof_thm1} and~\ref{sec:proof_cor1}.

    \item {\bf Experimental result reproducibility}
    \item[] Question: Does the paper fully disclose all the information needed to reproduce the main experimental results of the paper to the extent that it affects the main claims and/or conclusions of the paper (regardless of whether the code and data are provided or not)?
    \item[] Answer: \answerYes{}
    \item[] Justification: All details are in Appendix~\ref{ap:imp_details},~\ref{app:exp_protocol}, and~\ref{app.full_algo}. Code used to run experiments with instructions to reproduce our results is included as supplementary material.

\item {\bf Open access to data and code}
    \item[] Question: Does the paper provide open access to the data and code, with sufficient instructions to faithfully reproduce the main experimental results, as described in supplemental material?
    \item[] Answer: \answerYes{}
    \item[] Justification: Our code is submitted as a supplementary zip file. All benchmark environments are publicly available.

\item {\bf Experimental setting/details}
    \item[] Question: Does the paper specify all the training and test details (e.g., data splits, hyperparameters, how they were chosen, type of optimizer) necessary to understand the results?
    \item[] Answer: \answerYes{}
    \item[] Justification: Full experimental details are in Sec.~\ref{sec:experiments}, Appendix~\ref{ap:imp_details}, and Appendix~\ref{app:exp_protocol}.

\item {\bf Experiment statistical significance}
    \item[] Question: Does the paper report error bars suitably and correctly defined or other appropriate information about the statistical significance of the experiments?
    \item[] Answer: \answerYes{}
    \item[] Justification: All tables report mean $\pm$ sample standard deviation over 5 seeds, unless otherwise stated. The methodology is described in Appendix~\ref{app:exp_protocol}.

\item {\bf Experiments compute resources}
    \item[] Question: For each experiment, does the paper provide sufficient information on the computer resources (type of compute workers, memory, time of execution) needed to reproduce the experiments?
    \item[] Answer: \answerYes{}
    \item[] Justification: GPU type, memory, and run times are in Table~\ref{tab:compute_resources} and Appendix~\ref{app:runtime}.
    
\item {\bf Code of ethics}
    \item[] Question: Does the research conducted in the paper conform, in every respect, with the NeurIPS Code of Ethics \url{https://neurips.cc/public/EthicsGuidelines}?
    \item[] Answer: \answerYes{}
    \item[] Justification: We have reviewed the NeurIPS Code of Ethics and the research conforms with it.

\item {\bf Broader impacts}
    \item[] Question: Does the paper discuss both potential positive societal impacts and negative societal impacts of the work performed?
    \item[] Answer: \answerYes{}
    \item[] Justification: Broader impacts are discussed in Appendix~\ref{app:broader_impacts}.
    
\item {\bf Safeguards}
    \item[] Question: Does the paper describe safeguards that have been put in place for responsible release of data or models that have a high risk for misuse (e.g., pre-trained language models, image generators, or scraped datasets)?
    \item[] Answer: \answerNA{}
    \item[] Justification: The paper releases only small MLP training code on public benchmarks.

\item {\bf Licenses for existing assets}
    \item[] Question: Are the creators or original owners of assets (e.g., code, data, models), used in the paper, properly credited and are the license and terms of use explicitly mentioned and properly respected?
    \item[] Answer: \answerYes{}
    \item[] Justification: All environments and datasets are cited and publicly available under open-source licenses.

\item {\bf New assets}
    \item[] Question: Are new assets introduced in the paper well documented and is the documentation provided alongside the assets?
    \item[] Answer: \answerYes{}
    \item[] Justification: The supplementary zip file includes all code and documentation.

\item {\bf Crowdsourcing and research with human subjects}
    \item[] Question: For crowdsourcing experiments and research with human subjects, does the paper include the full text of instructions given to participants and screenshots, if applicable, as well as details about compensation (if any)? 
    \item[] Answer: \answerNA{}
    \item[] Justification: No crowdsourcing or human subjects involved.

\item {\bf Institutional review board (IRB) approvals or equivalent for research with human subjects}
    \item[] Question: Does the paper describe potential risks incurred by study participants, whether such risks were disclosed to the subjects, and whether Institutional Review Board (IRB) approvals (or an equivalent approval/review based on the requirements of your country or institution) were obtained?
    \item[] Answer: \answerNA{}
    \item[] Justification: No human subjects research involved.

\item {\bf Declaration of LLM usage}
    \item[] Question: Does the paper describe the usage of LLMs if it is an important, original, or non-standard component of the core methods in this research? Note that if the LLM is used only for writing, editing, or formatting purposes and does \emph{not} impact the core methodology, scientific rigor, or originality of the research, declaration is not required.
    %this research? 
    \item[] Answer: \answerNA{}
    \item[] Justification: The core method does not involve LLMs.

\end{enumerate}

\end{document}